\documentclass[10pt]{article} % For LaTeX2e
\usepackage[accepted]{tmlr}
\usepackage{amsmath,amsfonts,bm}

\def\eqref#1{equation~\ref{#1}}
\def\1{\bm{1}}

\DeclareMathAlphabet{\mathsfit}{\encodingdefault}{\sfdefault}{m}{sl}
\SetMathAlphabet{\mathsfit}{bold}{\encodingdefault}{\sfdefault}{bx}{n}

\usepackage{comment}
\usepackage[misc]{ifsym}
\usepackage{amsmath} % For mathematical symbols and equations
\usepackage{graphicx} % For inserting graphics and images
\usepackage{algorithm} % For typesetting algorithms
\usepackage{algorithmic} % For typesetting algorithms in a specific style
\usepackage{multirow} % For multirow cells in tables
\usepackage{listings} % For typesetting source code
\usepackage{booktabs}
\usepackage[dvipsnames]{xcolor}
\usepackage[normalem]{ulem}
\usepackage{adjustbox}
\usepackage{diagbox}
\usepackage{bbm}
\usepackage{csquotes}
\usepackage{colortbl}
\usepackage{hyphenat}
\usepackage{tabularray}
\usepackage{float}
\usepackage{array}
\usepackage{amssymb}
\usepackage{epsfig}
\usepackage{makecell}
\usepackage{tabularx,booktabs,ragged2e}
\usepackage{setspace}
\usepackage{subcaption}

\usepackage{amsthm}
\usepackage{hyperref}
\usepackage[capitalize,nameinlink]{cleveref}

\usepackage{placeins}
\usepackage[skip=4pt]{caption}

\newtheorem{proposition}{Proposition}

\title{Individual Fairness via Local Predictive Consistency}

\author{%
\makebox[\linewidth][l]{%
  {\name Zhouting Zhao\textsuperscript{*}}%
  \hfill{\small\it zhaozh@tcd.ie}%
}\\
{\addr School of Computer Science and Statistics}\\
{\addr Trinity College Dublin}
\AND
\makebox[\linewidth][l]{%
  {\name Tin Lok James Ng}%
  \hfill{\small\it ngja@tcd.ie}%
}\\
{\addr School of Computer Science and Statistics}\\
{\addr Trinity College Dublin}
}

\def\month{10}
\def\year{2026}
\def\openreview{\url{https://openreview.net/forum?id=Yd1cLycbmR}}

\begin{document}

\maketitle

\begin{abstract}
Individual fairness in machine learning requires that similar individuals receive similar predictions. Most existing definitions are pairwise in nature, imposing a Lipschitz constraint that must hold either globally across all pairs of data instances or locally within specified neighborhoods. In practice, however, such pairwise constraints may be overly restrictive and may not directly identify which individuals receive locally inconsistent predictions. We study local predictive consistency as a node-level relaxation of pairwise individual fairness, where each prediction is compared with a similarity-weighted average of predictions in its local neighborhood. This yields an interpretable node-level fairness score and a corresponding regularized learning approach based on local residuals. Since graph-based individual fairness depends critically on the underlying similarity structure, we make the construction of the neighborhood graph explicit and discuss threshold selection through graph-coverage diagnostics. Our approach applies to binary classification, regression, and time-to-event prediction. We further compare this formulation with pairwise constraint-based approaches and analyze the resulting fairness--performance trade-offs.
\end{abstract}

\section{Introduction}
Fairness in machine learning (ML) aims to prevent predictive models from producing systematically unjust or discriminatory outcomes for certain individuals or groups~\citep{rabonato2025systematic}. As ML systems are increasingly deployed in high-stakes domains such as healthcare~\citep{anderson2025algorithmic}, finance~\citep{song2022individual}, and criminal justice~\citep{berk2021fairness}, fairness becomes especially salient because model outputs may directly shape access to resources, opportunities, or liberty. This perspective is widely covered in surveys and monographs on algorithmic fairness~\citep{pessach2022review,barocas2023fairness,varona2022discrimination}.

Fairness is commonly conceptualized through two complementary perspectives: \emph{group fairness} and \emph{individual fairness}~\citep{pessach2022review}. Group fairness aims to equalize aggregate outcomes across sensitive groups, whereas individual fairness requires that \emph{similar individuals receive similar predictions}~\citep{dwork2012fairness}. In this work, we adopt the latter view and focus on fairness at the level of individual predictions. Existing definitions of individual fairness are typically pairwise in nature and may be enforced either globally or within local neighborhoods through a Lipschitz condition. Specifically, the distance between the predictions for two similar individuals must be bounded by the distance between their feature representations in the feature space, up to a Lipschitz constant. This definition of fairness can be formulated as a constraint or penalty and incorporated into learning algorithms, either by jointly training a predictive model with the fairness penalty or by post-processing a trained model to enforce individual fairness~\citep{li2023accurate,petersen2021post}.

There are two main limitations associated with existing pairwise-based definitions of individual fairness. First, requiring the Lipschitz condition to hold globally across all pairs, or locally across all neighboring pairs, can be overly restrictive in practice. Since the constraint is imposed on every relevant edge in the similarity graph, a small number of difficult local comparisons may force substantial changes to the learned predictor and may lead to unnecessary degradation in predictive performance. Second, pairwise fairness is primarily an edge-level notion. It identifies pairs of individuals whose predictions are inconsistent, but it does not directly provide a fairness diagnosis for a single individual. As a result, it can be difficult to determine which individuals are most exposed to locally anomalous predictions.

This work addresses these limitations by formulating individual fairness through \emph{local predictive consistency}. Instead of asking whether every neighboring pair satisfies a pairwise Lipschitz-style constraint, we ask whether each individual receives a prediction that is consistent with the predictions assigned to its local comparison set. More specifically, for each instance, we compare its prediction with a similarity-weighted neighborhood average of predictions among nearby instances. The resulting residual provides an instance-level fairness score: a large residual indicates that the individual receives a prediction that is locally anomalous relative to similarly situated individuals.

Conceptually, local predictive consistency should be understood as a node-level relaxation of pairwise individual fairness rather than as a replacement for it. Pairwise individual fairness treats local unfairness as a collection of edge violations between comparable individuals, whereas local predictive consistency summarizes the local comparison structure around each individual and assigns a diagnostic score to the node itself. The unit of fairness diagnosis therefore changes from an edge violation between two individuals to a node-level residual attached to a single individual. This shift is useful in applications where fairness assessment requires identifying which individuals are potentially treated inconsistently, rather than only identifying which pairs violate a constraint.

The main contributions of this work are summarized as follows:
\begin{enumerate}
    \item
    We formulate local predictive consistency as a node-level relaxation of pairwise individual fairness. The resulting score measures whether an individual prediction is locally anomalous relative to a similarity-weighted neighborhood, providing an interpretable per-instance fairness diagnosis.

    \item
    We clarify the relationship between node-level local predictive consistency and edge-level pairwise fairness. In particular, we show that local pairwise Lipschitz fairness controls the instance-level residual, positioning the formulation as a relaxation of pairwise fairness.

    \item
    We use the residual-based fairness score to define a regularization formulation that naturally corresponds to the row-random-walk Laplacian. This formulation can be integrated into standard learning objectives across binary classification, regression, and time-to-event prediction.

   % \item
   % We explicitly discuss graph construction and threshold selection for neighborhood-based individual fairness, and evaluate the sensitivity of the method to the neighborhood threshold through graph-coverage diagnostics.
\end{enumerate}

\section{Background and Related Work}
\label{section_background}

\subsection{Individual Fairness}
\label{subsec:if-metrics}

Individual fairness, introduced by \citet{dwork2012fairness}, requires that similar individuals receive similar predictions. Formally, let the data features be $\{\mathbf{x}_i\}_{i=1}^n \subset {\cal X}$, and let $f: {\cal X} \rightarrow {\cal Y}$ be a predictive map, with $d_x(\cdot,\cdot)$ and $d_y(\cdot,\cdot)$ denoting metrics on ${\cal X}$ and ${\cal Y}$, respectively. The canonical formulation of individual fairness imposes the Lipschitz condition
\begin{equation}
\label{eq:global_lipschitz}
d_y\bigl(f(\mathbf{x}_i), f(\mathbf{x}_j)\bigr) \le c \, d_x(\mathbf{x}_i, \mathbf{x}_j),
\qquad \forall i,j,
\end{equation}
for some constant $c>0$. To improve scalability and practical applicability, several works relax the global Lipschitz constraint by enforcing it only within local neighborhoods~\citep{petersen2021post,mahamadou2024individual}. For a neighborhood radius $\tau>0$, define the local edge set
\begin{equation}
\label{eq:local_edge_set}
\mathcal{E}_\tau = \{(i,j): d_x(\mathbf{x}_i,\mathbf{x}_j)\le \tau\}.    
\end{equation}
The corresponding local pairwise condition is
\begin{equation}
\label{eq:local_lipschitz}
d_y\bigl(f(\mathbf{x}_i),f(\mathbf{x}_j)\bigr)\le c\, d_x(\mathbf{x}_i,\mathbf{x}_j),
\qquad \forall (i,j)\in \mathcal{E}_\tau .
\end{equation}

Several works have proposed methods to enforce individual fairness in machine learning models by leveraging the pairwise fairness notion. These approaches can be broadly categorized into three types. Pre-processing methods modify the input data or its representations to reflect pairwise fairness constraints before training, ensuring that similar individuals are treated consistently in downstream models~\citep{lahoti2019ifair,lahoti2019operationalizing}. In-processing methods incorporate fairness directly into model training by adding a fairness penalty or constraint as a regularization term~\citep{yurochkin2019training,yurochkin2020sensei}. Post-processing methods adjust the predictions of an already trained model to satisfy fairness criteria~\citep{petersen2021post}. Together, these strategies provide multiple avenues to operationalize pairwise-based individual fairness at different stages of the machine learning pipeline.

A recurring challenge in individual fairness is that the principle becomes operational only after a meaningful comparison structure has been specified. The original formulation relies on a task-specific similarity metric to formalize the requirement that similar individuals should be treated similarly~\citep{dwork2012fairness}. Subsequent work has therefore studied how such metrics or representations can be obtained from human judgements, data, or learned representations~\citep{ilvento2019metric,mukherjee2020two,lahoti2019operationalizing}. Graph-based approaches make this comparison structure explicit by representing individuals as nodes and similarity relations as edges, for example through similarity graphs or graph Laplacian regularization~\citep{petersen2021post,mahamadou2024individual}. This dependence on the comparison structure motivates the graph-construction discussion in Section~\ref{subsec:proposed-if-score}.

\subsection{Local Predictive Consistency}
\label{subsec:if-metrics-knn}

Recent work has started to investigate instance-level diagnostics for individual fairness. For example, \citet{waller2025beyond} study local predictive consistency specifically in the binary-classification setting, where $f(\mathbf{x}) \in \{0,1\}$. Their consistency diagnostic measures, for each individual $\mathbf{x}_i$, the discrepancy between its prediction and the average prediction of its local neighborhood.

Let $\mathrm{nbr}(\mathbf{x}_i,k)$ denote the set of the $k$ nearest neighbors of $\mathbf{x}_i$ under the distance metric $d_x$. The corresponding local predictive inconsistency can be written as

\begin{equation}
\label{eq:consistency}
\mathrm{diff}(\mathbf{x}_i) =
\bigg|
f(\mathbf{x}_i) -
\frac{1}{k}
\sum_{\mathbf{x}_j \in \mathrm{nbr}(\mathbf{x}_i,k)}
f(\mathbf{x}_j)
\bigg|.
\end{equation}

Unlike the pairwise Lipschitz constraint in \cref{eq:local_lipschitz}, which formulates individual fairness through inequalities over pairs of instances, \cref{eq:consistency} provides a per-instance diagnostic of how far an individual prediction deviates from the average prediction of its local neighborhood. In \citet{waller2025beyond}, this quantity is used for post-hoc evaluation rather than to construct an in-processing regularization objective.

Our formulation is closely related to this local-consistency perspective but extends it in both scope and use. Rather than using a uniform average over a fixed $k$-nearest-neighbor set, we construct a similarity-weighted comparison graph, allowing nearby individuals to contribute to the local predictive reference according to their graph similarity. The formulation is defined on task-appropriate predictive outputs and is instantiated not only for binary classification but also for regression and time-to-event prediction. We further establish a theoretical connection between the resulting node-level notion of local consistency and pairwise individual fairness, and use the resulting residual to construct a differentiable in-processing regularizer. The full formulation is developed in \cref{section_method}.

\section{Methodology}
\label{section_method}
\subsection{Individual Fairness Based on Local Predictive Consistency}
\label{subsec:proposed-if-score}

We formulate local predictive consistency as a node-level relaxation of pairwise individual fairness, building on the diagnostic perspective discussed in Section~\ref{subsec:if-metrics-knn}. Rather than replacing pairwise individual fairness with a completely new definition, this formulation provides an instance-level diagnostic that summarizes how each prediction compares with its local neighborhood. In the definition given in \cref{eq:consistency}, the average prediction of an instance's neighbors is computed as the mean prediction of its $k$ nearest neighbors, which implicitly assigns equal influence to all neighbors. We generalize this idea by using a weighted average, where the contribution of each neighbor is determined by its similarity to the instance, so that closer neighbors exert greater influence on the local reference prediction.

We begin by presenting our formulation for the case where the output space $\mathcal{Y}=\mathbb{R}$ consists of real valued responses and is equipped with  the Euclidean distance $d_y(a,b)=|a-b|$. Extensions to other output spaces $\mathcal{Y}$ will be discussed later. 

Suppose we are given data instances $\{\mathbf{x}_i\}_{i=1}^n \subset \mathcal{X}$. 
For a predictive model $f_\theta:\mathcal{X}\rightarrow\mathbb{R}$ parameterized by $\theta \in \Theta$, we denote the vector of predictions on the dataset by
\[
\mathbf{f}_\theta :=
\bigl(f_\theta(\mathbf{x}_1),\ldots,f_\theta(\mathbf{x}_n)\bigr)^\top \in \mathbb{R}^n .
\]

Let $P\in\mathbb{R}^{n\times n}$ be a row-stochastic matrix satisfying $P_{ij}\ge 0$ and $\sum_{j=1}^n P_{ij}=1$ for all $i$. We interpret $P$ as a neighborhood averaging operator over the dataset. In particular, the $i$th row of $P$ encodes the influence of other instances on $\mathbf{x}_i$, typically determined by similarity in feature space. Applying $P$ to the prediction vector produces a local reference prediction for each instance,
\begin{equation}
\label{eq:local_reference}
(P\mathbf{f}_\theta)_i 
\;=\; \sum_{j=1}^n P_{ij}\, f_\theta(\mathbf{x}_j),
\end{equation}
which represents a weighted average of the predictions of the neighbors of $\mathbf{x}_i$.

We define the local residual at instance $i$ as the deviation between the prediction for $\mathbf{x}_i$ and the average prediction of its neighbors,
\begin{equation}
\label{eq:residual}
r_i(\theta;P)
\;:=\; f_\theta(\mathbf{x}_i) - (P\mathbf{f}_\theta)_i .
\end{equation}
The magnitude of this residual captures the degree to which the prediction at
$\mathbf{x}_i$ is inconsistent with its aggregated local neighborhood
reference. Accordingly, we define the instance-level individual fairness (IF) score as
\begin{equation}
\label{eq:instance_ind_fairness_score}
\mathrm{IF}_i(\theta;P) := |r_i(\theta;P)|.
\end{equation}
A large value of $\mathrm{IF}_i(\theta;P)$ indicates that the prediction for $\mathbf{x}_i$ deviates substantially from those of similar instances, suggesting a potential violation of local predictive consistency.

Although this residual is mathematically related to local smoothing, its role here is fairness-oriented. The neighborhood defines the comparison set required by the individual fairness principle, and the local reference prediction represents the prediction pattern among comparable individuals. The residual therefore measures departure from a local fairness standard rather than merely enforcing statistical smoothness.

It is straightforward to see that enforcing the local pairwise fairness constraint in \cref{eq:local_lipschitz} also provides control over the instance-level individual fairness score defined in \cref{eq:instance_ind_fairness_score}.

\begin{proposition}[Local pairwise fairness upper-bounds instance-level IF score]
\label{prop:pairwise-controls-if}
Let $P\in\mathbb{R}^{n\times n}$ be a row-stochastic matrix with $P_{ij}=0$ whenever
$(i,j)\notin \mathcal{E}_\tau$, where $\mathcal{E}_\tau$ is defined in 
\cref{eq:local_edge_set}. Suppose that the local pairwise fairness condition
\[
|f_\theta(\mathbf{x}_i)-f_\theta(\mathbf{x}_j)| \le c\, d_x(\mathbf{x}_i, \mathbf{x}_j),
\qquad \forall (i,j)\in \mathcal{E}_\tau,
\]
holds. Then, for every instance $i$,
\[
\mathrm{IF}_i(\theta;P) \le c\,\tau .
\]
\end{proposition}

The proof is provided in Appendix~\ref{app:proof-pairwise-controls-if}.

Proposition~\ref{prop:pairwise-controls-if} shows that local predictive consistency is formally connected to the standard local pairwise fairness criterion. If all neighboring pairs satisfy the local Lipschitz condition, then every instance-level residual is also bounded. Hence, the residual score is not an unrelated graph-smoothing heuristic: it is controlled by the classical pairwise notion of individual fairness.

At the same time, the converse need not hold. A small residual relative to the neighborhood average does not guarantee that all pairwise discrepancies within the neighborhood are small. In particular, deviations from different neighbors may partially cancel when forming the weighted neighborhood average, so a small aggregated residual may coexist with larger individual pairwise differences. Local predictive consistency should therefore be understood as a node-level relaxation of pairwise fairness rather than an equivalent characterization. This relaxation is intentional: it provides a less restrictive, node-level fairness diagnostic when the goal is to identify individuals whose predictions are locally inconsistent, rather than to enforce every edge-wise constraint exactly.

Perfect local consistency corresponds to the fixed-point condition $\mathbf{f}_\theta = P\mathbf{f}_\theta$ on the dataset, that is, $f_\theta(\mathbf{x}_i) = (P\mathbf{f}_\theta)_i$ for all $i$. In this case, the prediction vector $\mathbf f_\theta$ satisfies a discrete mean-value property induced by $P$, often referred to as $P$-harmonicity. The local predictive-consistency score therefore measures the extent to which predictions deviate from this local harmonic condition. This perspective connects naturally to recent work on harmonic machine learning models~\citep{kersting2024harmonic}.

\paragraph{Graph-based construction of the averaging operator.}
The local predictive-consistency score requires a neighborhood operator $P$ that defines the local comparison set for each individual. This requirement is shared by graph-based and metric-based formulations of individual fairness more generally: the individual fairness principle becomes operational only after a task-relevant notion of similarity has been specified. The graph therefore plays a substantive role in the formulation, since it encodes which individuals are considered comparable for fairness assessment.

The use of similarity graphs is a standard way to operationalize local smoothness and local individual fairness. In graph-based semi-supervised learning and manifold regularization, graph Laplacian penalties are widely used to encourage nearby instances to receive similar labels or predictions~\citep{belkin2006manifold,ando2007learning}. In individual fairness, this construction has a direct fairness interpretation: the graph specifies which individuals are considered comparable, and graph-based penalties encourage similar predictions along the resulting comparison structure~\citep{dwork2012fairness,petersen2021post}. Such formulations are also computationally attractive because they restrict the set of fairness-relevant comparisons from all pairs of instances to a sparse set of local edges, while preserving the principle that similar individuals should be treated similarly.

To construct this operator, we represent each instance $\mathbf{x}_i$ by a vector $\mathbf{z}_i \in \mathbb{R}^d$. The representation $\mathbf{z}_i$ may correspond to the original covariates, a preprocessed feature representation, or a learned embedding. In our experiments, the graph is constructed in the same projected and standardized representation space used for fairness regularization, so that the predictive model and the fairness operator rely on a consistent similarity structure. This representation-based construction follows the common assumption in individual fairness that similarity should be measured in a task-relevant space~\citep{ilvento2019metric,mukherjee2020two}.

For a locality threshold $\tau>0$, define the neighborhood of instance $i$ as
\[
\mathcal{N}_\tau(i)
=
\{ j : \| \mathbf{z}_i - \mathbf{z}_j \|_2 \le \tau \}.
\]
We then define kernel weights
\[
W_{ij}
=
\exp\!\left(-\gamma \| \mathbf{z}_i - \mathbf{z}_j \|_2^2\right)
\mathbb{I}\{j \in \mathcal{N}_\tau(i)\},
\qquad \gamma>0,
\]
where $\gamma$ controls the decay of similarity with distance. Let $D$ denote the diagonal degree matrix with entries $D_{ii}=\sum_{j=1}^n W_{ij}$. The averaging operator is obtained by row normalization:
\begin{equation}
\label{eq:P-row-norm}
P = D^{-1}W.
\end{equation}
Thus, each row of $P$ defines a probability distribution over the local comparison set of instance $i$, and $(P\mathbf f_\theta)_i$ is the corresponding similarity-weighted local reference prediction.

The threshold $\tau$ determines the scale at which individuals are considered
comparable. Graph-based formulations are a standard way to make this comparison structure explicit in individual fairness and related graph-regularization methods~\citep{dwork2012fairness,lahoti2019operationalizing,petersen2021post,kang2020inform}. If $\tau$ is too small, the graph may contain many isolated or weakly connected nodes, making the local fairness score uninformative for those instances. If $\tau$ is too large, the graph may connect individuals that are not meaningfully similar, weakening the interpretation of local predictive consistency. We therefore treat $\tau$ as a graph-scale parameter and assess its consequences using diagnostics such as graph coverage, isolated-node rates and graph density. Appendix~\ref{app:threshold-sensitivity} examines how predictive performance, graph structure, and local inconsistency vary across empirical distance thresholds.

\paragraph{Isolated nodes and graph coverage.}

Self-loops are included to keep the row-stochastic averaging operator well-defined. For a node whose only connection is its self-loop, $(P\mathbf f_\theta)_i=f_\theta(\mathbf x_i)$, and hence the residual is zero by construction. These nodes remain part of the supervised training objective, but they do not provide a non-self local comparison. This zero residual is mathematically consistent with the definition, although its interpretation differs from that of a zero residual obtained for a node with one or more non-self comparators.

Accordingly, graph coverage should be considered alongside the residual-based fairness assessment. Our evaluation protocol reports graph coverage separately and additionally provides fairness summaries conditional on covered, non-isolated nodes; see Section~\ref{subsec:evaluation_protocols} and Appendix~\ref{app:threshold-sensitivity}. Additional implementation details are provided in Appendix~\ref{app:implementation}.

\paragraph{Cohort dependence and deployment.}
Local predictive consistency is defined relative to a comparison graph and is therefore cohort-dependent rather than an intrinsic property of an individual in isolation. Changing the reference population can alter an individual's neighborhood and hence its residual even when the individual's prediction is unchanged. In our empirical evaluation, the comparison graph is constructed within the held-out test cohort. For single-instance deployment, a fixed reference population can instead be used, with the new instance connected to reference individuals using the same representation, similarity function, and graph-construction parameters. The resulting score should therefore always be interpreted relative to the specified comparison graph and reference population.

\subsection{Fairness-Regularized Learning}
\label{subsec:fairness-learning}

The instance-level individual fairness score in \cref{eq:instance_ind_fairness_score} can be incorporated into model training as a regularization term. Aggregating the squared residuals $r_i(\theta;P)$ from \cref{eq:residual} yields the fairness regularizer
\begin{equation}
\label{eq:omega}
\Omega_{\mathrm{IF}}(\theta;P)
:= \frac{1}{n} \sum_{i=1}^n r_i(\theta;P)^2
= \frac{1}{n} \|(I-P)\mathbf f_\theta\|_2^2.
\end{equation}

Model parameters $\theta$ are learned by minimizing the task loss $\mathcal{L}_{\mathrm{task}}$ augmented with the fairness regularization:
\begin{equation}
\label{eq:training-objective}
\min_{\theta} \quad
\mathcal{L}_{\mathrm{task}}(\theta)
+ \lambda \, \Omega_{\mathrm{IF}}(\theta;P),
\end{equation}
where $\lambda \ge 0$ controls the trade-off between performance and fairness.

We refer to this row-normalized graph-based local predictive-consistency regularization formulation as URW. When $P$ is constructed from a similarity graph via $P=D^{-1}W$, define
\[
L_{\mathrm{URW}} := I-D^{-1}W = I-P,
\]
which is the row-random-walk Laplacian. With this notation, the fairness regularizer can be written compactly as
\[
\Omega_{\mathrm{IF}}(\theta;P)
=
\frac{1}{n}\|L_{\mathrm{URW}}\mathbf f_\theta\|_2^2
=
\frac{1}{n}
\mathbf f_\theta^\top
L_{\mathrm{URW}}^\top L_{\mathrm{URW}}
\mathbf f_\theta.
\]

\paragraph{Relation to standard Laplacian regularization.}
For the pairwise individual fairness definition in \cref{eq:local_lipschitz}, one can incorporate it in model training using the standard (combinatorial) graph Laplacian $L_{\mathrm{comb}} = D - W$:
\[
\mathbf{f}_\theta^\top L_{\mathrm{comb}}\, \mathbf{f}_\theta
= \frac{1}{2} \sum_{i,j} W_{ij} \bigl(f_\theta(\mathbf{x}_i) - f_\theta(\mathbf{x}_j)\bigr)^2,
\]
which penalizes \emph{edge-wise} discrepancies in predictions.  These two Laplacians are closely related. Since $P = D^{-1} W$,
\[
L_{\mathrm{URW}} = I - D^{-1} W = D^{-1} (D - W) = D^{-1} L_{\mathrm{comb}},
\]
and it follows that
\[
\Omega_{\mathrm{IF}}(\theta;P)
= \frac{1}{n} \| D^{-1} L_{\mathrm{comb}} \mathbf{f}_\theta \|_2^2.
\]

\subsection{$p$-Norm Regularization for Individual Fairness}
\label{subsec:pnorm}

The standard squared residuals in \cref{eq:omega} treat all deviations equally. In practice, however, large violations - where a model's prediction deviates substantially from the local average - may be especially concerning. It can therefore be beneficial to place greater emphasis on these extreme deviations, effectively penalizing the worst-case disparities more strongly. Formally, for $p \ge 1$, we define
\begin{equation}
\label{eq:omega-p}
\Omega_{\mathrm{IF}}^{(p)}(\theta;P)
:= \left( \frac{1}{n} \sum_{i=1}^n |r_i(\theta;P)|^p \right)^{2/p},
\end{equation}
where $r_i(\theta;P)$ is the instance-level residual defined in \cref{eq:residual}. When $p=2$, this reduces exactly to \cref{eq:omega}. As $p$ increases, the objective becomes increasingly sensitive to large residuals, making it particularly effective in settings where extreme local violations are more critical to avoid. In the limit $p \to \infty$, $\Omega_{\mathrm{IF}}^{(p)}(\theta;P)$ converges to $\max_{i} |r_i(\theta;P)|^2$, corresponding to direct control over the single largest local violation.  

This $p$-norm formulation therefore provides a flexible tool to balance overall fairness with robustness to extreme individual discrepancies. It allows practitioners to tune the regularizer according to the fairness requirements of their specific application, prioritizing either average behavior ($p=2$) or worst-case violations ($p \gg 1$). Further technical discussion and examples are provided in Appendix~\ref{app:pnorm}.

\section{Task Instantiations}
\label{sec:task-instantiations}
We apply the individual fairness regularized learning approach introduced in Section~\ref{subsec:fairness-learning} to three settings: regression, binary classification, and time-to-event prediction. In the regression setting, where the output space is $\mathcal{Y} = \mathbb{R}$, the formulation described in Section~\ref{subsec:fairness-learning} can be applied directly. For binary classification and time-to-event prediction, however, certain adaptations are required.

\subsection{Regression}
\label{subsec:regression}
Given training data $\{(\mathbf{x}_i, y_i)\}_{i=1}^n$ for a regression task, with outputs $y_i \in \mathbb{R}$, we denote the response vector as 
\(
\mathbf{y} := (y_1, \ldots, y_n)^\top.
\)
In this setting, the instance-level individual fairness score is 
$\mathrm{IF}_i(\theta;P) = |r_i(\theta;P)|$, where $r_i(\theta;P)$ is defined 
in \cref{eq:residual}. Using the training objective in \cref{eq:training-objective} 
with mean squared error, the fairness-regularized regression problem becomes
\begin{equation}
\label{eq:reg-objective}
\min_{\theta} \quad
\underbrace{\frac{1}{n} \|\mathbf{f}_\theta - \mathbf{y}\|_2^2}_{\mathcal{L}_{\mathrm{MSE}}(\theta)}
+
\lambda \underbrace{\frac{1}{n} \|\mathbf{f}_\theta - P\mathbf{f}_\theta\|_2^2}_{\Omega_{\mathrm{IF}}(\theta;P)}.
\end{equation}

\subsection{Binary Classification}
\label{sec:task-instantiations-binary}

We consider binary classification on labeled data
$\{(\mathbf{x}_i, y_i)\}_{i=1}^n$ with $y_i \in \{0,1\}$.
A probabilistic classifier
$f_\theta : \mathbb{R}^d \to (0,1)$ outputs
\[
f_\theta(\mathbf{x}_i) = \Pr(Y=1 \mid \mathbf{x}_i).
\]

For binary classification, we consider two instantiations of the URW formulation, denoted URW-FKL and URW-RKL, where KL refers to the Kullback--Leibler divergence. The two variants use the same comparison graph, neighborhood operator, and regularization principle, and differ only in how the neighborhood-aggregated predictive reference is constructed. URW-FKL uses the Forward-KL barycenter, corresponding to probability-space averaging, whereas URW-RKL uses the Reverse-KL barycenter, corresponding to logit-space averaging.

Accordingly, to define the fairness regularizer $\Omega_{\mathrm{IF}}(\theta;P)$ using the neighborhood operator $P$, we consider these two constructions of the neighborhood reference prediction. Forward-KL preserves linear averaging of probabilities, while Reverse-KL performs the aggregation in logit space and therefore induces a different geometry for predictions near the boundaries of the probability interval.

\paragraph{Forward-KL (probability-space) averaging.}

For each instance $i$, URW-FKL defines the neighborhood reference as the solution to
\[
\bar f_{\theta,i}
=
\arg\min_{q \in (0,1)}
\sum_{j=1}^n P_{ij}\,
\mathrm{KL}\!\Bigl(
\mathrm{Bern}\bigl((\mathbf{f}_\theta)_j\bigr)
\;\|\;
\mathrm{Bern}(q)
\Bigr),
\]
which admits the closed-form solution (derivation in Appendix~\ref{app:forward-kl}):
\[
\bar f_{\theta,i} = (P \mathbf{f}_\theta)_i = \sum_{j=1}^n P_{ij}\, (\mathbf{f}_\theta)_j.
\]
Thus, Forward-KL recovers the linear neighborhood averaging used in the Euclidean formulation of Section~\ref{subsec:proposed-if-score}. Defining $\bar{\mathbf{f}}_\theta := (\bar f_{\theta,1}, \ldots, \bar f_{\theta,n})^\top$, the corresponding URW fairness regularizer can be written as
\[
\Omega_{\mathrm{IF}}(\theta;P) = \frac{1}{n} \| \mathbf{f}_\theta - \bar{\mathbf{f}}_\theta \|_2^2.
\]

\paragraph{Reverse-KL (logit-space) averaging.}

For each instance $i$, URW-RKL defines the neighborhood reference as the solution to
\[
\tilde f_{\theta,i}
=
\arg\min_{q \in (0,1)}
\sum_{j=1}^n P_{ij}\,
\mathrm{KL}\!\Bigl(
\mathrm{Bern}(q)
\;\|\;
\mathrm{Bern}\bigl((\mathbf{f}_\theta)_j\bigr)
\Bigr),
\]
which admits the closed form
\[
\tilde f_{\theta,i}
=
\sigma\!\Bigl(
\sum_{j=1}^n P_{ij}\,
\log\frac{(\mathbf{f}_\theta)_j}{1-(\mathbf{f}_\theta)_j}
\Bigr),
\qquad
\sigma(t) = \frac{1}{1+e^{-t}}.
\]
This corresponds to averaging neighbor predictions in logit space. A derivation is provided in Appendix~\ref{app:reverse-kl}. In vector form, we have
\[
\tilde{\mathbf f}_\theta = \sigma\bigl(P\,\mathrm{logit}(\mathbf f_\theta)\bigr),
\]
where $\mathrm{logit}(p)=\log\frac{p}{1-p}$ and both $\mathrm{logit}$ and $\sigma$ act elementwise. In practice, for numerical stability, we clip probabilities to $[\varepsilon, 1-\varepsilon]$ before applying $\mathrm{logit}(\cdot)$, for a small $\varepsilon>0$.
The corresponding URW fairness regularizer is
\[
\Omega_{\mathrm{IF}}(\theta;P) = \frac{1}{n} \| \mathbf{f}_\theta - \tilde{\mathbf{f}}_\theta \|_2^2.
\]

\paragraph{Predictive loss.}
We use the binary cross-entropy loss
\begin{equation}
\mathcal{L}_{\mathrm{CE}}(\theta)
=
-\frac{1}{n}\sum_{i=1}^n
\Bigl(
y_i \log f_\theta(\mathbf{x}_i)
+
(1-y_i) \log \bigl(1 - f_\theta(\mathbf{x}_i)\bigr)
\Bigr).
\end{equation}
Combining $\mathcal{L}_{\mathrm{CE}}(\theta)$ with the fairness regularizer 
$\Omega_{\mathrm{IF}}(\theta;P)$ - computed via either Forward-KL or Reverse-KL 
averaging - yields the fairness-regularized training objective for binary classification.

\subsection{Time-to-event prediction}
\label{subsec:task-instantiations-survival}
We consider right-censored survival data $\{(\mathbf{x}_i, t_i, e_i)\}_{i=1}^n$, where $\mathbf{x}_i$ denotes covariates, $t_i > 0$ is the observed follow-up time, and $e_i \in \{0,1\}$ is the event indicator ($e_i = 1$ if the event occurs at $t_i$ and $e_i = 0$ if the observation is censored). We instantiate the local-consistency formulation for two common survival modeling paradigms: (i) models that output a \emph{scalar risk score}, and (ii) models that directly output a \emph{survival function}.

\paragraph{Log-risk prediction.}
We first consider Cox proportional hazards models~\citep{cox1972regression} with a neural network parameterization (DeepSurv)~\citep{katzman2018deepsurv}. The hazard function is modeled as
\[
h(t \mid \mathbf{x}) = h_0(t) \exp\!\bigl(f_\theta(\mathbf{x})\bigr),
\]
where $h_0(t)$ is an unspecified baseline hazard and $f_\theta(\mathbf{x}) \in \mathbb{R}$ is the predicted \emph{log-risk} score. This parameterization is commonly used because differences $f_\theta(\mathbf{x}_i) - f_\theta(\mathbf{x}_j)$ correspond to log hazard ratios, and the predictive loss function can often be specified such that it depends only on the risk scores without requiring explicit modeling of $h_0(t)$. In this setting, given an averaging operator $P$, the individual fairness regularizer can be defined as
\begin{equation}
\label{eq:risk-fair-main}
\Omega_{\mathrm{IF}}(\theta;P) = \frac{1}{n} \| \mathbf{f}_\theta - P \mathbf{f}_\theta \|_2^2.
\end{equation}

For the predictive loss in Cox models, let $R_i = \{j : t_j \ge t_i \}$ denote the risk set at time $t_i$, and let $n_e = \sum_{i=1}^n e_i$ be the total number of observed events. The negative Cox partial log-likelihood
\begin{equation}
\label{eq:cox-loss}
\mathcal{L}_{\mathrm{Cox}}(\theta)
=
-\frac{1}{n_e} \sum_{i:\, e_i = 1}
\Biggl[
f_\theta(\mathbf{x}_i)
-
\log \Biggl(
\sum_{j \in R_i} \exp\bigl(f_\theta(\mathbf{x}_j)\bigr)
\Biggr)
\Biggr]
\end{equation}
is minimized. This objective leverages the partial likelihood to focus only on the ordering of event times, without requiring explicit modeling of the baseline hazard $h_0(t)$. 

Combining the fairness regularizer $\Omega_{\mathrm{IF}}(\theta;P)$ in \cref{eq:risk-fair-main} with the Cox partial log-likelihood in \cref{eq:cox-loss} yields the fairness-regularized training objective for log-risk prediction.

\paragraph{Survival-function prediction.}
For survival-function prediction, we use DeepHazard~\citep{rava2020deephazard} as the neural survival-function backbone. We next consider models that directly predict a survival function for a given covariate $\mathbf{x}$:
\[
S_\theta(t \mid \mathbf{x}) = \Pr(T > t \mid \mathbf{x}), \qquad t \ge 0,
\]
where $T$ denotes the event time random variable. For $i = 1, \ldots, n$, we denote
\[
S_{\theta,i}(t) := S_\theta(t \mid \mathbf{x}_i).
\]
Given an averaging operator $P$, we define the neighborhood average at any time $t$ as 

\begin{equation}
\label{eq:curve-bary-main}
\bar S_i(t) = \sum_{j=1}^n P_{ij} \, S_{\theta,j}(t),
\qquad t \ge 0 .
\end{equation}

This is an arithmetic (convex) average at each $t$; consequently, $\bar S_i$ remains a valid survival curve whenever each $S_{\theta,j}$ is valid (i.e., values in $[0,1]$ and non-increasing in $t$). 

The resulting fairness regularizer is then defined as
\begin{equation}
\label{eq:curve-fair-main}
\Omega_{\mathrm{IF}}(\theta;P)
=
\frac{1}{n}\sum_{i=1}^n
\int_0^{T_{\max}}
\bigl(S_{\theta,i}(t)-\bar S_i(t)\bigr)^2 dt,
\end{equation}
where $T_{\max}$ denotes the maximum follow-up time. In practice, the integrals are evaluated on a discrete time grid, and the integral is approximated by a Riemann sum.

To define the predictive loss under right censoring, we consider a discrete time grid
\[
0 = \tau_0 < \tau_1 < \cdots < \tau_K = T_{\max},
\]
so that $S_\theta(\tau_0 \mid \mathbf{x}) = 1$. Let $b_i \in \{1,\ldots,K\}$ denote the index of the interval containing the observed time $t_i$, i.e., $t_i \in (\tau_{b_i-1}, \tau_{b_i}]$. We define the predictive loss as the negative discrete-time log-likelihood
\begin{equation}
\label{eq:surv-curve-loss}
\mathcal{L}_{\mathrm{Surv}}(\theta)
=
-\frac{1}{n} \sum_{i=1}^n
\Bigl[
e_i \log\!\bigl(S_{\theta,i}(\tau_{b_i-1}) - S_{\theta,i}(\tau_{b_i})\bigr)
+
(1-e_i)\log S_{\theta,i}(\tau_{b_i})
\Bigr],
\end{equation}
where the first term corresponds to the probability of experiencing the event in the observed interval, and the second term corresponds to the probability of surviving through the censoring interval.

Combining the predictive loss in \cref{eq:surv-curve-loss} with the fairness regularizer $\Omega_{\mathrm{IF}}(\theta;P)$ in \cref{eq:curve-fair-main} yields the fairness regularized training objective for survival function prediction.

\section{Experiments}
\label{section_experiments}
\subsection{Datasets}
We evaluate our formulation on benchmark datasets covering binary classification, regression, and survival analysis (Table~\ref{tab:datasets}). For classification, we use the \textsc{Adult} dataset to predict whether income exceeds \$50K and \textsc{COMPAS}~\citep{angwin2016compas} to predict two-year recidivism risk. For regression, we consider \textsc{Employee Salary}~\citep{employerdata_kaggle} and \textsc{Insurance}~\citep{suess2017insurance}, which predict employee salaries and individual medical insurance charges, respectively. Survival datasets include \textsc{FLCHAIN}~\citep{dispenzieri2012use}, \textsc{WHAS}~\citep{hosmer1999regression}, and \textsc{ROSSI}~\citep{rossi1980money,fox2012rcmdrplugin}.

\begin{table}[htbp]
\centering
\footnotesize
\begin{adjustbox}{max width=\linewidth}
\begin{tblr}{
  cells = {c},
  cell{2}{1} = {r=2}{},
  cell{4}{1} = {r=2}{},
  cell{6}{1} = {r=3}{},
  hline{1} = {-}{},
  hline{2} = {-}{Black},
  hline{9} = {1}{},
  hline{9} = {2-5}{Black},
}
Task                  & Dataset  & Size  & {No.\\ Covariates} & {Response\\Variable} \\
Binary Classification & \textsc{Adult}    & 46447 & 21 & Income (binary) \\
                      & \textsc{COMPAS}   & 7214  & 15 & Two-year Recidivism (binary) \\
Regression            & \textsc{Employee Salary} & 10000 & 9 & Salary \\
                      & \textsc{Insurance}       & 1338     & 7  & Medical Charges \\
Survival Analysis     & \textsc{FLCHAIN}  & 7874  & 9 & Mortality \\
                      & \textsc{WHAS}     & 481   & 9 & Mortality \\
                      & \textsc{ROSSI}    & 432   & 7 & Recidivism (arrest) \\
\end{tblr}
\end{adjustbox}
\caption{Datasets used in experiments. Censoring rates: 72.45\% (FLCHAIN), 48.23\% (WHAS), 26.39\% (ROSSI).}
\label{tab:datasets}
\end{table}

\subsection{Implementation and Baselines}
\label{subsec:implementation-baselines}
All methods use the same backbone architecture and preprocessing pipeline to ensure a controlled comparison (implementation details in Appendix~\ref{app:implementation}). We compare against representative individual fairness baselines, including the pre-processing method iFlipper~\citep{zhang2023iflipper}, combinatorial Laplacian regularization~\citep{petersen2021post}, and the adversarial representation-learning approaches SenSR and SenSeI~\citep{yurochkin2019training,yurochkin2020sensei}.

\subsection{Evaluation Protocols}
\label{subsec:evaluation_protocols}

\paragraph{Performance metrics.}
We report task-specific predictive performance metrics. For classification,
we use balanced accuracy and AUC~\citep{hanley1982meaning}, where higher values
indicate better predictive performance. For regression, we report RMSE and
MSE~\citep{willmott2005advantages}, where lower values are preferred. For
survival analysis, we use the concordance index (C-index; higher is better)
and the integrated Brier score (IBS; lower is better)
~\citep{antolini2005time,graf1999assessment,harrell1982evaluating}.

\paragraph{Fairness metrics.}
Our primary fairness measure is the instance-level local predictive-consistency score defined in Section~\ref{subsec:proposed-if-score}. For a tolerance $\epsilon>0$, we summarize violations using
\begin{equation}
\mathrm{VR}(\epsilon)
=
\frac{1}{|\mathcal{I}_{\mathrm{test}}|}
\sum_{i\in\mathcal{I}_{\mathrm{test}}}
\mathbf{1}
\left(
\mathrm{IF}_i(\theta;P)>\epsilon
\right),
\label{eq:vr}
\end{equation}
i.e., the fraction of test instances whose IF score exceeds the tolerance.

All graph-based fairness diagnostics are evaluated relative to the held-out test comparison cohort. Because isolated nodes have no meaningful non-self comparison set, we report graph coverage separately and provide coverage-conditional fairness summaries in Appendix~\ref{app:threshold-sensitivity}. Cohort dependence and single-instance deployment are discussed in Section~\ref{subsec:proposed-if-score}.

To facilitate comparison with prior work, we additionally report GR-consistency and S-consistency following \citet{yurochkin2019training} and \citet{yurochkin2020sensei}.

\paragraph{Independent pairwise fairness diagnostics.}
\label{par:independent-pairwise-diagnostics}

To validate the proposed residual using a fairness quantity independent of the neighborhood-aggregation construction, we additionally evaluate direct prediction discrepancies over non-self edges of the held-out test comparison graph. Our primary pairwise diagnostic is \textsc{EdgeMeanAbs}, which measures the average absolute prediction difference between pairs of instances that are considered comparable by the graph; smaller values therefore indicate stronger pairwise local consistency. Because this diagnostic is based directly on absolute pairwise differences, it is not reduced by signed-neighbor cancellation. The formal definition, additional summaries, task-specific outputs, and complete results are reported in Appendix~\ref{app:pairwise-diagnostics}.

\paragraph{Waller-style local consistency.}
For binary classification, we additionally evaluate the models using the instance-level consistency diagnostic of \citet{waller2025beyond}, defined in Section~\ref{subsec:if-metrics-knn}. We apply the diagnostic to hard predicted
class labels $\hat y_i\in\{0,1\}$ using $k$-nearest-neighbor comparison sets.

Following \citet{waller2025beyond}, we report mean consistency together with the
Proportional Consistency Score (PCS) and Balanced Conditioned Consistency (BCC)
at $\delta=0.5$. We evaluate $k\in\{5,10,15\}$ to assess sensitivity to neighborhood size. These metrics are computed post hoc on the test predictions and are not used in the URW training objective. Complete results are reported in Appendix~\ref{app:waller-comparison}.

\section{Results and Discussion}
\label{section_results}

Our primary comparison is between the row-random-walk regularizer (URW) associated with the local predictive-consistency formulation and the standard combinatorial Laplacian baseline. Table~\ref{tab:baselines} compares these methods with representative individual-fairness baselines on \textsc{Adult}. SenSR achieves the highest balanced accuracy, but with substantially larger violation rates. SenSeI reduces violations, but at a larger predictive cost. At operating points selected to match the validation VR(0.005) of SenSeI, URW-FKL and URW-RKL obtain comparable violation rates while preserving higher balanced accuracy. This suggests that the residual-based formulation can reach competitive individual-fairness operating points without the same loss in predictive performance.

The full \(\lambda\)-sweeps in Figs.~\ref{fig:tradeoffs_binary}--\ref{fig:survival_tradeoff_twomodels} show that increasing \(\lambda\) generally reduces \(\mathrm{VR}(\epsilon)\), although the curves are not perfectly monotone in every dataset. We therefore interpret the $\lambda$-sweeps as sets of attainable operating points rather than as strictly ordered paths. The combinatorial Laplacian directly penalizes edge-wise discrepancies and can attain lower violation rates in some settings, but these operating points are often associated with larger predictive degradation, especially in the regression experiments.

Under this interpretation, URW and the combinatorial Laplacian occupy different regions of the empirical fairness--performance frontier. The combinatorial Laplacian directly penalizes edge-wise discrepancies and can attain lower violation rates in some settings, but often with larger predictive degradation. URW penalizes deviation from a similarity-weighted local neighborhood average and, in several datasets, attains competitive or favorable predictive performance at comparable achieved fairness levels. In binary classification, the Forward-KL and Reverse-KL variants often lead to similar operating points, suggesting that the main effect arises from enforcing local agreement rather than from the specific averaging geometry. Survival results are more dataset-dependent, particularly for \textsc{WHAS} and \textsc{ROSSI}.

Overall, these results support the intended role of local predictive consistency as a node-level relaxation of pairwise individual fairness. URW does not uniformly dominate the combinatorial Laplacian across all datasets or all fairness levels. Rather, the empirical benefit of URW is best understood through attainable operating points and matched-fairness comparisons: in several settings, it provides competitive or favorable points on the fairness--performance frontier while using a node-level residual objective instead of an edge-level pairwise penalty.

Because the local predictive-consistency formulation depends on the neighborhood graph, Appendix~\ref{app:threshold-sensitivity} examines sensitivity to the graph scale. Graph density and isolated-node rates vary systematically with the threshold, while predictive performance is generally stable over moderate graph scales.

Additional appendix analyses report GR- and S-consistency (Appendix~\ref{app:consistency}) and sensitivity to neural network depth (Appendix~\ref{app:depth}).

\begin{figure}[htbp]
\centering
\captionsetup{font=small}
\setlength{\tabcolsep}{2pt}
\renewcommand{\arraystretch}{1.0}
\resizebox{0.9\textwidth}{!}{
\begin{tabular}{@{}
  >{\centering\arraybackslash}m{0.26\textwidth}
  >{\centering\arraybackslash}m{0.26\textwidth}
  >{\centering\arraybackslash}m{0.26\textwidth}
  @{\hspace{0.4em}}
  >{\centering\arraybackslash}m{0.08\textwidth}
@{}}

{\scriptsize Performance} &
{\scriptsize Forward-KL: VR($\varepsilon$)} &
{\scriptsize Reverse-KL: VR($\varepsilon$)} &
\\
\cmidrule(lr){1-3}

\includegraphics[width=\linewidth]{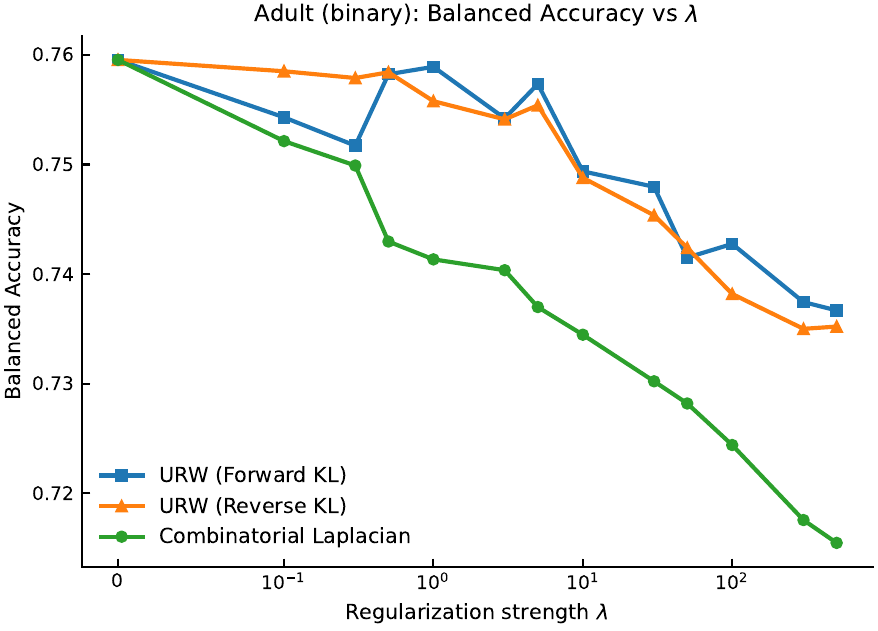} &
\includegraphics[width=\linewidth]{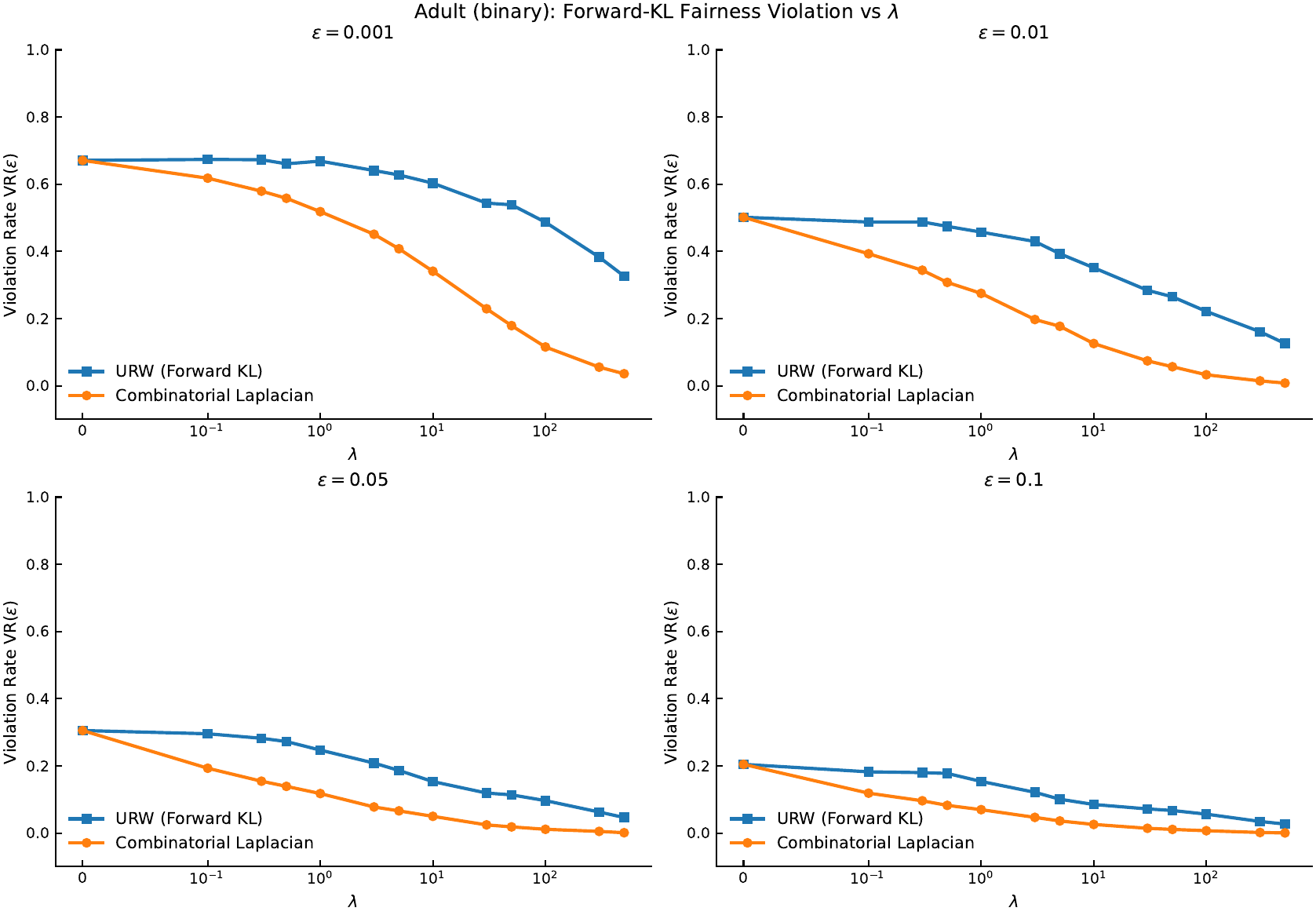} &
\includegraphics[width=\linewidth]{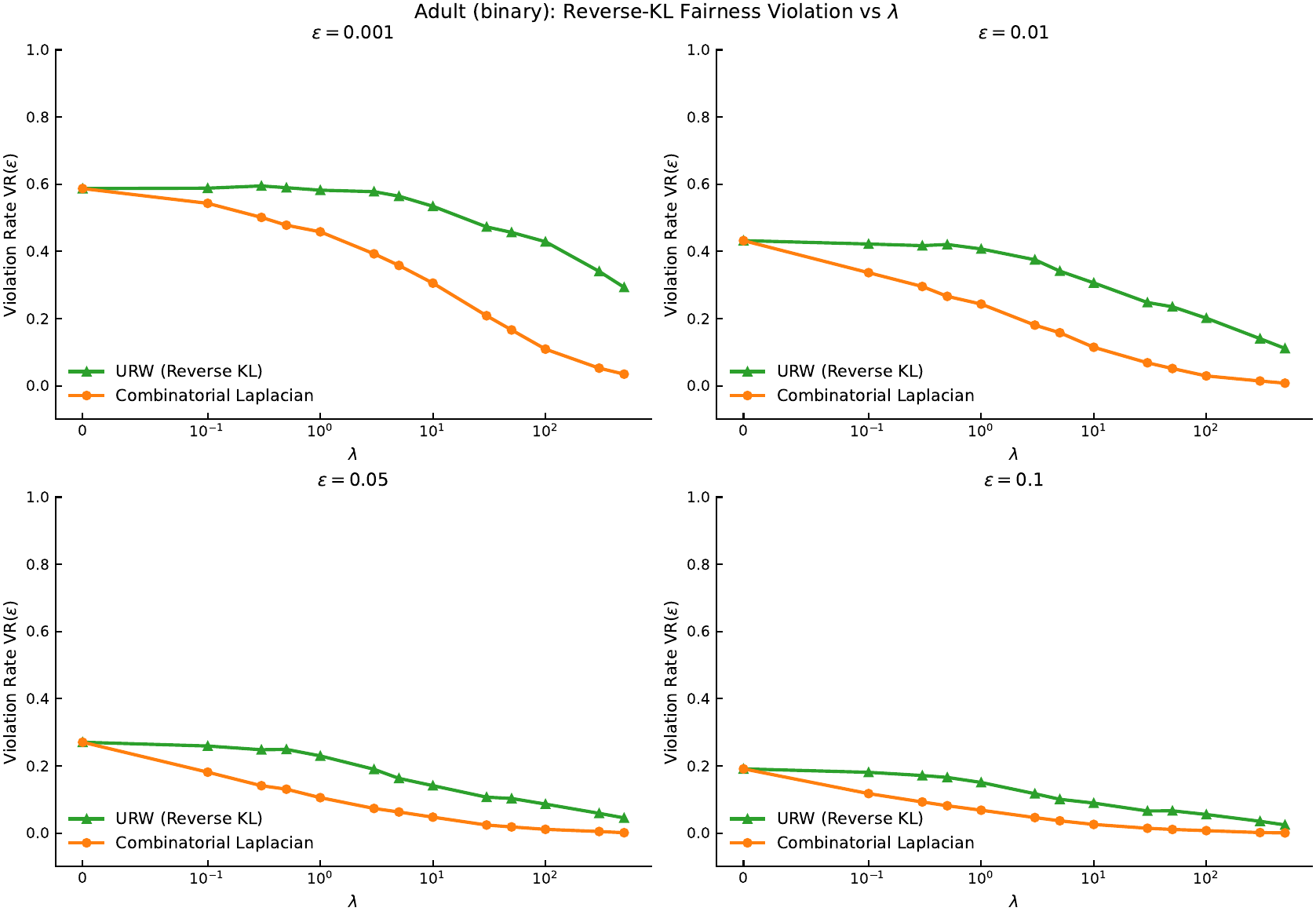} &
\shortstack{\tiny\textsc{Adult}\\ \tiny(binary)} \\
\addlinespace[0.6em]

\includegraphics[width=\linewidth]{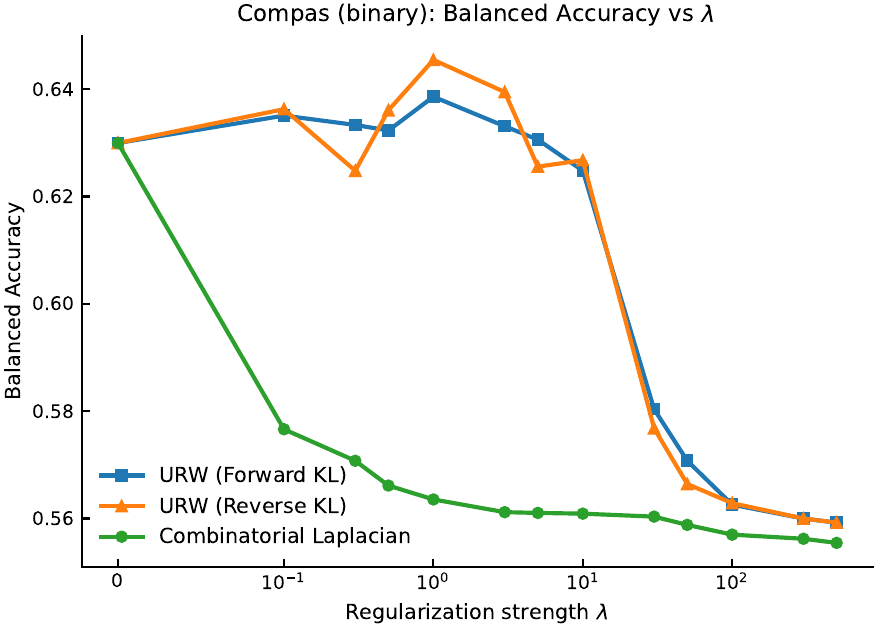} &
\includegraphics[width=\linewidth]{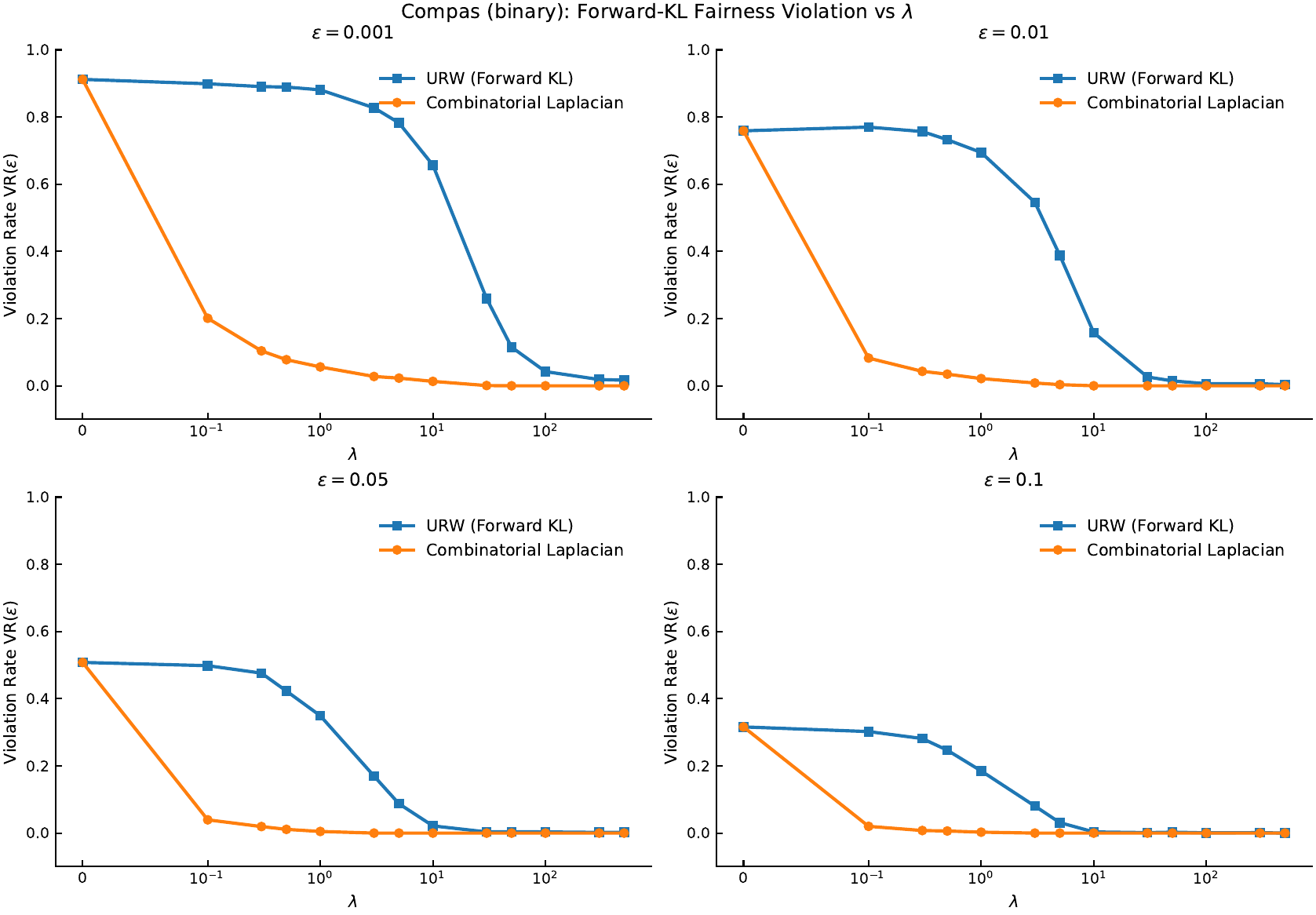} &
\includegraphics[width=\linewidth]{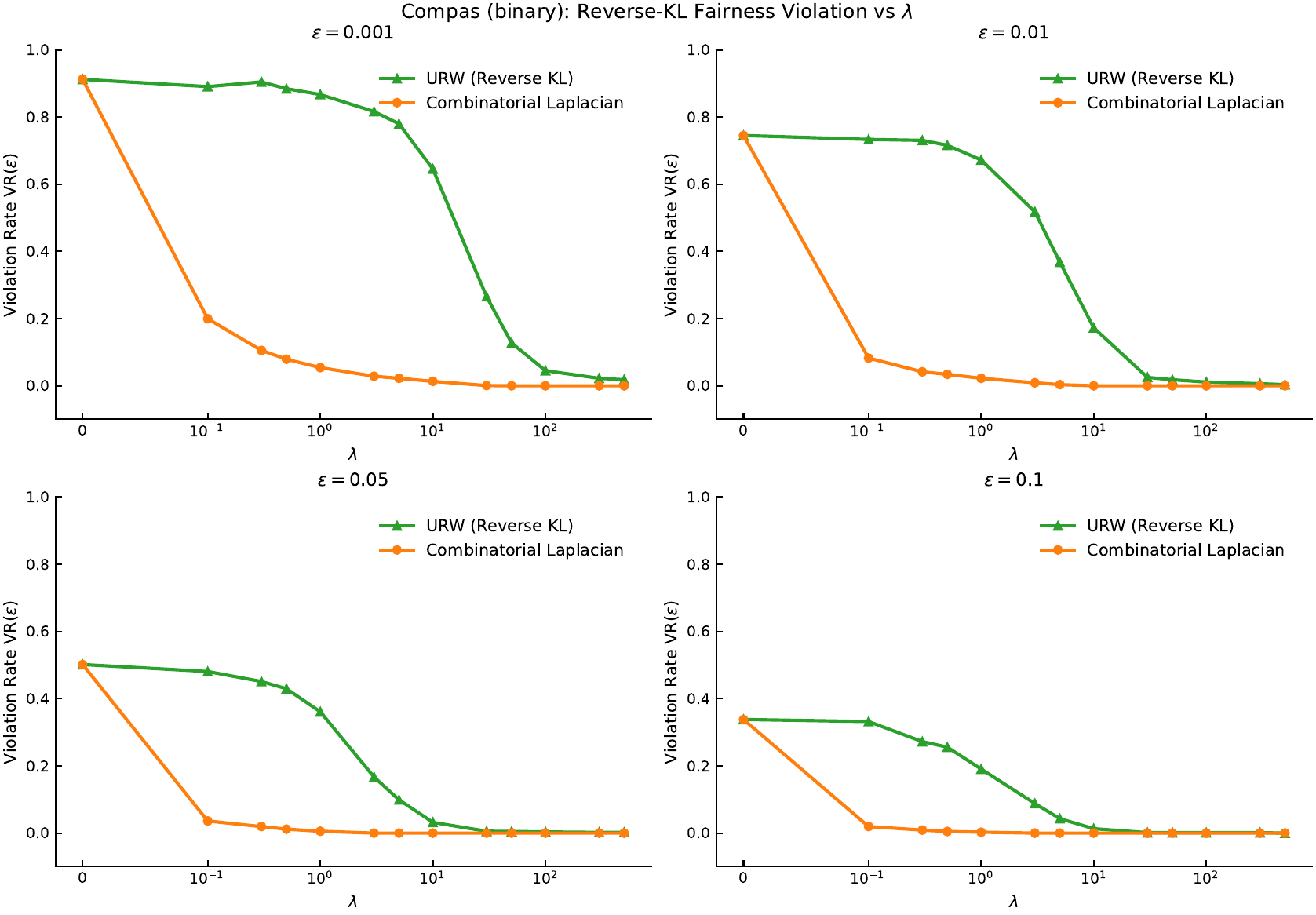} &
\shortstack{\tiny\textsc{COMPAS}\\ \tiny(binary)} \\

\end{tabular}
}
\vspace{0.3em}
\caption{
Performance--fairness operating points for binary classification as $\lambda$ varies.
Different regularization formulations attain different regions of the empirical
fairness--performance space.
}
\label{fig:tradeoffs_binary}
\end{figure}

\begin{figure}[htbp]
\centering
\captionsetup{font=small}
\setlength{\tabcolsep}{2pt}
\renewcommand{\arraystretch}{1.0}
\resizebox{0.9\textwidth}{!}{
\begin{tabular}{@{}
  >{\centering\arraybackslash}m{0.38\textwidth}
  >{\centering\arraybackslash}m{0.38\textwidth}
  @{\hspace{0.6em}}
  >{\centering\arraybackslash}m{0.10\textwidth}
@{}}

{\scriptsize Performance} & {\scriptsize VR($\varepsilon$)} & \\
\cmidrule(lr){1-2}

\includegraphics[width=\linewidth]{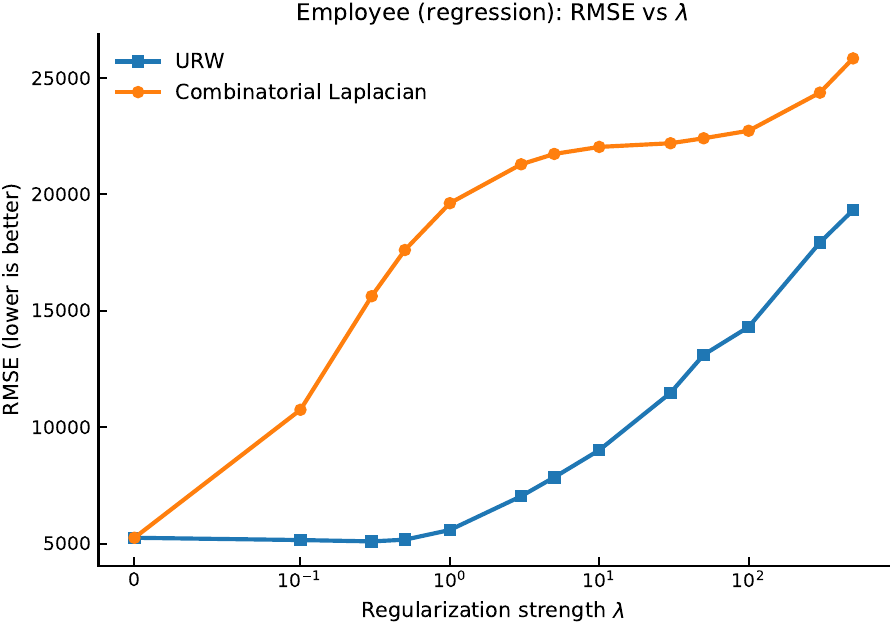} &
\includegraphics[width=\linewidth]{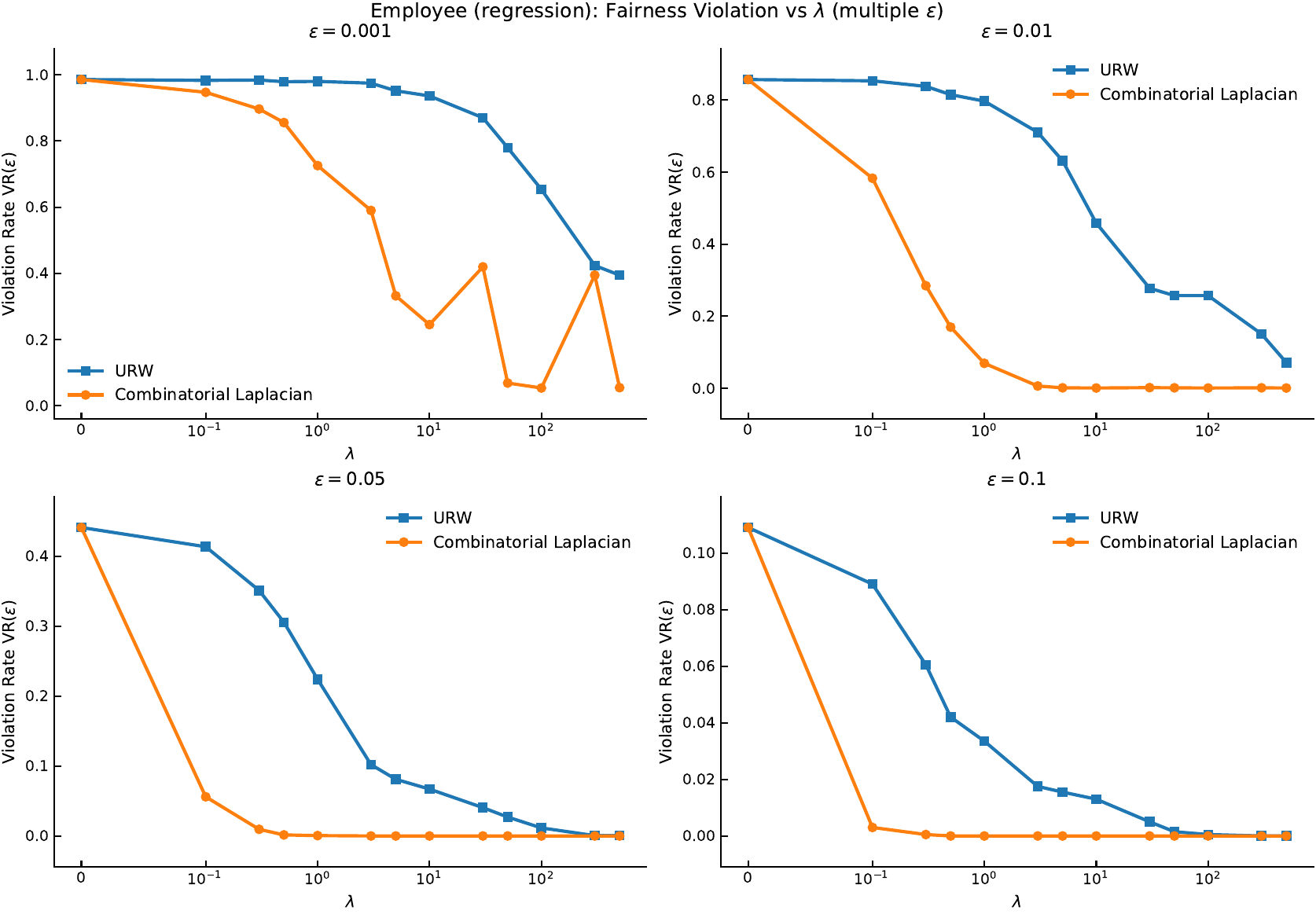} &
\shortstack{\tiny\textsc{Employee Salary}\\ \tiny(regression)} \\
\addlinespace[0.6em]

\includegraphics[width=\linewidth]{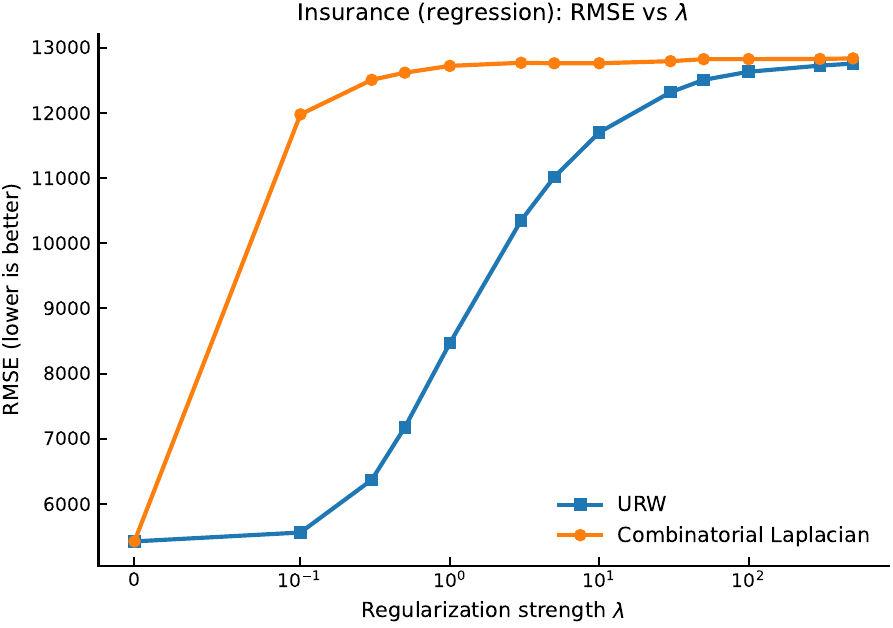} &
\includegraphics[width=\linewidth]{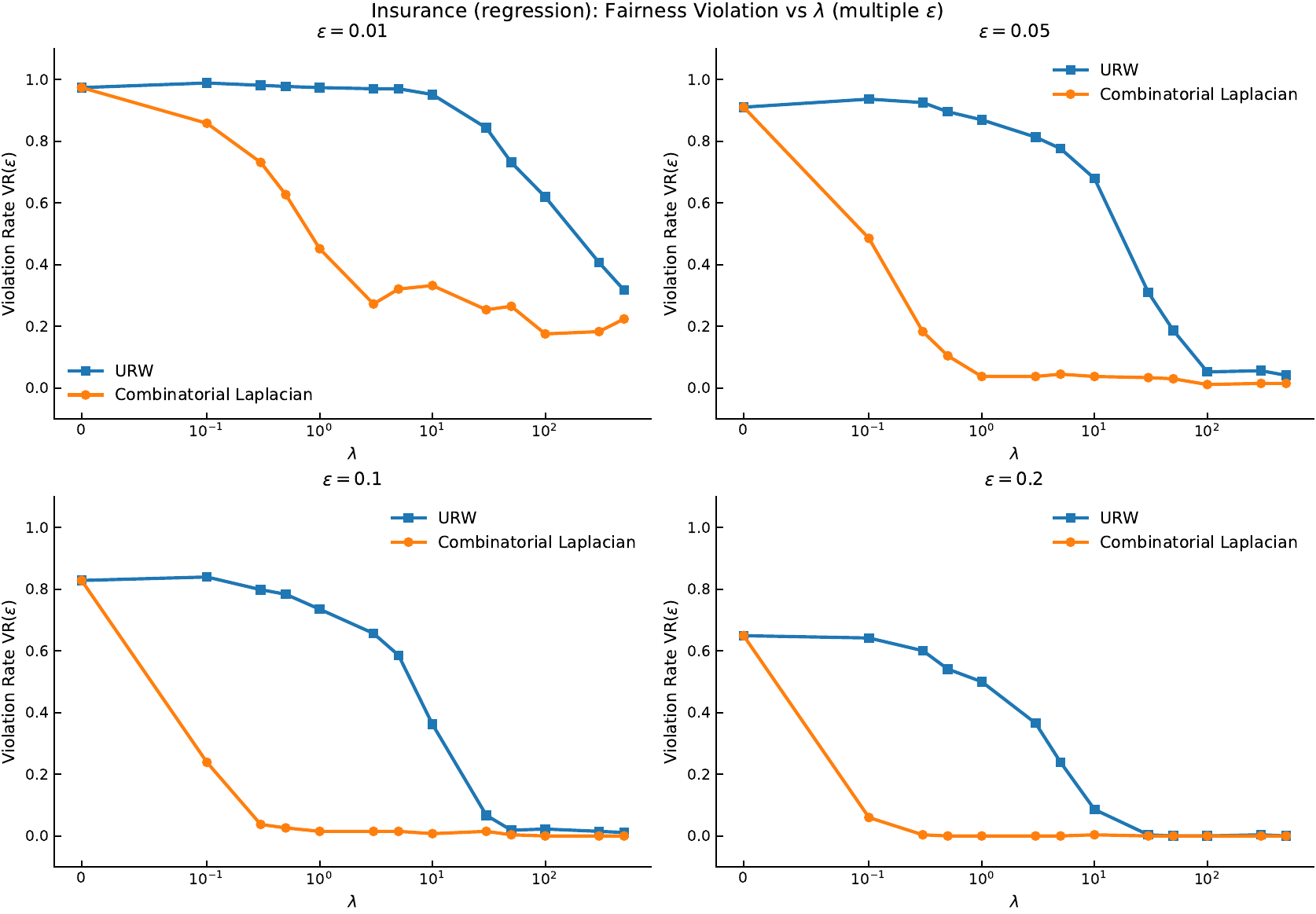} &
\shortstack{\tiny\textsc{Insurance}\\ \tiny(regression)} \\

\end{tabular}
}
\vspace{0.3em}
\caption{
Performance--fairness operating points for regression as $\lambda$ varies. Because the scale of $\lambda$ is regularizer-dependent, comparisons are interpreted through achieved fairness levels rather than raw $\lambda$ values.}
\label{fig:tradeoffs_regression}
\end{figure}

\begin{figure}[htbp]
\centering
\captionsetup{font=small}
\setlength{\tabcolsep}{1.5pt}
\renewcommand{\arraystretch}{1.0}
\resizebox{\textwidth}{!}{
\begin{tabular}{@{}
  *{3}{>{\centering\arraybackslash}m{0.138\textwidth}}
  @{\hspace{0.6em}}
  >{\centering\arraybackslash}m{0.08\textwidth}
  @{\hspace{0.6em}}
  *{3}{>{\centering\arraybackslash}m{0.138\textwidth}}
@{}}

\multicolumn{3}{c}{\small\textsc{DeepSurv} (log-risk)} & &
\multicolumn{3}{c}{\small\textsc{DeepHazard} (survival-function)} \\
\cmidrule(lr){1-3}\cmidrule(lr){5-7}

{\scriptsize IBS} & {\scriptsize C-index} & {\scriptsize VR($\varepsilon$)} &
& {\scriptsize IBS} & {\scriptsize C-index} & {\scriptsize VR($\varepsilon$)} \\
\addlinespace[0.25em]

% ===================== FLCHAIN =====================
\includegraphics[width=\linewidth]{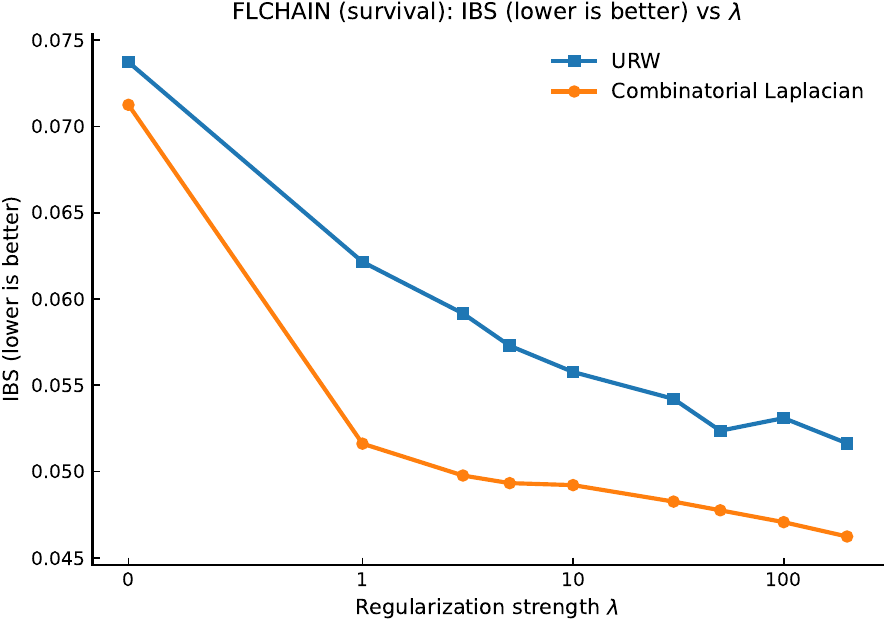} &
\includegraphics[width=\linewidth]{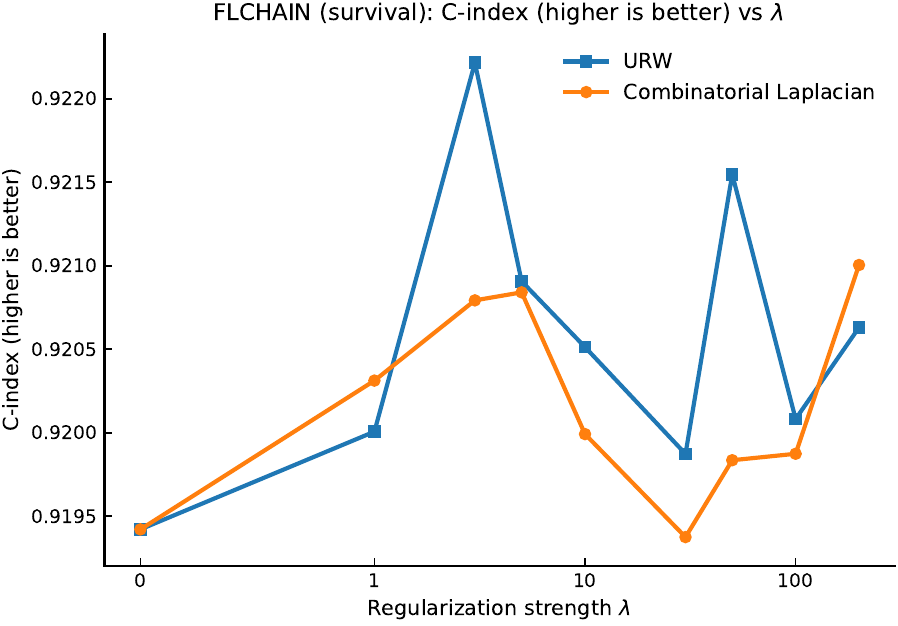} &
\includegraphics[width=\linewidth]{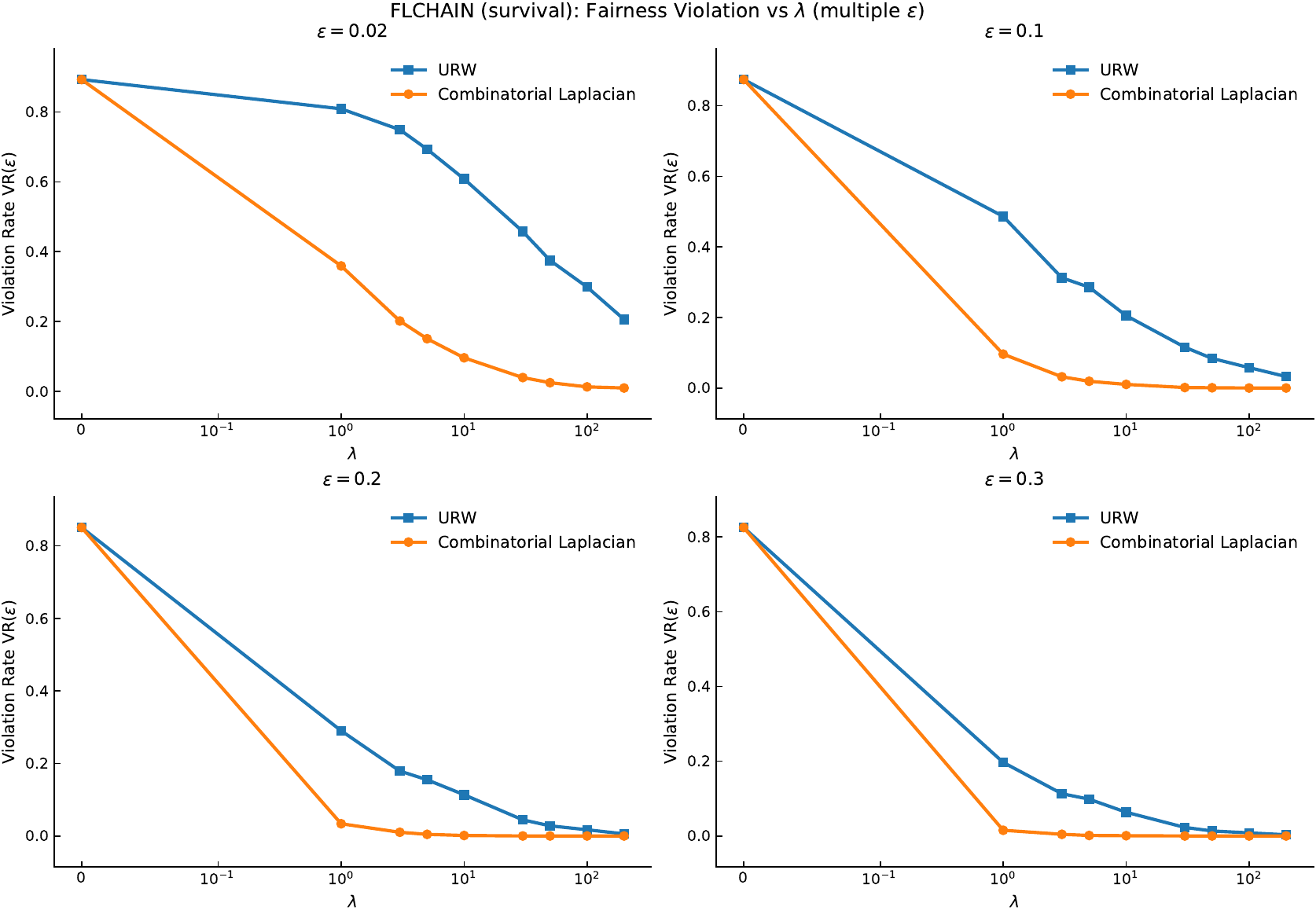} &
\makebox[\linewidth][c]{\tiny\textsc{FLCHAIN}} &
\includegraphics[width=\linewidth]{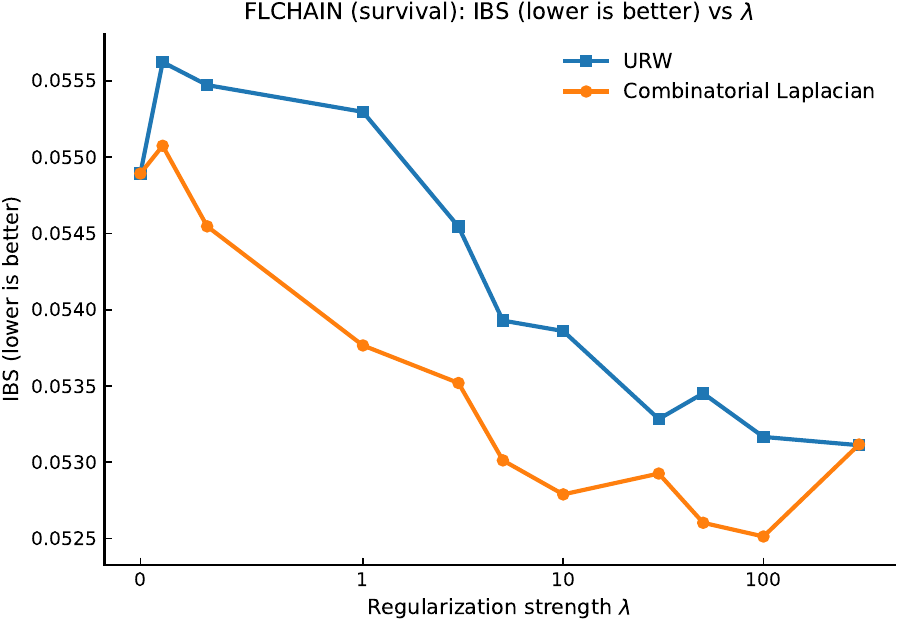} &
\includegraphics[width=\linewidth]{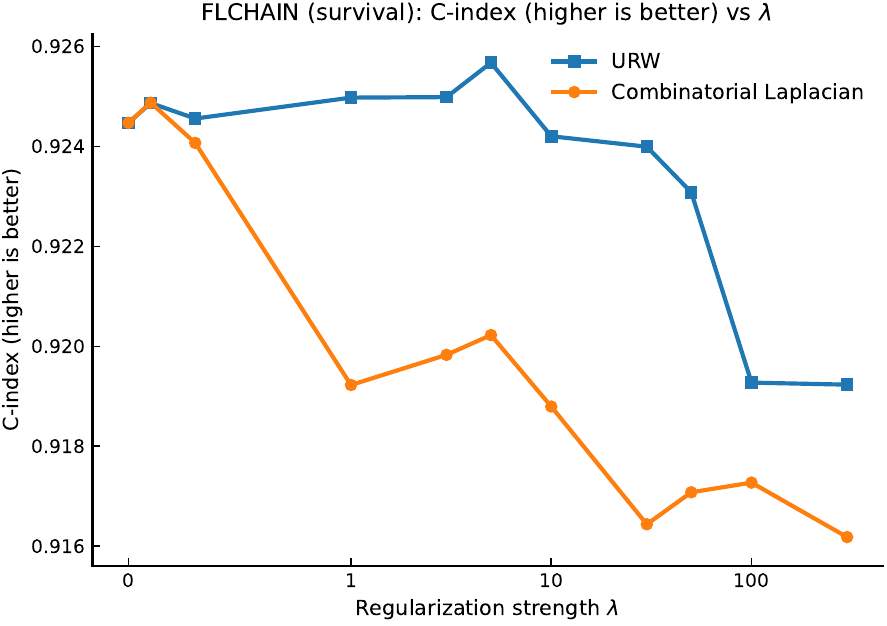} &
\includegraphics[width=\linewidth]{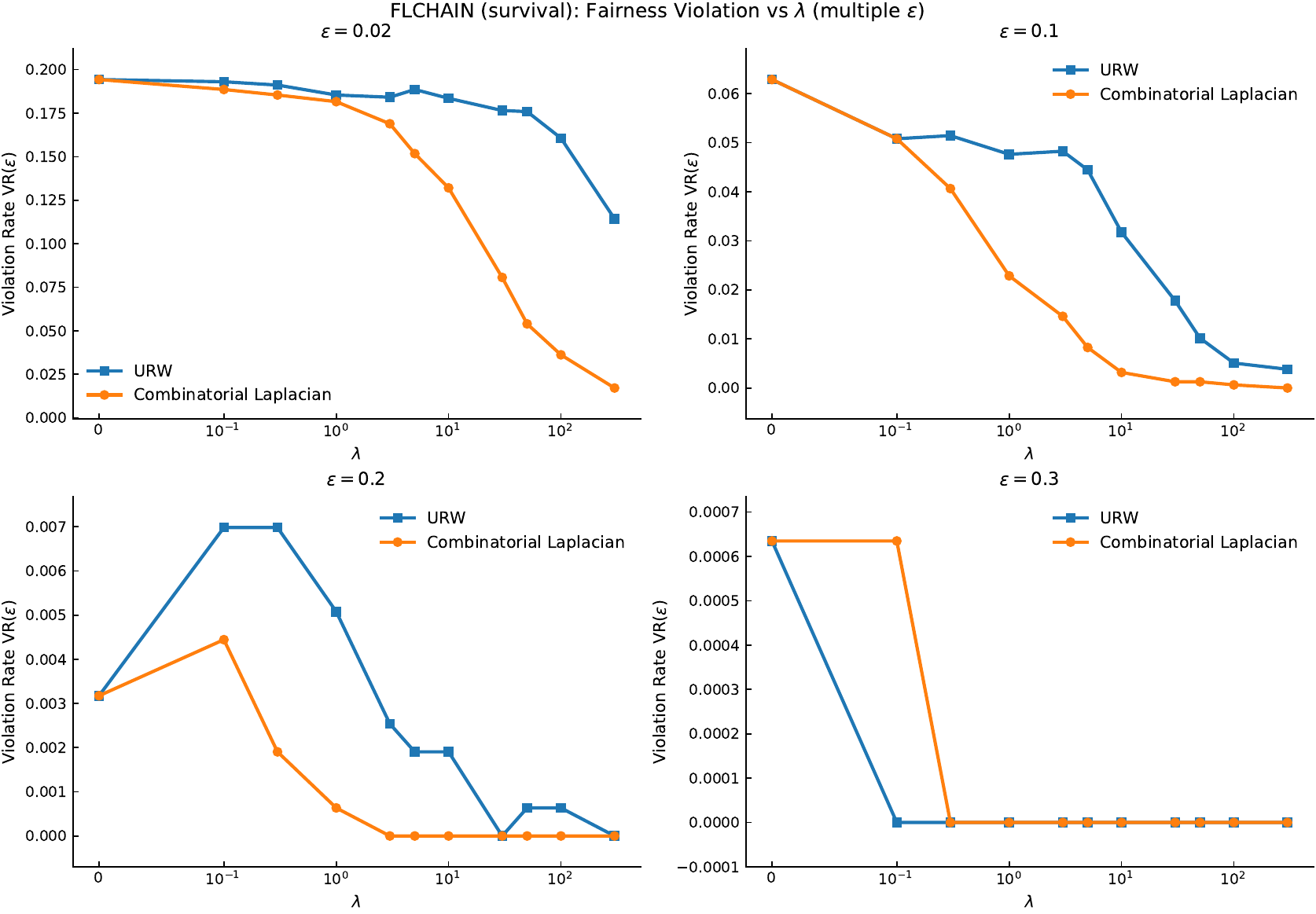} \\
\addlinespace[0.45em]

% ===================== WHAS =====================
\includegraphics[width=\linewidth]{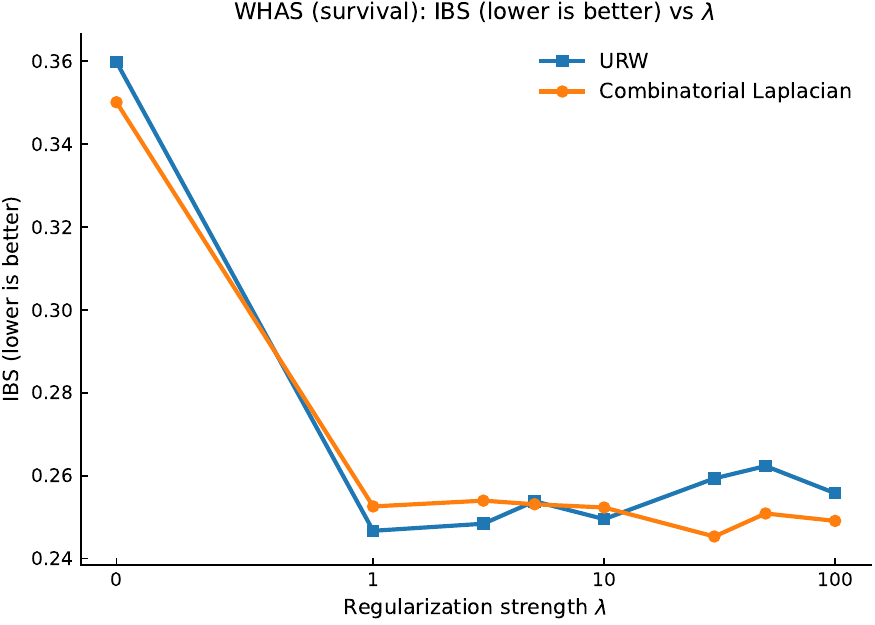} &
\includegraphics[width=\linewidth]{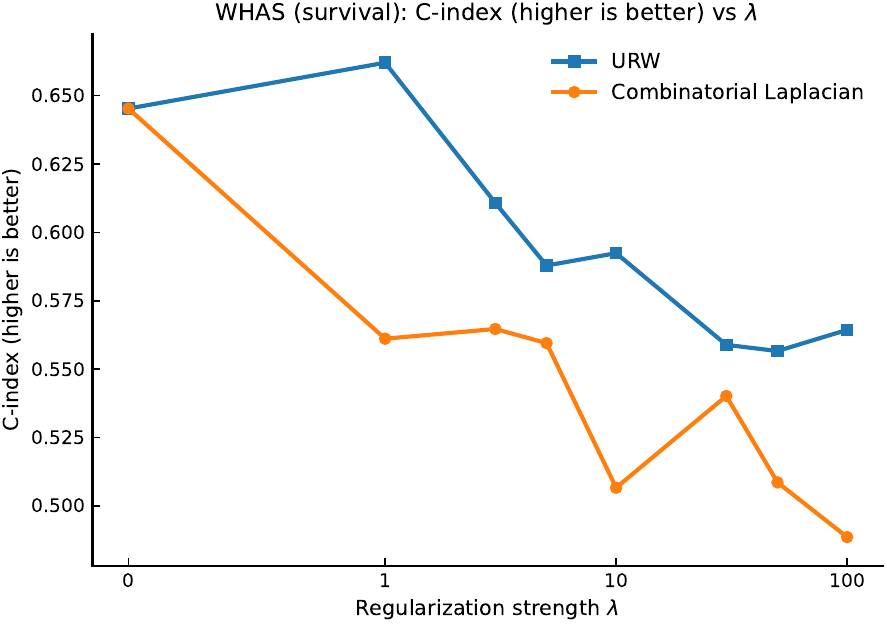} &
\includegraphics[width=\linewidth]{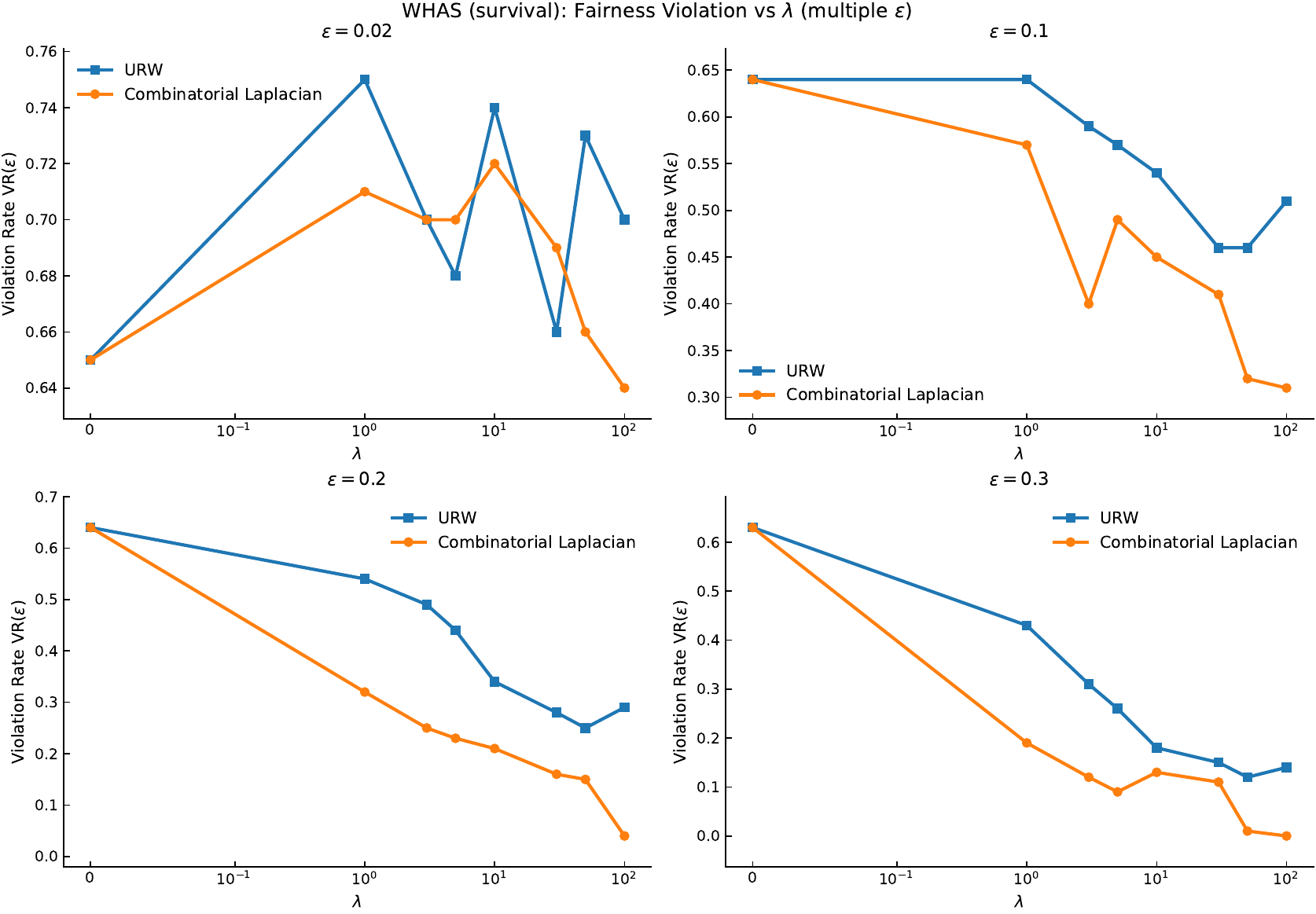} &
\makebox[\linewidth][c]{\tiny\textsc{WHAS}} &
\includegraphics[width=\linewidth]{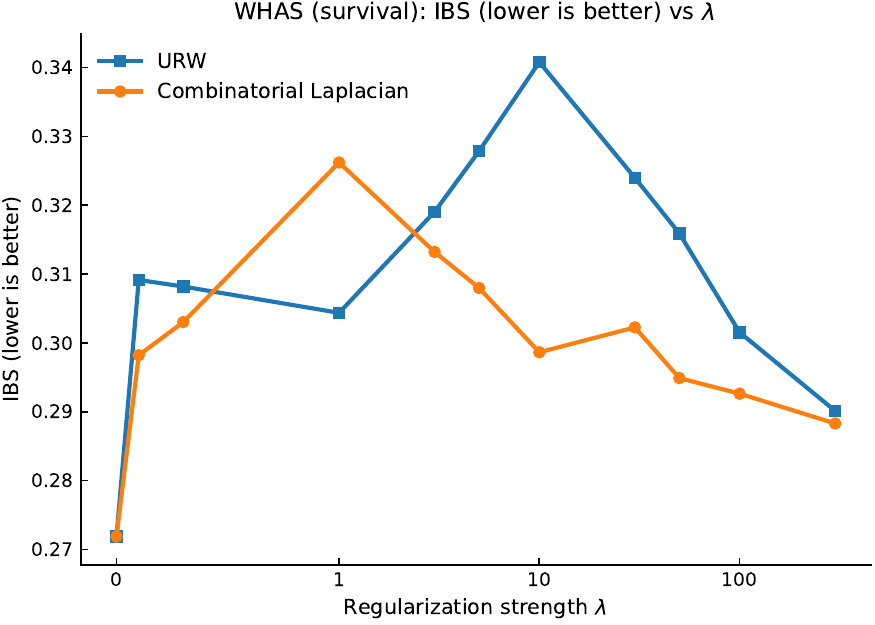} &
\includegraphics[width=\linewidth]{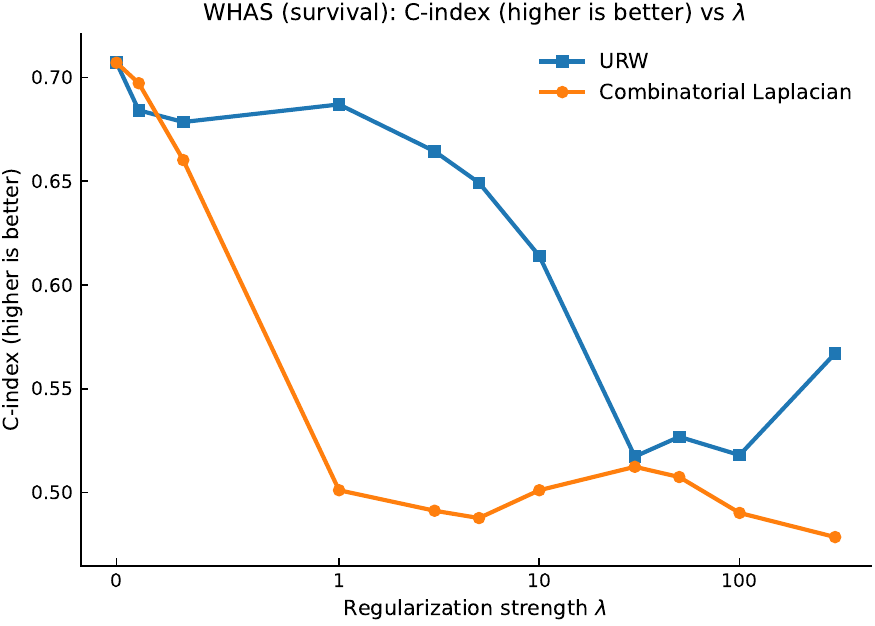} &
\includegraphics[width=\linewidth]{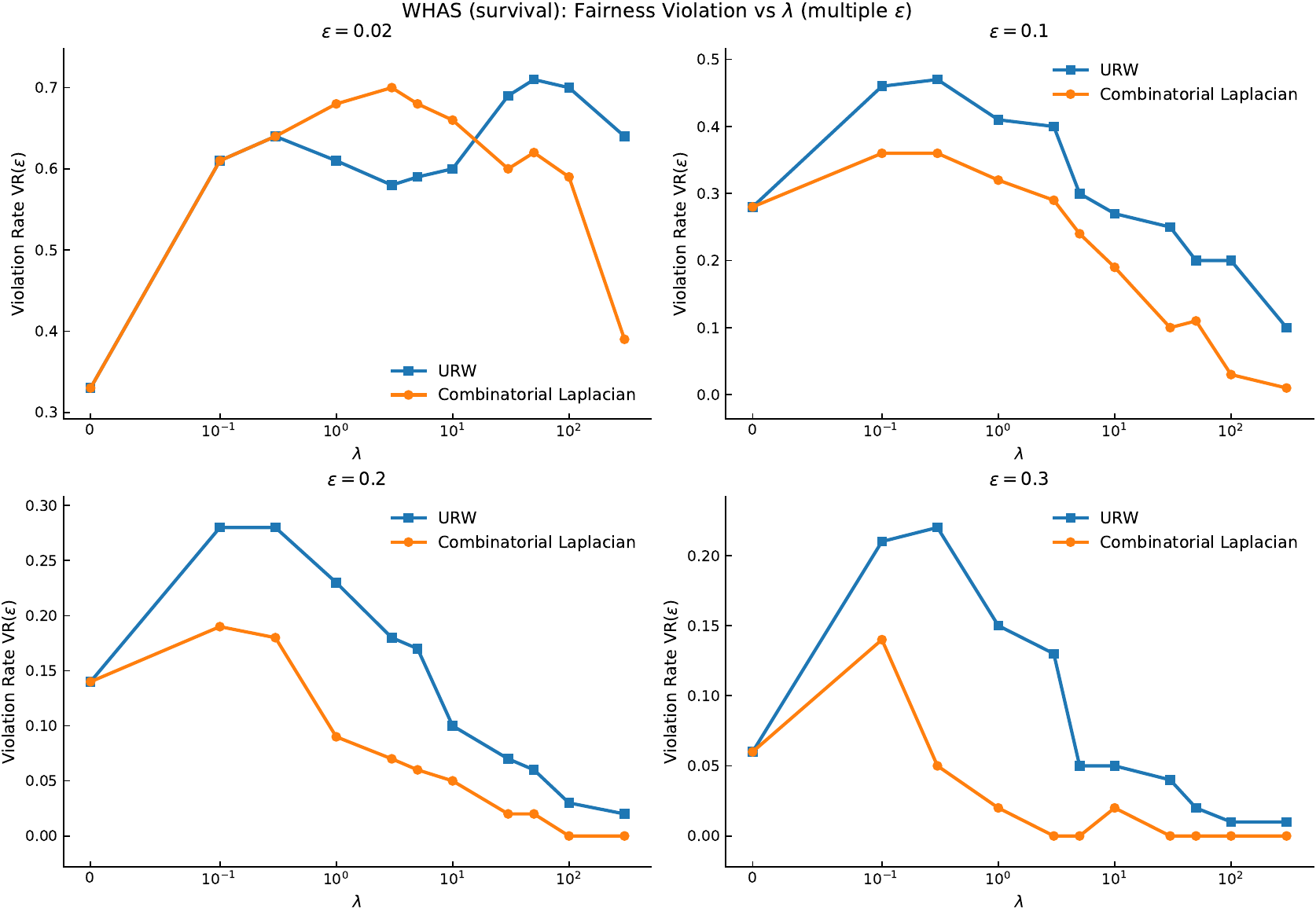} \\
\addlinespace[0.45em]

% ===================== ROSSI =====================
\includegraphics[width=\linewidth]{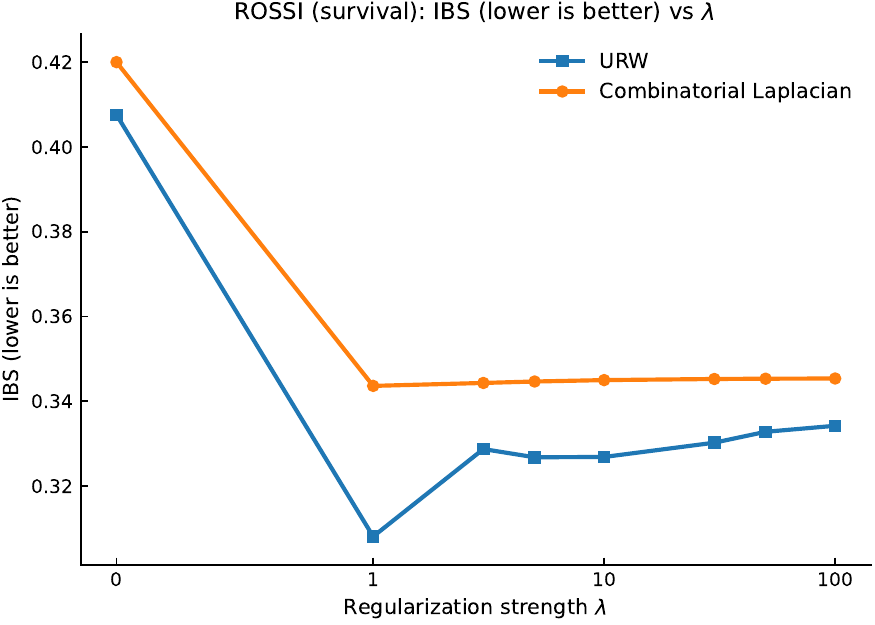} &
\includegraphics[width=\linewidth]{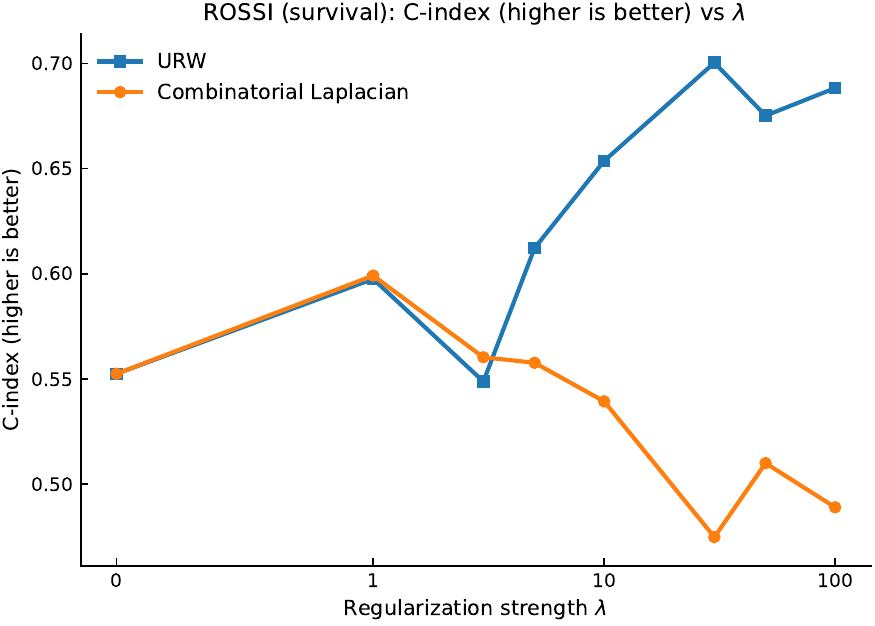} &
\includegraphics[width=\linewidth]{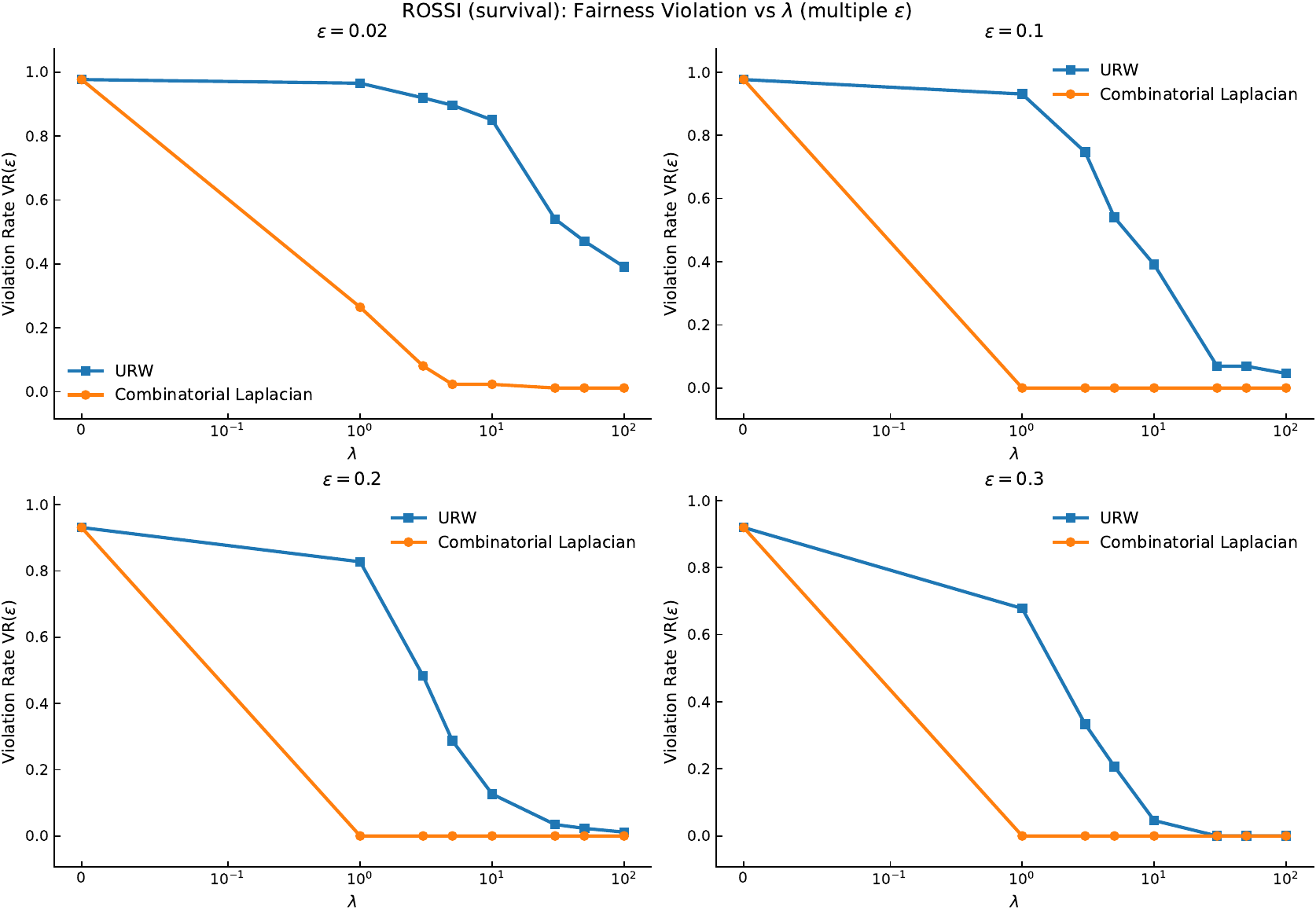} &
\makebox[\linewidth][c]{\tiny\textsc{ROSSI}} &
\includegraphics[width=\linewidth]{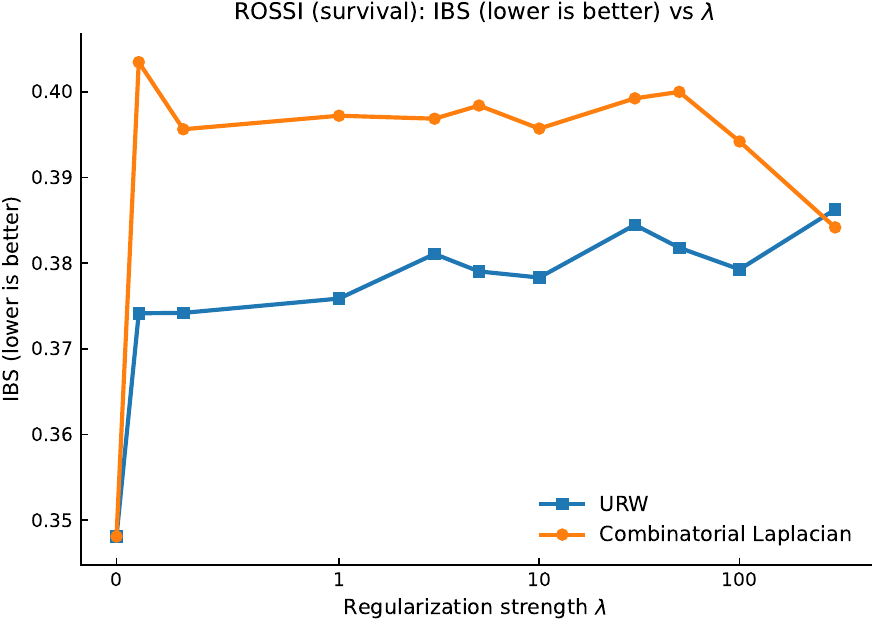} &
\includegraphics[width=\linewidth]{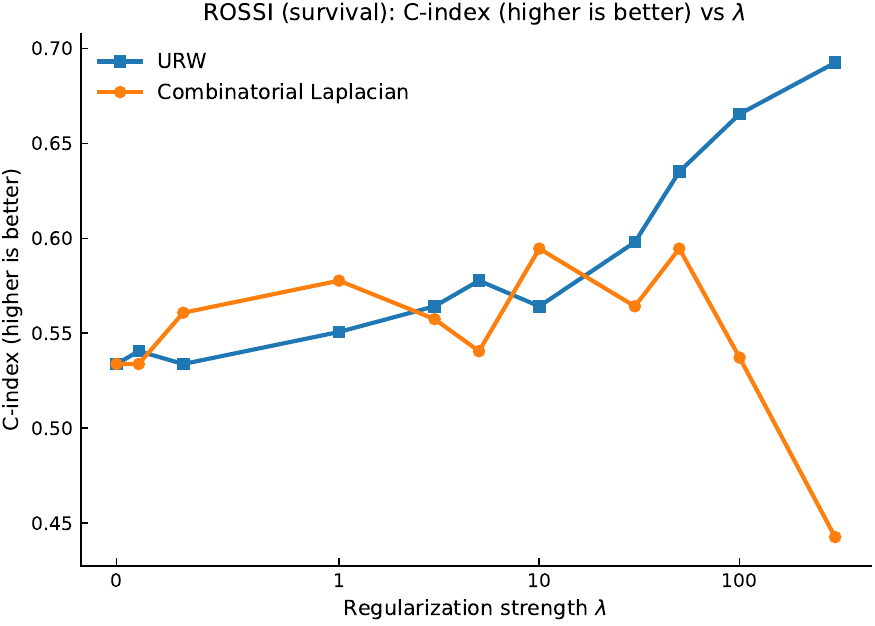} &
\includegraphics[width=\linewidth]{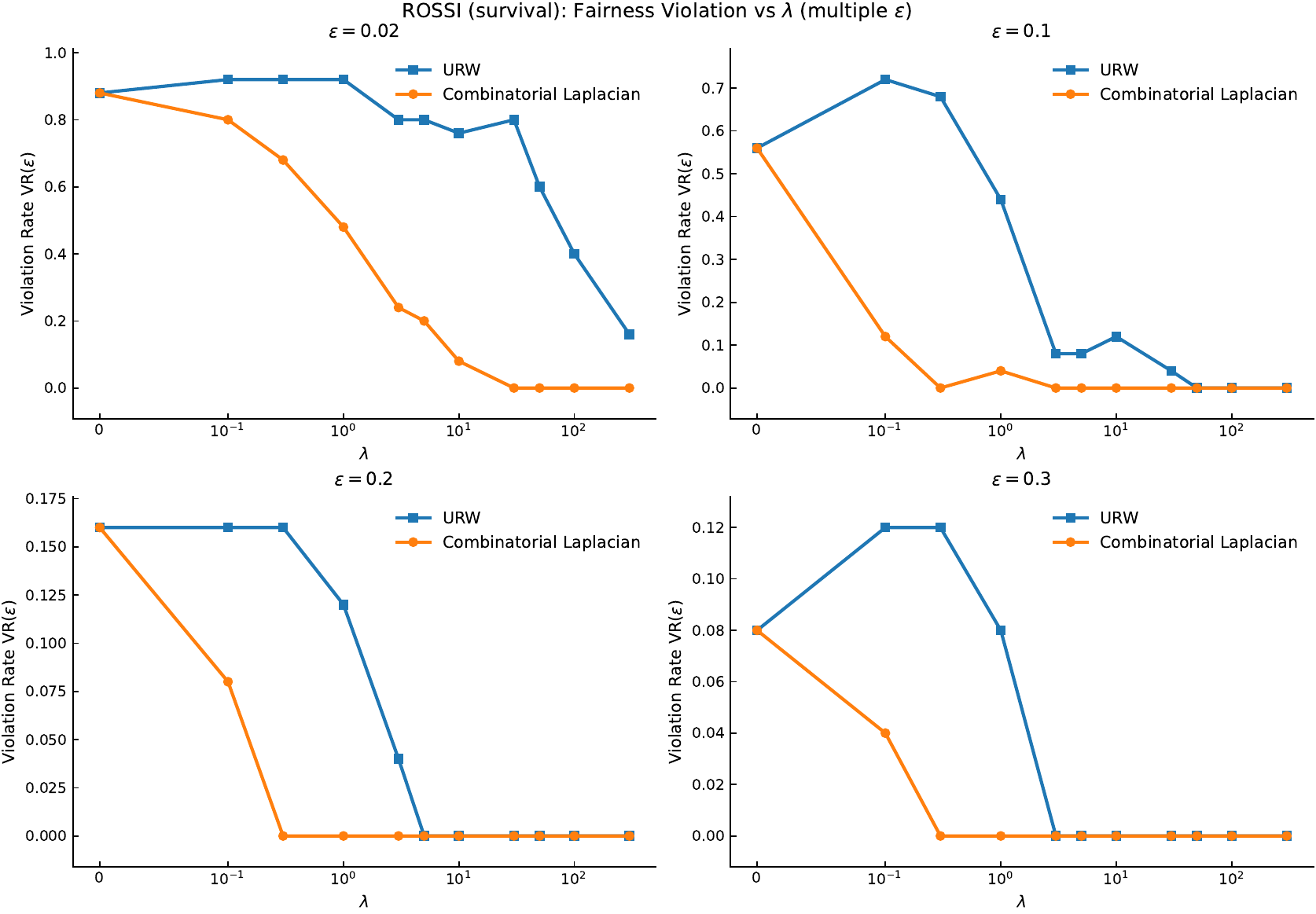} \\

\end{tabular}
}
\vspace{0.25em}
\caption{
Performance--fairness trade-offs for survival prediction.
Rows correspond to datasets; columns report IBS, C-index, and violation rate $\mathrm{VR}(\varepsilon)$
for DeepSurv (log-risk) and DeepHazard with survival-function outputs as $\lambda$ varies.
}
\label{fig:survival_tradeoff_twomodels}
\end{figure}

Consistent with the residual-based results, the independent pairwise diagnostics reported in Appendix~\ref{app:pairwise-diagnostics} also decrease under stronger fairness regularization across classification, regression, and survival tasks.
\begin{table}[h]
\centering
\begin{adjustbox}{max width=\linewidth}
\begin{tblr}{
  cells = {c},
  hline{1-2,5} = {-}{},
}
\diagbox{Metrics}{Methods}      & SenSR  & SenSeI & iFlip ($\rho=0.9$) & iFlip ($\rho=0.5$) & iFlip ($\rho=0.1$) & Comb. Laplacian & URW-FKL & URW-RKL \\
Bal-Acc $\uparrow$    
& \textbf{0.7665} & 0.7100 & 0.7563 & 0.7577 & 0.7478 & 0.7345 & 0.7367 & 0.7352 \\

VR(0.005) $\downarrow$  
& 0.5002 & 0.1941 & 0.5643 & 0.5563 & 0.5474 & 0.1804 & 0.1814 & \textbf{0.1767} \\

VR(0.001) $\downarrow$  
& 0.6740 & 0.4814 & 0.6655 & 0.6637 & 0.6517 & 0.3410 & 0.3264 & \textbf{0.3202} \\
\end{tblr}
\end{adjustbox}
\caption{
Balanced accuracy and violation rates on the \textsc{Adult} dataset. For methods with a parameter $\lambda$ (combinatorial Laplacian and URW), we report a balanced operating point selected from the $\lambda$ grid such that the resulting violation rate at $\varepsilon=0.005$ is closest to that of SenSeI.
}
\label{tab:baselines}
\end{table}

Importantly, the same overall fairness trend is observed under diagnostics that are not defined by the proposed residual. The direct pairwise statistics in Appendix~\ref{app:pairwise-diagnostics} decrease under stronger fairness regularization across classification, regression, and survival tasks, indicating that reductions in the neighborhood residual are accompanied by reductions in direct prediction discrepancies between comparable instances.

The Waller-style diagnostic provides a complementary external comparison for binary classification. For URW-FKL, mean Waller consistency increases from $0.8929$ to $0.9473$ on Adult and from $0.8797$ to $0.9904$ on COMPAS as regularization strengthens, with corresponding improvements in PCS. At the instance level, however, the two diagnostics need not rank individuals similarly: correlations are modest on Adult and close to zero for many COMPAS configurations. This reflects the distinction between hard-label neighborhood agreement in the Waller diagnostic and probability-level local consistency in the proposed formulation. Complete results and representative disagreement cases are reported in Appendix~\ref{app:waller-comparison}.

\section{Conclusion and Future Work}
\label{section_conclusion}

We studied local predictive consistency as a node-level relaxation of pairwise individual fairness. Instead of enforcing fairness only through edge-level constraints between pairs of similar individuals, the proposed formulation assigns each individual a residual that measures whether its prediction is locally anomalous relative to its neighborhood. This yields an interpretable instance-level fairness score and an associated regularized learning formulation that can be incorporated into standard learning objectives.

The local predictive-consistency formulation is formally connected to pairwise individual fairness and graph Laplacian regularization. Local pairwise Lipschitz fairness upper-bounds the instance-level residual, while the row-random-walk Laplacian regularizer serves as a node-level counterpart to edge-based Laplacian smoothing. Empirically, the resulting framework attains competitive or favorable operating points on the fairness--performance frontier across binary classification, regression, and time-to-event prediction.

\paragraph{Limitations.}
The proposed diagnostic depends on the comparison structure used to define local similarity. Consequently, its interpretation is sensitive to the representation, graph construction, graph coverage, and reference population. As discussed in Section~\ref{subsec:proposed-if-score}, isolated instances indicate insufficient
comparison coverage, the aggregated residual may attenuate opposing pairwise
discrepancies through cancellation, and the resulting score is cohort-relative
rather than an intrinsic property of an individual.

A second limitation is that we do not provide a convergence analysis for the non-convex neural-network objectives induced by the proposed regularization. Our theoretical analysis instead focuses on the relationship between local predictive consistency and pairwise individual fairness, while the empirical analysis examines the resulting behavior across tasks, graph constructions, and model configurations. Future work could investigate robust comparison-set construction, fair metric learning, extensions to more complex output spaces, and further theoretical properties of the regularized objectives.

\paragraph{Broader impact.}
The proposed formulation should be interpreted as a diagnostic of local predictive consistency rather than as a general certificate that an individual has been treated fairly. Its conclusions are conditional on the specified comparison structure, and low local inconsistency alone does not establish pairwise, group-level, or normative fairness.

Local predictive consistency therefore complements rather than replaces broader fairness analysis. A model may exhibit strong local consistency while still producing systematic disparities across demographic groups, while favorable group-level criteria such as demographic parity or equalized odds do not guarantee consistent treatment of comparable individuals. In high-stakes applications such as healthcare, insurance, income prediction, and recidivism, the proposed diagnostic should therefore be considered alongside group-level fairness criteria, domain-specific requirements, and other relevant forms of algorithmic evaluation.

\bibliography{main}
\bibliographystyle{tmlr}

\appendix
\newpage

\onecolumn

\clearpage
% \ifsupplement
\setcounter{page}{1}
\begin{center}
    {\LARGE\bfseries
    Individual Fairness via Local Predictive Consistency
    \par}

    \vspace{0.5em}

    {\large\bfseries Supplementary Material\par}
\end{center}

\section{Implementation Details}
\label{app:implementation}
\subsection{Preprocessing and Data Splits}

Across classification and regression tasks, we use a unified preprocessing pipeline. 
Missing values are removed or median-imputed, categorical variables are one-hot encoded, and numerical features are standardized using training-set statistics. 
Sensitive attributes are excluded from predictive inputs but retained for evaluation and fairness construction. 
Datasets are split into 60\% training, 20\% validation, and 20\% test sets using a fixed random seed. 
For survival datasets, due to smaller sample sizes, we use an 80\%/20\% train–test split without a separate validation set.

\subsection{Model Architecture and Optimization}

All methods share the same MLP backbone to ensure comparability. The network uses ReLU activations with hidden width 256 and one hidden layer, unless otherwise specified. For classification, we optimize binary cross-entropy on logits. For regression, we use mean squared error on log-transformed targets. Specifically, for strictly positive outcomes, targets are clipped below at a small positive constant and transformed via the natural logarithm during training; predictive performance is then evaluated on the original scale after exponentiating model outputs. For survival analysis, we optimize either the Cox partial likelihood on predicted risk scores or survival-function objectives for models that directly predict a survival function.

\subsection{Fair Representation and Graph Construction}
To ensure consistent comparisons across methods and tasks, predictive models are trained on fair-projected standardized non-sensitive features. Similarity graphs for fairness regularization and evaluation are constructed from the corresponding metric-transformed fair-projected representations.

We first standardize the predictive input features using the mean and standard deviation computed from the training split. Sensitive attributes are not used as predictive inputs, but they are used to construct the fairness projection and retained for evaluation. To construct the projection, we compute group prototype vectors in the standardized feature space, where each prototype is the mean standardized feature vector of training instances belonging to a sensitive group. These prototypes summarize the main directions along which feature distributions differ across sensitive groups. We then apply singular value decomposition (SVD) to the matrix of group prototypes to estimate a low-rank sensitive subspace, and obtain the fair projection by projecting standardized features onto the orthogonal complement of this subspace.

After fitting the scaler, sensitive-subspace basis, and projection matrix on the training split, we apply the resulting fair projection to all splits for graph construction. We then fit the metric transformation on the fair-projected training representations and apply it unchanged to validation and test representations. Thus, all reported experiments construct similarity graphs in the same metric-transformed fair-projected representation space for both the URW regularizer used in the local predictive-consistency formulation and the combinatorial Laplacian baseline.

During training, the averaging operator $P$ is constructed using the training split. For held-out evaluation, the fitted preprocessing pipeline, fair projection, and metric transformation are applied unchanged to the corresponding split, and the evaluation graph is constructed within that split using the same graph-construction parameters as in training. This ensures that validation and test representations are not used to fit the preprocessing, projection, or metric transformation.

Given metric-transformed fair-projected features $\{\mathbf{z}_i\}$, we define edge weights using a radial basis function (RBF) kernel
\[
W_{ij}=\exp\!\left(-\gamma \lVert \mathbf{z}_i-\mathbf{z}_j\rVert_2^2\right),
\]
and sparsify the graph by retaining only local edges:
\[
W_{ij}=0 \quad \text{whenever} \quad \lVert \mathbf{z}_i-\mathbf{z}_j\rVert_2>\tau .
\]
Self-loops are included so that the row-stochastic averaging operator is well-defined for every instance. The averaging matrix is then obtained by row normalization, $P=D^{-1}W$, where $D_{ii}=\sum_j W_{ij}$.

For the main experiments, we use a fixed distance threshold $\tau=5$ in the metric-transformed fair-projected representation space, with $\gamma=10^{-3}$ for the RBF kernel. We do not enforce a minimum-neighbor requirement. Instead, we report graph-coverage diagnostics, including the isolated-node rate and edge density, to characterize the resulting comparison graph. These diagnostics are used as an empirical check on whether the graph provides non-trivial local comparison sets while preserving locality, rather than as an automatic threshold-selection rule.

A node whose only neighbor is itself has a residual of zero by construction, but this should not be interpreted as evidence of fairness. Rather, it indicates that the graph does not provide a meaningful local comparison set for that instance. We therefore treat isolated nodes as graph-coverage limitations rather than as automatically fair cases. The same graph construction protocol is used for both URW and the combinatorial Laplacian baseline, so this issue is shared by both graph-based methods rather than being specific to the residual-based fairness score.

Additional sensitivity results with respect to $\tau$ are reported in Appendix~\ref{app:threshold-sensitivity}.

\subsection{Baseline Configurations}

All baselines are implemented following their original papers while using the same preprocessing and backbone. 

For URW, we tune the regularization parameter $\lambda$ over the grid
\[
\lambda \in \{0, 0.1, 0.3, 0.5, 1, 3, 5, 10, 30, 50, 100, 300, 500\}.
\]

For iFlipper we evaluate label-adjustment strengths $\rho\in\{0.1,0.5,0.9\}$.  
SenSR and SenSeI are trained using the default hyperparameters suggested in their respective implementations.

For combinatorial Laplacian regularization, we tune the regularization weight on the same $\lambda$ grid as URW to ensure a fair comparison.

\section{Independent Pairwise Fairness Diagnostics}
\label{app:pairwise-diagnostics}

This appendix provides the complete definitions of the direct pairwise diagnostics summarized in \cref{par:independent-pairwise-diagnostics}. These diagnostics are computed directly over the edges of the held-out test comparison graph and do not use the proposed neighborhood-aggregation residual. They therefore provide a complementary edge-level view of local predictive consistency on the same comparison graph.

Let
\[
\mathcal{E}_{\mathrm{test}}
=
\left\{
(i,j): i<j,\; W_{ij}>0
\right\}
\]
denote the set of unique non-self undirected edges in the test graph, and let
$s_i$ denote the scalar prediction associated with instance $i$. For each
comparison edge $(i,j)\in\mathcal{E}_{\mathrm{test}}$, we define
\[
d_{ij}=|s_i-s_j|.
\]
We summarize the resulting pairwise discrepancies using
\begin{align}
\mathrm{EdgeMeanAbs}
&=
\frac{1}{|\mathcal{E}_{\mathrm{test}}|}
\sum_{(i,j)\in\mathcal{E}_{\mathrm{test}}}
d_{ij},
\\
\mathrm{EdgeMedian}
&=
\operatorname{median}_{(i,j)\in\mathcal{E}_{\mathrm{test}}}
d_{ij},
\\
\mathrm{EdgeRMS}
&=
\left(
\frac{1}{|\mathcal{E}_{\mathrm{test}}|}
\sum_{(i,j)\in\mathcal{E}_{\mathrm{test}}}
d_{ij}^{2}
\right)^{1/2},
\\
\mathrm{EdgeMax}
&=
\max_{(i,j)\in\mathcal{E}_{\mathrm{test}}}
d_{ij}.
\end{align}

We also compute node-level pairwise summaries. For each covered node
$i\in\mathcal{C}_{\mathrm{test}}$, with
\[
\mathcal{N}(i)
=
\{j\neq i:W_{ij}>0\},
\]
we define
\begin{align}
D_i^{\mathrm{mean}}
&=
\frac{1}{|\mathcal{N}(i)|}
\sum_{j\in\mathcal{N}(i)}
|s_i-s_j|,
\\
D_i^{\mathrm{RMS}}
&=
\left(
\frac{1}{|\mathcal{N}(i)|}
\sum_{j\in\mathcal{N}(i)}
(s_i-s_j)^2
\right)^{1/2},
\\
D_i^{\max}
&=
\max_{j\in\mathcal{N}(i)}
|s_i-s_j|.
\end{align}
Among these quantities, \textsc{EdgeMeanAbs} is our primary pairwise diagnostic.
It measures the average absolute prediction difference between pairs of instances
that are considered comparable by the test graph. Thus, a smaller \textsc{EdgeMeanAbs} indicates that comparable individuals receive more similar
predictions on average.

Unlike the signed neighborhood residual, these quantities are based on absolute or squared pairwise discrepancies and therefore cannot be reduced by cancellation between neighbors whose predictions lie on opposite sides of the focal prediction. They consequently provide a complementary diagnostic of whether reductions in the proposed local-consistency residual are accompanied by reductions in direct pairwise prediction differences.

Isolated nodes are excluded from the node-level pairwise summaries because they
have no non-self comparison set.

\paragraph{Task-specific prediction scores.}
For classification, $s_i$ is the predicted probability of the positive class.
For regression, we use the predicted log-scale outcome, matching the scale on
which the regression models are trained and regularized. For the DeepSurv
formulation, $s_i$ is the predicted log-risk score $\eta_i$. For the
discrete-time survival-curve formulation, we use the scalar risk
summary $-\log S_i(t^\ast)$ at the evaluation horizon $t^\ast$. The latter
should therefore be interpreted as a scalar pairwise risk diagnostic rather
than as a distance between entire predicted survival curves.

All pairwise diagnostics are computed post hoc using the same test comparison
graph and the same evaluated test instances as the corresponding main
experiment. Self-loops are excluded, and no model retraining is required.

\paragraph{Results.}
Table~\ref{tab:pairwise_fairness_summary} summarizes the change in EdgeMeanAbs and
EdgeRMS between the unregularized model ($\lambda=0$) and the largest
regularization strength evaluated for each formulation. The purpose of this
comparison is to characterize the response of the independent pairwise
diagnostics to increasingly strong fairness regularization; it is not intended
as a model-selection criterion.

\begin{table}[htbp]
\centering
\caption{Independent pairwise fairness diagnostics on the test comparison
graph. Values show EdgeMeanAbs and EdgeRMS for the unregularized model and the
largest evaluated fairness-regularization strength. The percentage reduction
is relative to $\lambda=0$. Lower values indicate greater local pairwise
predictive consistency.}
\label{tab:pairwise_fairness_summary}
\small
\begin{adjustbox}{max width=\linewidth}
\begin{tabular}{lllrrrrr}
\toprule
Task & Dataset & Formulation &
$\lambda$ &
EdgeMeanAbs$_0$ $\rightarrow$ EdgeMeanAbs$_\lambda$ &
Reduction &
EdgeRMS$_0$ $\rightarrow$ EdgeRMS$_\lambda$ &
Reduction \\
\midrule

Classification
& Adult
& URW-FKL
& 500
& $0.2622 \rightarrow 0.1966$
& 25.0\%
& $0.3615 \rightarrow 0.2700$
& 25.3\% \\

Classification
& Adult
& URW-RKL
& 500
& $0.2622 \rightarrow 0.1969$
& 24.9\%
& $0.3615 \rightarrow 0.2702$
& 25.2\% \\

Classification
& COMPAS
& URW-FKL
& 500
& $0.2799 \rightarrow 0.0886$
& 68.3\%
& $0.3492 \rightarrow 0.1164$
& 66.7\% \\

Classification
& COMPAS
& URW-RKL
& 500
& $0.2799 \rightarrow 0.0884$
& 68.4\%
& $0.3492 \rightarrow 0.1164$
& 66.7\% \\

\midrule

Regression
& Employee Salary
& URW
& 500
& $0.0827 \rightarrow 0.0072$
& 91.3\%
& $0.1271 \rightarrow 0.0193$
& 84.8\% \\

Regression
& Insurance
& URW
& 500
& $0.7965 \rightarrow 0.0182$
& 97.7\%
& $1.0733 \rightarrow 0.0307$
& 97.1\% \\

\midrule

Survival
& FLCHAIN
& Log-risk
& 200
& $6.2599 \rightarrow 0.0127$
& 99.8\%
& $8.1530 \rightarrow 0.0272$
& 99.7\% \\

Survival
& WHAS
& Log-risk
& 100
& $17.9916 \rightarrow 0.2743$
& 98.5\%
& $22.4025 \rightarrow 0.3585$
& 98.4\% \\

Survival
& ROSSI
& Log-risk
& 100
& $5.7462 \rightarrow 0.0235$
& 99.6\%
& $8.0277 \rightarrow 0.0578$
& 99.3\% \\

Survival
& FLCHAIN
& Survival curve
& 300
& $19.3441 \rightarrow 4.1164$
& 78.7\%
& $122.0174 \rightarrow 28.7744$
& 76.4\% \\

Survival
& WHAS
& Survival curve
& 300
& $444.3582 \rightarrow 171.9090$
& 61.3\%
& $701.1402 \rightarrow 236.0681$
& 66.3\% \\

Survival
& ROSSI
& Survival curve
& 300
& $5.7984 \rightarrow 0.3985$
& 93.1\%
& $8.1298 \rightarrow 0.5305$
& 93.5\% \\

\bottomrule
\end{tabular}
\end{adjustbox}
\end{table}

The independent pairwise diagnostics exhibit the same overall pattern as the
residual-based fairness measure. In classification, both URW-FKL and URW-RKL
reduce EdgeMeanAbs and EdgeRMS on Adult and COMPAS, with particularly large reductions on COMPAS. The effect is also pronounced in regression: relative to the unregularized model, EdgeMeanAbs decreases by approximately 91\% on Employee Salary and 98\% on Insurance at the strongest evaluated regularization setting.

The survival results show similarly substantial reductions. For the DeepSurv log risk formulation, EdgeMeanAbs decreases by more than 98\% on FLCHAIN, WHAS, and ROSSI. The scalar survival-curve risk diagnostic also decreases substantially across all three datasets, although its numerical scale is formulation dependent and should not be compared directly with the classification, regression, or log-risk values.

These results provide a complementary empirical validation of the proposed regularization objective: improvements in the neighborhood-residual measure are accompanied by reductions in direct prediction differences between instances connected by the comparison graph. In particular, the reductions in EdgeMeanAbs and EdgeRMS cannot be explained solely by signed-neighbor cancellation in the proposed residual.

\section{Sensitivity to Neighborhood Threshold}
\label{app:threshold-sensitivity}

\subsection{Coverage-Conditional Fairness Evaluation}
Because a test instance with no non-self neighbor does not have a meaningful
local comparison set, we distinguish graph coverage from the fairness
statistics reported over the full test cohort. Let
\[
\mathcal{C}_{\mathrm{test}}
=
\left\{
i\in\mathcal{I}_{\mathrm{test}}
:
\sum_{j\neq i}\mathbb{I}(W_{ij}>0)>0
\right\}
\]
denote the set of covered test instances, i.e., instances with at least one
non-self comparison neighbor. We define graph coverage as
\[
\mathrm{Coverage}
=
\frac{|\mathcal{C}_{\mathrm{test}}|}
     {|\mathcal{I}_{\mathrm{test}}|}.
\]

The main experiments retain the original violation ratio computed over the full test set. As a complementary coverage-aware diagnostic, we additionally define
\begin{equation}
\mathrm{VR}_{\mathrm{cov}}(\epsilon)
=
\frac{1}{|\mathcal{C}_{\mathrm{test}}|}
\sum_{i\in\mathcal{C}_{\mathrm{test}}}
\mathbf{1}
\left(
\mathrm{IF}_i(\theta;P)>\epsilon
\right),
\label{eq:vr_covered}
\end{equation}
using the same tolerance $\epsilon$ as the corresponding full-test-set VR.
Thus, isolated nodes remain part of the original full-test evaluation but are
excluded from this conditional summary; their prevalence is instead reflected
by the graph-coverage and isolated-node statistics reported below.

\begin{table}[htbp]
\centering
\caption{Graph coverage at the main operating threshold $\tau=5$.
Coverage denotes the proportion of evaluation instances with at least one non-self comparison neighbor.}
\label{tab:main_graph_coverage}
\small
\begin{tabular}{llrr}
\toprule
Task & Dataset / formulation & Coverage (\%) & Isolated (\%) \\
\midrule
Classification
& Adult                         & 98.51 & 1.49 \\
& COMPAS                        & 99.72 & 0.28 \\
\midrule
Regression
& Employee Salary               & 99.80 & 0.20 \\
& Insurance                     & 100.00 & 0.00 \\
\midrule
Survival
& FLCHAIN -- log-risk           & 90.79 & 9.21 \\
& FLCHAIN -- survival curve     & 90.79 & 9.21 \\
& WHAS -- log-risk              & 75.00 & 25.00 \\
& WHAS -- survival curve        & 73.00 & 27.00 \\
& ROSSI -- log-risk             & 98.85 & 1.15 \\
& ROSSI -- survival curve       & 96.00 & 4.00 \\
\bottomrule
\end{tabular}
\end{table}

Table~\ref{tab:main_graph_coverage} shows that coverage is high for most datasets. The main exception is WHAS, where coverage is $75\%$ for the
log-risk formulation and $73\%$ for the survival-curve formulation, making coverage-conditional reporting particularly relevant for this dataset.

\subsection{Threshold Sensitivity}
We further examine the sensitivity of the results to the threshold $\tau$ used to define local edges. While the main experiments use the fixed threshold $\tau=5$, we construct a grid of alternative thresholds from empirical distance quantiles in the metric-transformed fair-projected representation space to assess sensitivity to the graph scale. Specifically, we consider thresholds corresponding to the 5th, 10th, 15th, 20th, and 25th percentiles of pairwise distances on the training set. This procedure does not enforce a minimum-neighbor constraint. Instead, it evaluates how predictive performance and local inconsistency vary as the comparison graph becomes progressively denser. We report the non-self edge density and isolated-node rate as graph-coverage diagnostics, where edge density characterizes the overall graph scale and the isolated-node rate measures the extent to which instances lack meaningful non-self comparison sets.

Tables~\ref{tab:tau_sensitivity_adult}, \ref{tab:tau_sensitivity_employee}, and~\ref{tab:tau_sensitivity_whas} report sensitivity results for representative binary classification, regression, and survival prediction tasks. As expected from the quantile-based construction, increasing $\tau$ produces denser graphs across all tasks. On \textsc{Adult}, the non-self edge density increases from 0.05 to 0.25, while the isolated-node rate decreases from 0.0294 to 0.0044. On \textsc{Employee Salary}, the edge density similarly increases from 0.05 to 0.25, and the isolated-node rate remains zero across all tested thresholds. On \textsc{WHAS}, the edge density increases from 0.05 to 0.25, while the isolated-node rate decreases from 0.1000 to 0.0300. These results show that larger thresholds improve graph coverage, but at the cost of using broader local comparison sets.

The predictive metrics remain reasonably stable across moderate threshold choices, although the effect of graph scale differs by task. On \textsc{Adult}, balanced accuracy and AUC vary only modestly across thresholds for both Forward-KL and Reverse-KL aggregation. On \textsc{Employee Salary}, predictive performance remains strong across the threshold range, although larger neighborhoods are associated with some degradation in RMSE and $R^2$. On \textsc{WHAS}, the sparsest tested graph yields the highest C-index and lowest IBS, while denser graphs are associated with weaker discrimination and higher prediction error. This pattern is consistent with the small-sample survival setting, where overly broad neighborhoods may smooth risk scores across less comparable individuals.

Across the three tasks, the local inconsistency score varies with $\tau$ because changing the threshold changes the local comparison structure itself. In general, wider neighborhoods induce denser graphs and compare each instance against a broader, potentially more heterogeneous set of neighbors. Therefore, changes in the IF score should not be interpreted as numerical instability; rather, they reflect the sensitivity of local predictive consistency to the chosen scale of comparison. These results support the use of graph-coverage diagnostics when selecting $\tau$: very small thresholds may provide insufficient local coverage, whereas overly large thresholds may weaken the locality of the fairness comparison and, in some settings, degrade predictive performance.

\begin{table}[t]
\centering
\caption{
Sensitivity to the neighborhood threshold on the \textsc{Adult} binary classification task. Thresholds are selected from empirical distance quantiles in the metric-transformed fair-projected training representation. Graph density and
isolated rate are computed on the corresponding training comparison graph, whereas balanced accuracy, AUC, and IF mean are evaluated on the held-out
test set.
}
\label{tab:tau_sensitivity_adult}
\begin{adjustbox}{max width=\linewidth}
\begin{tabular}{lccccccc}
\toprule
Aggregation & Threshold & $\tau$ & Edge density & Isolated rate & Bal. Acc. & AUC & IF mean \\
\midrule
Forward-KL & $q_5$    & 4.7482 & 0.05 & 0.0294 & 0.7567 & 0.8793 & 0.1521 \\
Forward-KL & $q_{10}$ & 5.5980 & 0.10 & 0.0168 & 0.7565 & 0.8782 & 0.1907 \\
Forward-KL & $q_{15}$ & 6.0279 & 0.15 & 0.0114 & 0.7507 & 0.8782 & 0.2057 \\
Forward-KL & $q_{20}$ & 6.3947 & 0.20 & 0.0078 & 0.7496 & 0.8771 & 0.2136 \\
Forward-KL & $q_{25}$ & 6.7516 & 0.25 & 0.0044 & 0.7485 & 0.8770 & 0.2230 \\
\midrule
Reverse-KL & $q_5$    & 4.7482 & 0.05 & 0.0294 & 0.7538 & 0.8780 & 0.2545 \\
Reverse-KL & $q_{10}$ & 5.5980 & 0.10 & 0.0168 & 0.7522 & 0.8743 & 0.3404 \\
Reverse-KL & $q_{15}$ & 6.0279 & 0.15 & 0.0114 & 0.7457 & 0.8751 & 0.3981 \\
Reverse-KL & $q_{20}$ & 6.3947 & 0.20 & 0.0078 & 0.7519 & 0.8744 & 0.4395 \\
Reverse-KL & $q_{25}$ & 6.7516 & 0.25 & 0.0044 & 0.7522 & 0.8740 & 0.4619 \\
\bottomrule
\end{tabular}
\end{adjustbox}
\end{table}

\begin{table}[t]
\centering
\caption{
Sensitivity to the neighborhood threshold on the \textsc{Employee Salary}
regression task. Thresholds are selected from empirical distance quantiles in
the metric-transformed fair-projected training representation. Graph density
and isolated rate are computed on the corresponding training comparison graph
using non-self edges, whereas RMSE, $R^2$, and IF mean are evaluated on the
held-out test set.
}
\label{tab:tau_sensitivity_employee}
\begin{adjustbox}{max width=\linewidth}
\begin{tabular}{ccccccc}
\toprule
Threshold & $\tau$ & Edge density & Isolated rate & RMSE & $R^2$ & IF mean \\
\midrule
$q_5$    & 3.4490 & 0.05 & 0.0000 & 5370.50 & 0.9863 & 0.0758 \\
$q_{10}$ & 3.9358 & 0.10 & 0.0000 & 5485.42 & 0.9857 & 0.0966 \\
$q_{15}$ & 4.3195 & 0.15 & 0.0000 & 5842.11 & 0.9837 & 0.1264 \\
$q_{20}$ & 4.5652 & 0.20 & 0.0000 & 6125.28 & 0.9821 & 0.1465 \\
$q_{25}$ & 4.7947 & 0.25 & 0.0000 & 6673.61 & 0.9788 & 0.1756 \\
\bottomrule
\end{tabular}
\end{adjustbox}
\end{table}

\begin{table}[t]
\centering
\caption{
Sensitivity to the neighborhood threshold on the \textsc{WHAS} survival
prediction task. Thresholds are selected from empirical distance quantiles in
the metric-transformed fair-projected training representation. Graph density
and isolated rate are computed on the corresponding training comparison graph
using non-self edges, whereas C-index, IBS, and IF mean are evaluated on the
held-out test set.
}
\label{tab:tau_sensitivity_whas}
\begin{adjustbox}{max width=\linewidth}
\begin{tabular}{ccccccc}
\toprule
Threshold & $\tau$ & Edge density & Isolated rate & C-index & IBS & IF mean \\
\midrule
$q_5$    & 2.4460 & 0.05 & 0.1000 & 0.7236 & 0.1741 & 0.6136 \\
$q_{10}$ & 2.7953 & 0.10 & 0.0750 & 0.6926 & 0.1752 & 0.6413 \\
$q_{15}$ & 3.0374 & 0.15 & 0.0675 & 0.6890 & 0.1953 & 0.7414 \\
$q_{20}$ & 3.2330 & 0.20 & 0.0375 & 0.6704 & 0.2049 & 0.6745 \\
$q_{25}$ & 3.4078 & 0.25 & 0.0300 & 0.6817 & 0.2007 & 0.7243 \\
\bottomrule
\end{tabular}
\end{adjustbox}
\end{table}

\section{A Controlled Experiment to Investigate Cancellation Behavior}
\label{app:cancellation-diagnostic}
We consider a small controlled diagnostic in which the focal prediction and its neighboring
predictions are specified directly. For simplicity, all neighbors receive
equal weight. We compare the proposed aggregated residual
\[
\left|
f_i-\sum_j P_{ij}f_j
\right|
\]
with three natural node-level aggregations of pairwise discrepancies:
\[
\sum_j P_{ij}|f_i-f_j|,
\qquad
\sqrt{\sum_j P_{ij}(f_i-f_j)^2},
\qquad
\max_j |f_i-f_j|.
\]

\begin{table}[htbp]
\centering
\caption{Controlled examples illustrating the distinction between the proposed
aggregated residual and node-level aggregations of pairwise discrepancies.
Uniform neighbor weights are used in all cases.}
\label{tab:cancellation-diagnostic}
\small
\begin{adjustbox}{max width=\textwidth}
\begin{tabular}{lcccccc}
\toprule
Case & Focal & Neighbor predictions & Residual & MeanAbs & RMS & Max \\
\midrule
Perfectly consistent
& 0.5 & $(0.5,0.5)$
& 0.000 & 0.000 & 0.000 & 0.000 \\

One-sided shift
& 0.5 & $(0.7,0.7)$
& 0.200 & 0.200 & 0.200 & 0.200 \\

Symmetric cancellation
& 0.5 & $(0.0,1.0)$
& 0.000 & 0.500 & 0.500 & 0.500 \\

Partial cancellation
& 0.5 & $(0.2,0.9)$
& 0.050 & 0.350 & 0.354 & 0.400 \\

Single extreme neighbor
& 0.5 & $(0.5,0.5,1.0)$
& 0.167 & 0.167 & 0.289 & 0.500 \\
\bottomrule
\end{tabular}
\end{adjustbox}
\end{table}

The examples illustrate that the proposed residual and the pairwise
aggregations measure related but distinct notions of local inconsistency.
When neighboring predictions shift in the same direction relative to the
focal prediction, as in the one-sided-shift example, the residual and the
pairwise measures agree. By contrast, when neighbor deviations occur in
opposite directions, they can partially or completely cancel in the weighted
local reference. The symmetric-cancellation example therefore yields a zero
residual despite substantial pairwise discrepancies, while the
partial-cancellation example produces a residual much smaller than the
corresponding MeanAbs, Root mean square (RMS), and maximum discrepancies.

This behavior reflects the intended role of local predictive consistency as a
node-level relaxation of strict pairwise fairness. The proposed score asks
whether the focal prediction is atypical relative to the aggregate predictive
pattern among comparable individuals. It is therefore less sensitive to
isolated edge-wise deviations or opposing deviations that balance at the
neighborhood level. In contrast, MeanAbs, RMS, and maximum pairwise
discrepancies are more appropriate when the objective is to detect or penalize
individual edge-wise differences directly.

The relaxation is useful when fairness assessment is intended to identify
instances whose predictions depart systematically from their local comparison
group without requiring every comparison edge to satisfy a strict constraint.
It should not, however, be interpreted as guaranteeing uniformly small
pairwise discrepancies. For this reason, we supplement the proposed residual
with the independent pairwise diagnostics reported in
Appendix~\ref{app:pairwise-diagnostics}.

\section{Pareto Analysis Using Achieved Pairwise Fairness}
\label{app:pareto-achieved}

To characterize the fairness--utility trade-off without relying on the raw
regularization coefficient, we additionally evaluate all tested operating
points in terms of their achieved independent pairwise fairness and predictive
utility.

Pairwise fairness is measured using EdgeMeanAbs, where lower values indicate
smaller average prediction discrepancies across neighboring test instances.
Predictive utility is measured using AUC for classification, RMSE for
regression, C-index for log-risk survival models, and integrated Brier score
(IBS) for survival-curve models. Pareto efficiency is assessed directly over
the tested operating points, without imposing an additional utility-loss
threshold. An operating point is considered Pareto-efficient if no other
tested configuration achieves both lower EdgeMeanAbs and better predictive
utility, with at least one strict improvement.
\begin{table}[htbp]
\centering
\caption{Pareto-efficient operating points based on achieved independent
pairwise fairness (EdgeMeanAbs) and predictive utility. For survival tasks,
both C-index and IBS are considered to assess sensitivity to the utility
definition.}
\label{tab:pareto-achieved}
\small
\begin{adjustbox}{max width=\textwidth}
\begin{tabular}{lllccc}
\toprule
Task & Dataset & Formulation & Utility & Pareto $\lambda$ & Pareto/Tested \\
\midrule

Classification & Adult & URW-FKL & AUC
& $0.5,1,3,5,10,500$ & $6/13$ \\

Classification & Adult & URW-RKL & AUC
& $0.5,5,500$ & $3/13$ \\

Classification & COMPAS & URW-FKL & AUC
& $1,3,5,10,30,50,100,300,500$ & $9/13$ \\

Classification & COMPAS & URW-RKL & AUC
& $3,5,10,30,50,100,300,500$ & $8/13$ \\

Regression & Employee Salary & URW & RMSE
& $0.3,0.5,1,3,5,10,30,50,100,300,500$ & $11/13$ \\

Regression & Insurance & URW & RMSE
& $0,0.1,0.3,0.5,1,3,5,10,30,50,100,300,500$ & $13/13$ \\

\midrule

Survival & FLCHAIN & Log-risk & C-index
& $3,50,200$ & $3/9$ \\

Survival & FLCHAIN & Log-risk & IBS
& $200$ & $1/9$ \\

Survival & FLCHAIN & Survival curve & C-index
& $5,10,30,50,100,300$ & $6/11$ \\

Survival & FLCHAIN & Survival curve & IBS
& $300$ & $1/11$ \\

Survival & ROSSI & Log-risk & C-index
& $30,100$ & $2/8$ \\

Survival & ROSSI & Log-risk & IBS
& $1,5,10,30,50,100$ & $6/8$ \\

Survival & ROSSI & Survival curve & C-index
& $300$ & $1/11$ \\

Survival & ROSSI & Survival curve & IBS
& $0,0.3,1,10,100,300$ & $6/11$ \\

Survival & WHAS & Log-risk & C-index
& $1,3,10,50,100$ & $5/8$ \\

Survival & WHAS & Log-risk & IBS
& $1,3,10,50,100$ & $5/8$ \\

Survival & WHAS & Survival curve & C-index
& $0,1,3,5,10,300$ & $6/11$ \\

Survival & WHAS & Survival curve & IBS
& $0,300$ & $2/11$ \\

\bottomrule
\end{tabular}
\end{adjustbox}
\end{table}

The resulting frontiers show that the fairness--utility geometry varies
substantially across datasets and prediction settings. For example, all tested
Insurance configurations are Pareto-efficient in the EdgeMeanAbs--RMSE plane,
indicating an approximately monotonic trade-off in which reductions in pairwise
prediction discrepancy are accompanied by increasing predictive error. By
contrast, the FLCHAIN survival-curve formulation yields a much sparser frontier:
six configurations are Pareto-efficient when utility is measured by C-index,
whereas only $\lambda=300$ remains Pareto-efficient when utility is measured by
IBS. This indicates that the set of Pareto-efficient operating points can depend on the predictive-utility criterion being considered.

A similar metric dependence is observed in the other survival experiments. For
FLCHAIN under the log-risk formulation, three configurations are Pareto-efficient
with respect to C-index, while only $\lambda=200$ remains Pareto-efficient with
respect to IBS. For ROSSI, the C-index- and IBS-based frontiers also differ for
both survival formulations, whereas for WHAS under the log-risk formulation the
two frontiers coincide. We therefore report both C-index and IBS for survival
prediction rather than relying on a single utility measure. Overall, these
results show that the effect of fairness regularization is more appropriately
characterized through the achieved fairness--utility frontier, and, for survival
prediction, across complementary utility criteria, rather than through the
numerical value of $\lambda$ alone.

\section{Additional Evaluations on Binary Classification Tasks}
\label{app:consistency}
To assess how fairness regularization affects counterfactual prediction stability, we report GR- and S-consistency as functions of $\lambda$ (higher is better) (Fig.~\ref{fig:consistency_adult} and~\ref{fig:consistency_compas}). 

On the \textsc{Adult} dataset, increasing $\lambda$ generally improves both metrics, with minor non-monotonicity at medium $\lambda$ reflecting the trade-off between task fitting and neighborhood agreement. The combinatorial Laplacian baseline attains larger consistency gains at some operating points, whereas URW reaches comparable consistency levels at operating points with better predictive performance. Between URW variants, Forward-KL typically attains slightly higher GR/S-consistency.

On the \textsc{COMPAS} dataset, the advantage of URW becomes more pronounced. Both URW variants achieve consistently higher GR-consistency across a wide range of $\lambda$, while the Combinatorial Laplacian baseline exhibits weaker improvements and larger fluctuations. This suggests that the node-level residual regularization better captures the structure of the similarity graph in this setting, leading to more stable counterfactual predictions.

Overall, these results complement the VR curves in the main text (Fig.~\ref{fig:tradeoffs_binary}), indicating that URW improves local stability under regularization while maintaining strong predictive performance.

\begin{figure}[!t]
\centering

\begin{subfigure}[t]{0.32\linewidth}
    \centering
    \includegraphics[width=\linewidth]{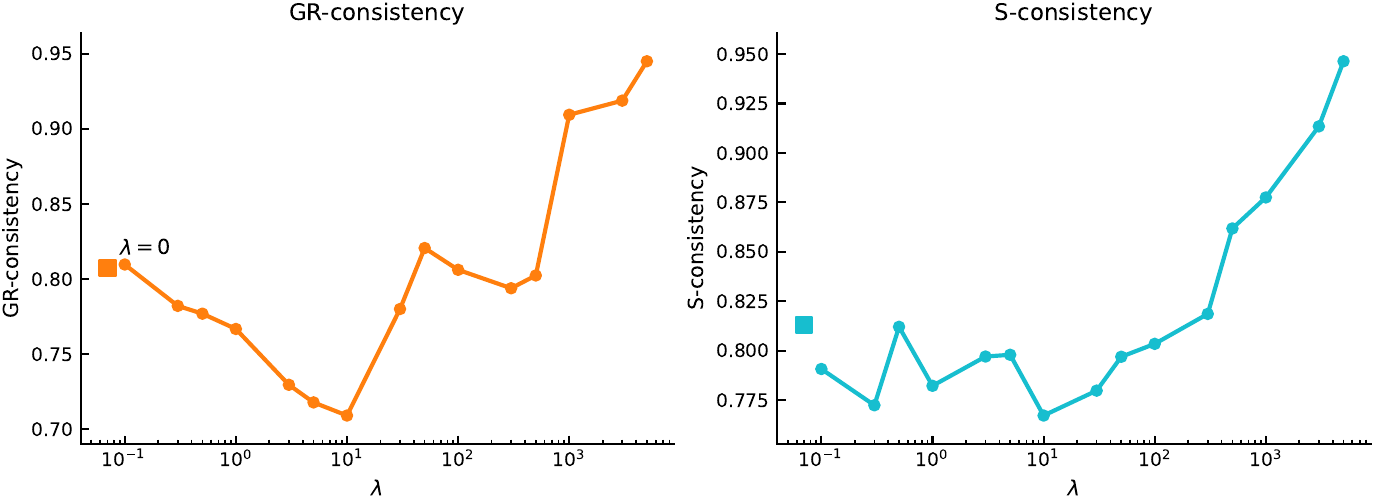}
    \caption{URW (Forward KL)}
\end{subfigure}\hfill
\begin{subfigure}[t]{0.32\linewidth}
    \centering
    \includegraphics[width=\linewidth]{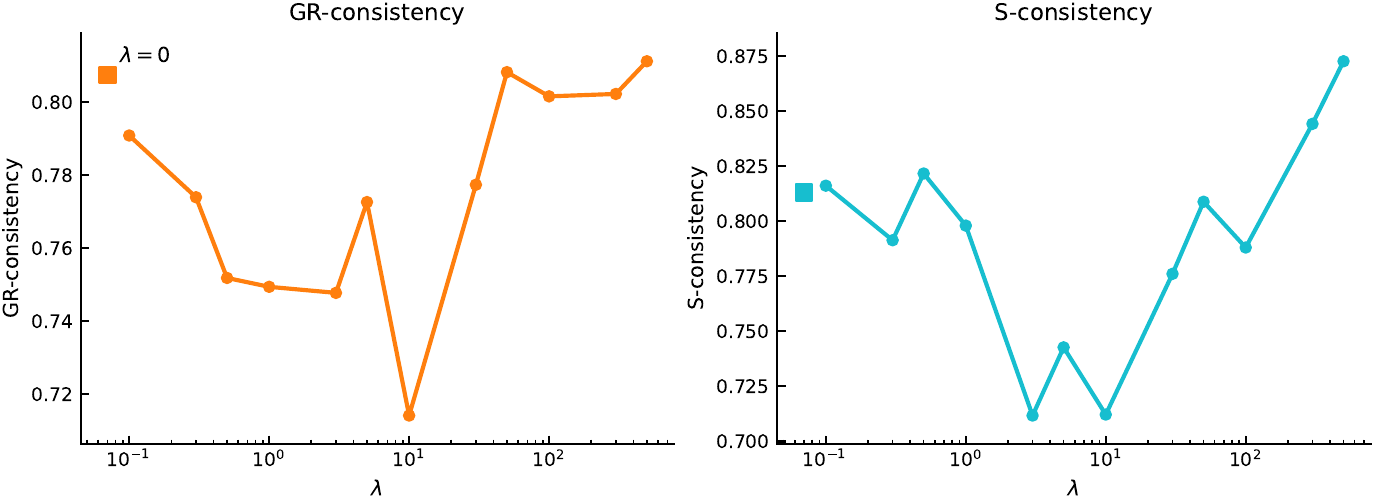}
    \caption{URW (Reverse KL)}
\end{subfigure}\hfill
\begin{subfigure}[t]{0.32\linewidth}
    \centering
    \includegraphics[width=\linewidth]{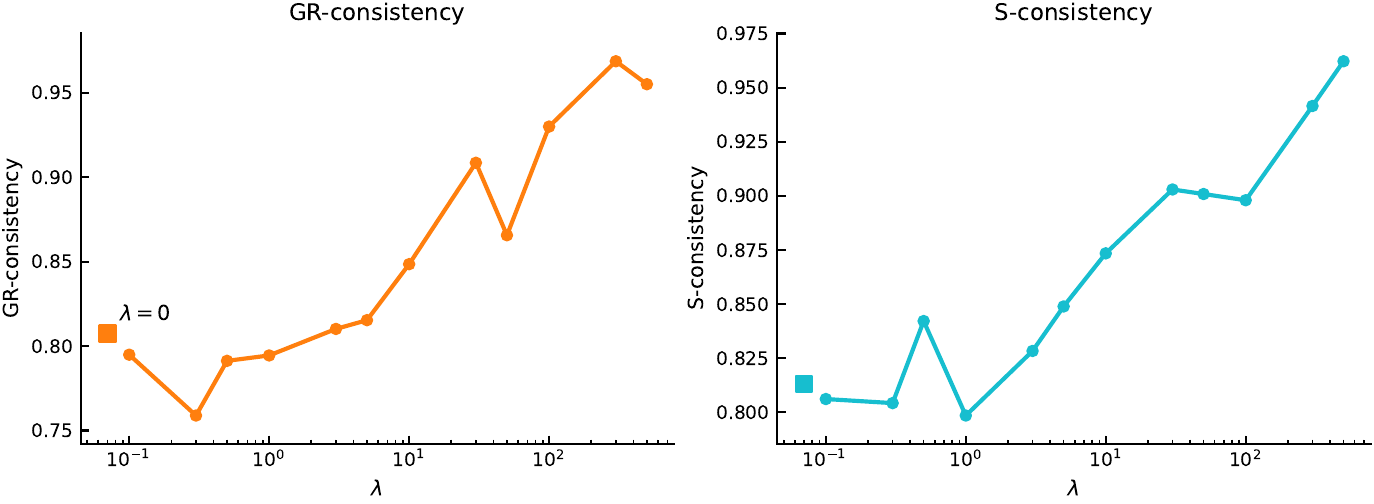}
    \caption{Comb. Laplacian}
\end{subfigure}

\caption{
 Race and Gender (GR) and Spouse (S) consistency as functions of the regularization strength $\lambda$ on the \textsc{Adult} dataset (higher is better). Markers at $\lambda=0$ denote the unregularized model.
}
\label{fig:consistency_adult}
\end{figure}

\begin{figure}[h]
\centering

\begin{subfigure}[t]{0.32\linewidth}
    \centering
    \includegraphics[width=\linewidth]{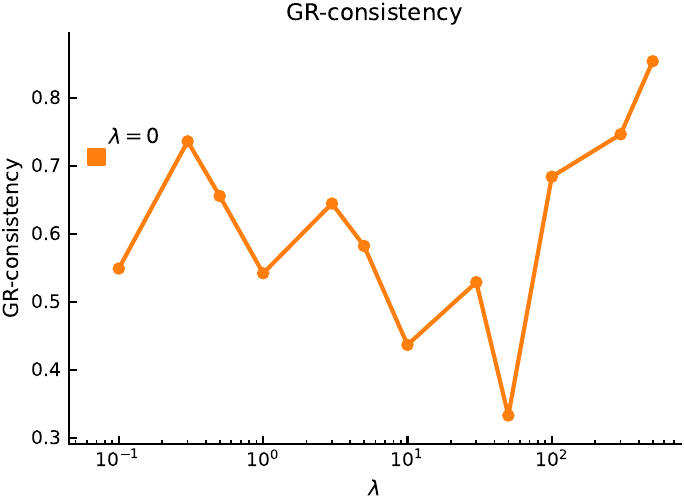}
    \caption{URW (Forward KL)}
\end{subfigure}\hfill
\begin{subfigure}[t]{0.32\linewidth}
    \centering
    \includegraphics[width=\linewidth]{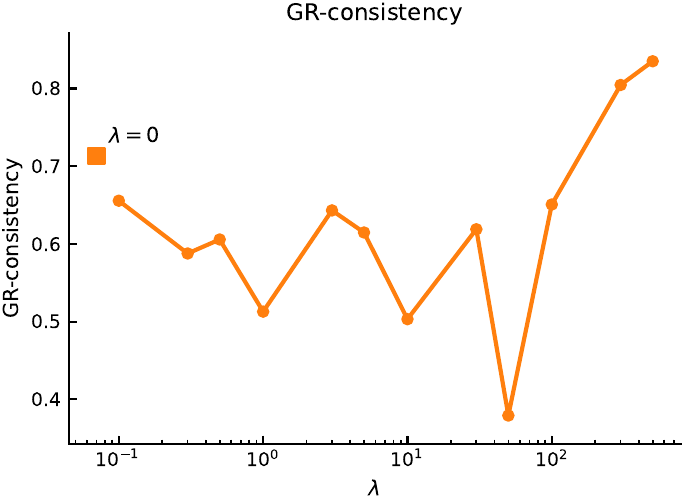}
    \caption{URW (Reverse KL)}
\end{subfigure}\hfill
\begin{subfigure}[t]{0.32\linewidth}
    \centering
    \includegraphics[width=\linewidth]{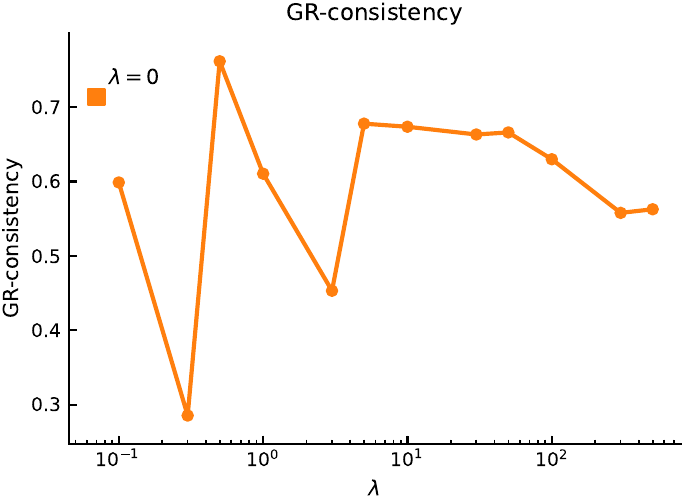}
    \caption{Comb. Laplacian}
\end{subfigure}

\caption{
GR-consistency as a function of the regularization strength $\lambda$ on the \textsc{COMPAS} dataset (higher is better). Markers at $\lambda=0$ denote the
unregularized model.
}
\label{fig:consistency_compas}
\end{figure}

\section{Tail-Sensitive $p$-Norm Regularization}
\label{app:pnorm}

This section further examines how the aggregation of instance-level local predictive-consistency residuals affects sensitivity to large local inconsistencies. In particular, we clarify the relationship between the squared $p$-norm family introduced in the main text and the empirical $p$-norm used in the supplementary tail-sensitivity experiment.

Recall that the main text defines
\[
\Omega_{\mathrm{IF}}^{(p)}(\theta;P)
=
\left(
\frac{1}{n}
\sum_{i=1}^{n}
|r_i(\theta;P)|^p
\right)^{2/p},
\]
where
\[
r_i(\theta;P)
=
f_\theta(x_i)-(Pf_\theta)_i.
\]
This family can equivalently be written as
\[
\Omega_{\mathrm{IF}}^{(p)}(\theta;P)
=
R_p(\theta;P)^2,
\]
where
\begin{equation}
\label{eq:empirical-pnorm}
R_p(\theta;P)
=
\left(
\frac{1}{n}
\sum_{i=1}^{n}
|r_i(\theta;P)|^p
\right)^{1/p}
\end{equation}
is the empirical $p$-norm of the local residual magnitudes. The squared convention in the main text is useful because $p=2$ recovers exactly the original squared-residual regularizer.

For the supplementary experiment reported here, we optimize the underlying
empirical $p$-norm $R_p$ in \eqref{eq:empirical-pnorm}, rather than its squared counterpart. The purpose of this experiment is to isolate the effect of changing the residual aggregation from average-oriented to increasingly tail-sensitive forms. The use of $R_p$ instead of $R_p^2$ changes the numerical scale of the regularization term and therefore should not be interpreted as an identical optimization problem under the same value of $\lambda$. Instead, this experiment provides a controlled comparison of how different aggregation geometries emphasize the residual distribution.

For a fixed residual vector, $R_p$ becomes increasingly sensitive to larger residuals as $p$ increases. In particular,
\[
R_1(\theta;P)
=
\frac{1}{n}
\sum_{i=1}^{n}
|r_i(\theta;P)|
\]
measures average local inconsistency, whereas
\[
\lim_{p\rightarrow\infty}
R_p(\theta;P)
=
\max_i |r_i(\theta;P)|
\]
corresponds to worst-case local inconsistency. Thus, increasing $p$ shifts the emphasis of the regularizer from the center of the residual distribution toward its upper tail.

Directly optimizing the maximum residual is inconvenient because the maximum
operator is non-smooth. For the experiment reported as $p=\infty$, we therefore use the normalized log-sum-exp approximation
\begin{equation}
\label{eq:linf-lse}
\widetilde{\Omega}_{\mathrm{IF}}^{(\infty,\alpha)}
(\theta;P)
=
\frac{1}{\alpha}
\left[
\log
\left(
\sum_{i=1}^{n}
\exp\{
\alpha |r_i(\theta;P)|
\}
\right)
-
\log n
\right].
\end{equation}
Here,
$\widetilde{\Omega}_{\mathrm{IF}}^{(\infty,\alpha)}$
denotes the unsquared smooth surrogate used specifically in this supplementary experiment. As $\alpha\rightarrow\infty$,
\[
\widetilde{\Omega}_{\mathrm{IF}}^{(\infty,\alpha)}
(\theta;P)
\longrightarrow
\max_i |r_i(\theta;P)|.
\]
It therefore approximates the $p\rightarrow\infty$ limit of the empirical $p$-norm $R_p$ used in the supplementary experiment. The corresponding limit of the squared family defined in the main text is $(\max_i |r_i|)^2$.

We use $\alpha=200$ in all reported $p=\infty$ experiments. The subtraction of $\log n$ places the log-sum-exp expression on the same empirical-average convention as the finite-$p$ objectives and does not affect its optimization gradient.

Taken together, the finite-$p$ objectives and the approximate $p=\infty$ objective form a tail-sensitivity spectrum. The $p=1$ objective emphasize average local inconsistency, larger finite values of $p$ place progressively greater emphasis on large residuals, and the approximate $p=\infty$ objective focuses most strongly on the largest local inconsistencies.

To isolate the effect of the aggregation itself, all compared values of $p$ use the same regularization coefficient $\lambda$, without any additional $p$-dependent multiplicative gain. For numerical stability, residual magnitudes are clipped at a common upper bound of $10$ before aggregation for all values of $p$. Gradient clipping is also applied during optimization. These numerical stabilization steps are held fixed across the compared objectives.

We evaluate this tail-sensitive aggregation on the \textsc{Employee Salary} regression task using
\[
p\in\{1,4,\infty\}
\]
while varying the fairness regularization strength $\lambda$. The residuals are computed on the log-prediction scale used by the regression model.

Figure~\ref{fig:pnorm_results} reports predictive error together with the violation rate (VR) at tolerance $\epsilon=0.05$. At small regularization strengths, the three aggregation choices exhibit similar behavior. As $\lambda$ increases, however, the more tail-sensitive $p=4$ and approximate $p=\infty$ objectives generally reduce the violation rate more aggressively than $p=1$, while incurring a larger predictive-error cost. The ordering is not strictly monotone in $p$ at every value of $\lambda$, which is not unexpected given the non-convex neural-network optimization problem. The relevant interpretation is therefore that increasing $p$ changes the objective toward greater sensitivity to the largest residuals, rather than guaranteeing a pointwise ordering of the learned models for every value of $\lambda$.

For example, at $\lambda=1$, the violation rate at $\epsilon=0.05$ decreases from $0.0755$ for $p=1$ to $0.0220$ for $p=4$ and $0.0125$ for the approximate $p=\infty$ objective. The corresponding MAE values increase from $8236.8$ to $9318.9$ and $10944.8$, respectively. These results illustrate the expected fairness--utility trade-off: increasingly tail-sensitive aggregation can suppress large local inconsistencies more strongly, but may do so at a greater predictive cost.

\begin{figure}[H]
\centering
\begin{subfigure}[t]{0.45\linewidth}
    \centering
    \includegraphics[width=\linewidth]{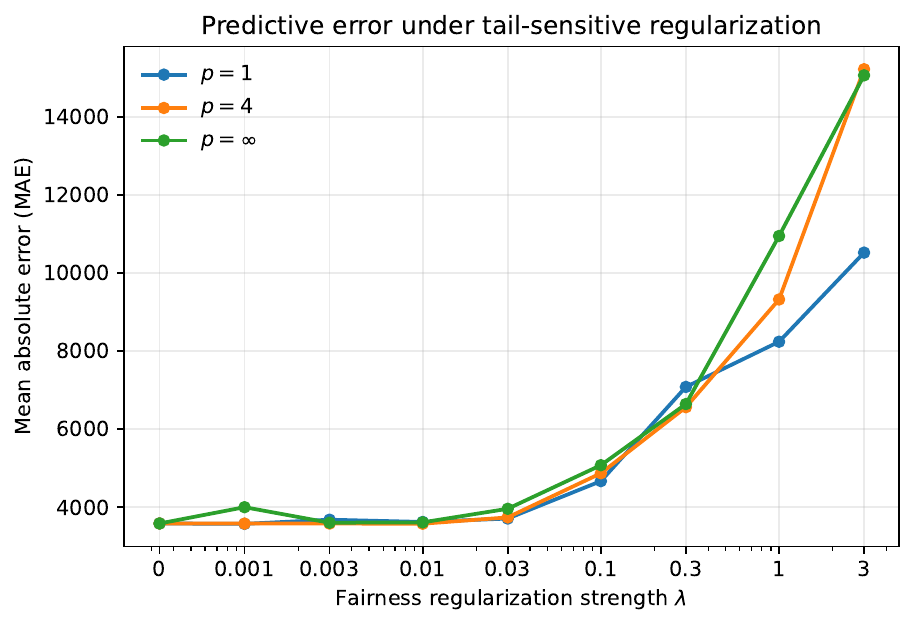}
    \caption{MAE vs.\ $\lambda$}
\end{subfigure}
\hfill
\begin{subfigure}[t]{0.45\linewidth}
    \centering
    \includegraphics[width=\linewidth]{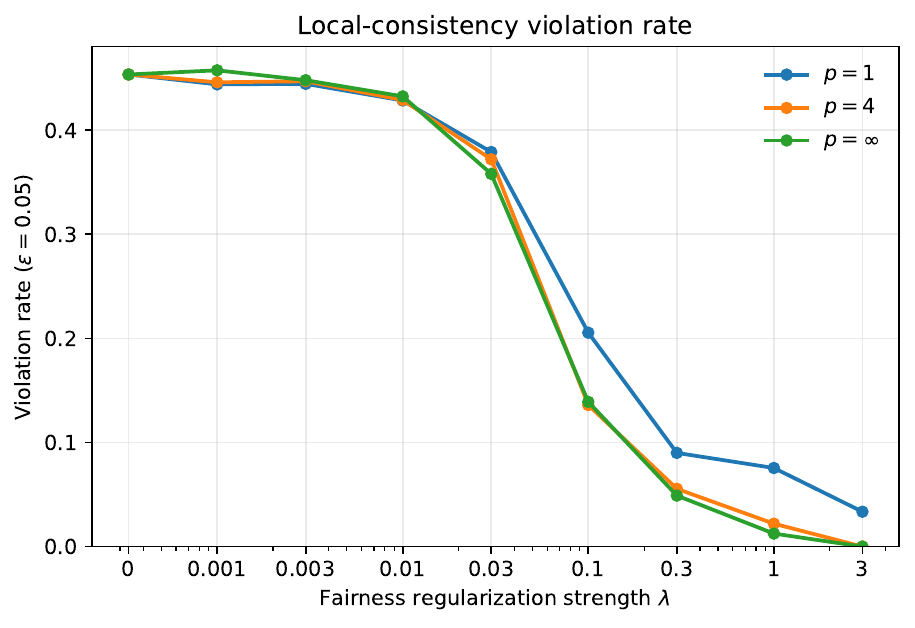}
    \caption{VR vs.\ $\lambda$ ($\epsilon=0.05$)}
\end{subfigure}
\caption{
Performance--fairness behavior under increasingly tail-sensitive local
predictive-consistency aggregation on the \textsc{Employee Salary} regression
dataset. Finite $p$ uses the empirical $p$-norm, while $p=\infty$ uses the
log-sum-exp approximation in \cref{eq:linf-lse}. Greater tail sensitivity
reduces high local-consistency violations more strongly under moderate-to-large
regularization, at the cost of increased predictive error.
}
\label{fig:pnorm_results}
\end{figure}

\section{Multi-Seed Robustness on Smaller Survival Datasets}
\label{app:multiseed-survival}

To assess the robustness of the survival results to random variation, we
repeat representative experiments on the two smaller survival datasets,
\textsc{WHAS} and \textsc{ROSSI}, using five random seeds. We consider both the log-risk and survival-curve formulations and evaluate the unregularized baseline together with representative moderate and strong fairness-regularization settings.

The random seed controls model initialization and the train--test split. Consequently, this analysis captures variability arising from both training
randomness and sampling variation. This is particularly relevant for the
smaller survival datasets, for which predictive performance can be sensitive
to the particular train--test partition.

Table~\ref{tab:multiseed-survival} reports the mean and sample standard
deviation over the five seeds. C-index is reported in its standard
orientation, where higher values indicate better concordance, whereas IBS
and the local predictive-consistency score are lower-is-better.

\begin{table}[H]
\centering
\caption{Multi-seed robustness results on the smaller survival datasets.
Values are mean $\pm$ standard deviation over five random seeds.}
\label{tab:multiseed-survival}
\small
\begin{tabular}{lllrrr}
\toprule
Formulation & Dataset & $\lambda$
& C-index $\uparrow$
& IBS $\downarrow$
& IF mean $\downarrow$ \\
\midrule

Log-risk
& ROSSI & $0$
& $0.5572 \pm 0.0498$
& $0.4062 \pm 0.0796$
& $4.6107 \pm 0.3593$ \\

Log-risk
& ROSSI & $10$
& $0.6226 \pm 0.0375$
& $0.3229 \pm 0.0540$
& $0.0976 \pm 0.0080$ \\

Log-risk
& ROSSI & $100$
& $0.6425 \pm 0.0376$
& $0.3292 \pm 0.0553$
& $0.0168 \pm 0.0041$ \\

Log-risk
& WHAS & $0$
& $0.6688 \pm 0.0701$
& $0.3004 \pm 0.0788$
& $6.2926 \pm 0.6123$ \\

Log-risk
& WHAS & $10$
& $0.5679 \pm 0.0634$
& $0.2388 \pm 0.0255$
& $0.2111 \pm 0.0332$ \\

Log-risk
& WHAS & $100$
& $0.5433 \pm 0.0736$
& $0.2476 \pm 0.0298$
& $0.1456 \pm 0.0158$ \\

\midrule

Survival curve
& ROSSI & $0$
& $0.4625 \pm 0.0582$
& $0.3780 \pm 0.0763$
& $0.1090 \pm 0.0315$ \\

Survival curve
& ROSSI & $10$
& $0.4963 \pm 0.0478$
& $0.3867 \pm 0.0758$
& $0.0768 \pm 0.0274$ \\

Survival curve
& ROSSI & $300$
& $0.5862 \pm 0.0664$
& $0.3816 \pm 0.0656$
& $0.0213 \pm 0.0077$ \\

Survival curve
& WHAS & $0$
& $0.6799 \pm 0.0251$
& $0.2591 \pm 0.0319$
& $0.0910 \pm 0.0088$ \\

Survival curve
& WHAS & $10$
& $0.5655 \pm 0.0632$
& $0.2891 \pm 0.0302$
& $0.0653 \pm 0.0084$ \\

Survival curve
& WHAS & $300$
& $0.4997 \pm 0.0972$
& $0.2869 \pm 0.0392$
& $0.0431 \pm 0.0052$ \\

\bottomrule
\end{tabular}
\end{table}

Across both survival formulations and both datasets, increasing the fairness
regularization produces a consistent reduction in the mean local
predictive-consistency score. Under the log-risk formulation, the mean IF
score on ROSSI decreases from $4.6107 \pm 0.3593$ at $\lambda=0$ to
$0.0168 \pm 0.0041$ at $\lambda=100$, while on WHAS it decreases from
$6.2926 \pm 0.6123$ to $0.1456 \pm 0.0158$. Under the survival-curve
formulation, the corresponding reductions are from
$0.1090 \pm 0.0315$ to $0.0213 \pm 0.0077$ on ROSSI and from
$0.0910 \pm 0.0088$ to $0.0431 \pm 0.0052$ on WHAS.

Predictive utility exhibits greater variability, particularly for the
survival-curve formulation on the smaller datasets. On ROSSI, the
survival-curve C-index varies substantially across train--test splits at weak
regularization, while its mean increases from $0.4625$ at $\lambda=0$ to
$0.5862$ at $\lambda=300$. On WHAS, stronger survival-curve regularization
reduces the mean C-index and increases its variability, illustrating a clearer
fairness--utility trade-off.

Overall, the multi-seed analysis indicates that stronger fairness regularization
consistently reduces the mean local predictive-consistency score across the
evaluated random seeds, whereas predictive utility can be more sensitive to
sampling variation on the smaller survival datasets. We therefore interpret the survival results primarily in terms of the consistency of the fairness trend rather than any single predictive operating point.

\section{Comparison with Laplacian Regularization}
\label{app:laplacian-comparison}

\subsection{Standard Pairwise Laplacian Regularization}
We additionally compare URW against a standard pairwise graph-Laplacian regularizer across classification, regression, and survival prediction. Because the URW and Laplacian objectives have different scales, their raw regularization parameters are not directly comparable. We therefore compare the methods through their achieved fairness--utility trade-offs over their respective regularization paths rather than at identical values of $\lambda$.

For classification, predictive utility is measured by AUC; for regression by RMSE; and for survival analysis by C-index. Fairness is evaluated using the mean instance-level IF score on the held-out test graph, with lower values indicating stronger local predictive consistency.

\begin{table}[htbp]
\centering
\caption{Representative operating points for URW and standard pairwise
Laplacian regularization. Raw $\lambda$ values should not be compared across
methods because the regularization objectives have different scales.}
\label{tab:laplacian-comparison}
\small
\begin{tabular}{lllrr}
\toprule
Task / Dataset & Method & $\lambda$
& Utility & Fairness $\downarrow$ \\
\midrule

Classification / Adult
& Laplacian & 30
& AUC $=0.8751$ & $0.001002$ \\
& URW-FKL & 5
& AUC $=0.8771$ & $0.012202$ \\
& URW-RKL & 5
& AUC $=0.8784$ & $0.012263$ \\

\midrule

Classification / COMPAS
& Laplacian & 3
& AUC $=0.6197$ & $1.24\times10^{-5}$ \\
& URW-FKL & 10
& AUC $=0.6610$ & $2.28\times10^{-3}$ \\
& URW-RKL & 10
& AUC $=0.6665$ & $2.40\times10^{-3}$ \\

\midrule

Regression / Employee Salary
& Laplacian & 0.1
& RMSE $=10748.6$ & $0.0173$ \\
& URW & 1
& RMSE $=5586.0$ & $0.0338$ \\

\midrule

Regression / Insurance
& Laplacian & 0.1
& RMSE $=11977.9$ & $0.0686$ \\
& URW & 1
& RMSE $=8467.7$ & $0.2918$ \\

\midrule

DeepSurv / FLCHAIN
& Laplacian & 3
& C-index $=0.9208$ & $0.0175$ \\
& URW & 50
& C-index $=0.9215$ & $0.0341$ \\

\midrule

DeepSurv / WHAS
& Laplacian & 1
& C-index $=0.5611$ & $0.1608$ \\
& URW & 1
& C-index $=0.6620$ & $0.3084$ \\

\midrule

DeepSurv / ROSSI
& Laplacian & 1
& C-index $=0.5991$ & $0.0161$ \\
& URW & 30
& C-index $=0.7003$ & $0.0384$ \\

\midrule

Survival curve / FLCHAIN
& Laplacian & 5
& C-index $=0.9202$ & $0.00807$ \\
& URW & 30
& C-index $=0.9240$ & $0.01094$ \\

\midrule

Survival curve / WHAS
& Laplacian & 30
& C-index $=0.5124$ & $0.04231$ \\
& URW & 300
& C-index $=0.5670$ & $0.04877$ \\

\midrule

Survival curve / ROSSI
& Laplacian & 50
& C-index $=0.5946$ & $0.00451$ \\
& URW & 300
& C-index $=0.6926$ & $0.01321$ \\

\bottomrule
\end{tabular}
\end{table}

The comparison reveals a consistent distinction between the two forms of graph regularization. Pairwise Laplacian regularization directly penalizes edge-wise prediction differences and therefore typically attains stronger absolute graph smoothing. URW instead penalizes deviation from a neighborhood aggregate and consequently behaves as a less aggressive, node-level relaxation of direct pairwise smoothing.

Across several datasets, this difference is associated with additional utility-preserving operating points for URW at comparable achieved fairness levels. This is particularly visible on COMPAS, the two regression datasets, WHAS, and ROSSI, where stronger Laplacian smoothing incurs a larger predictive cost. On ROSSI, for example, the DeepSurv URW model attains a C-index of $0.7003$ while substantially reducing local inconsistency, whereas stronger Laplacian smoothing reaches smaller fairness values but lower concordance.

The pattern is not universal. On FLCHAIN, Laplacian regularization achieves
very strong graph smoothing with little loss in predictive concordance.
Accordingly, these results do not suggest that URW uniformly dominates
pairwise Laplacian regularization. Rather, the two objectives emphasize
different regions of the fairness--utility trade-off: Laplacian regularization favors stronger edge-wise smoothing, while URW often provides additional higher-utility operating points through its neighborhood-level relaxation.

\subsection{Symmetric Normalized Laplacian Regularization}
\label{app:normalized-laplacian}

To assess whether the favorable operating points attained by URW can be explained primarily by degree normalization, we additionally compare against a symmetric normalized-Laplacian baseline. Using the same comparison graph, the normalized adjacency is defined as
\[
S = D^{-1/2} W D^{-1/2},
\]
and the corresponding regularizer is
\[
\mathcal{R}_{\mathrm{NLap}}(f)
=
\frac{1}{n} f^\top (I-S) f.
\]
For binary classification, $f$ denotes the predicted positive-class probability. All other components of the experimental pipeline, including the train/test split, graph construction, fair-projected representation, metric projection, and optimization setup, are kept unchanged.

For evaluation, we avoid comparing method-specific regularization objectives. Instead, all methods are assessed using the same \textsc{EdgeMeanAbs} diagnostic defined in Appendix~\ref{app:pairwise-diagnostics}, computed on the held-out comparison graph. Lower values indicate smaller direct prediction discrepancies between comparable instances.

\begin{table}[htbp]
\centering
\caption{
Comparison of URW, the standard Laplacian baseline, and the symmetric normalized-Laplacian baseline on representative classification, regression, and survival prediction tasks. Fairness is evaluated using the independent \textsc{EdgeMeanAbs} metric on a common held-out comparison graph. Higher AUC and C-index are better, while lower RMSE and \textsc{EdgeMeanAbs} are better.
}
\label{tab:normalized_laplacian_comparison}
\small
\begin{tabular}{llrrr}
\toprule
Task & Method & $\lambda$ & Utility & EdgeMeanAbs $\downarrow$ \\
\midrule

\textsc{Adult}
& Baseline & 0
& AUC $=0.8772$ & 0.2622 \\
& URW-FKL & 5
& AUC $=0.8771$ & 0.2333 \\
& URW-RKL & 5
& AUC $=0.8784$ & 0.2332 \\
& Laplacian & 0.1
& AUC $=0.8726$ & 0.2302 \\
& Normalized Laplacian & 1
& AUC $=0.8765$ & 0.2354 \\

\midrule

\textsc{Employee Salary}
& Baseline & 0
& RMSE $=5251.79$ & 0.08269 \\
& URW & 0.3
& RMSE $=5096.61$ & 0.07394 \\
& URW & 1
& RMSE $=5585.98$ & 0.06269 \\
& Laplacian & 0.1
& RMSE $=10748.62$ & 0.03341 \\
& Normalized Laplacian & 0.1
& RMSE $=5545.94$ & 0.07826 \\
& Normalized Laplacian & 0.3
& RMSE $=9310.49$ & 0.08082 \\

\midrule

\textsc{WHAS}
& Baseline & 0
& C-index $=0.6453$ & 27.3904 \\
& URW & 1
& C-index $=0.6620$ & 1.3146 \\
& URW & 3
& C-index $=0.6108$ & 0.9948 \\
& Laplacian & 1
& C-index $=0.5611$ & 0.6583 \\
& Normalized Laplacian & 1
& C-index $=0.6553$ & 1.2111 \\
& Normalized Laplacian & 3
& C-index $=0.6346$ & 0.9455 \\

\bottomrule
\end{tabular}
\end{table}

The results show that degree normalization is an important factor, but does not fully account for the behavior of URW. On \textsc{Adult}, the methods can be compared at closely matched levels of direct pairwise discrepancy. URW-FKL and URW-RKL at $\lambda=5$ achieve
\textsc{EdgeMeanAbs} values of $0.2333$ and $0.2332$, respectively, with AUCs of $0.8771$ and $0.8784$. The symmetric normalized-Laplacian baseline at $\lambda=1$ reaches a similar \textsc{EdgeMeanAbs} of $0.2354$, with AUC $0.8765$. The standard Laplacian reaches slightly stronger smoothing (\textsc{EdgeMeanAbs} $=0.2302$) but with a lower AUC of $0.8726$. Thus, at this moderate fairness level, URW retains slightly higher predictive utility than either Laplacian variant.

On Adult, the normalized Laplacian can attain stronger pairwise smoothing at some operating points, but the corresponding points are associated with lower predictive utility than the URW operating points with comparable achieved fairness. For example, at $\lambda=3$ it reaches \textsc{EdgeMeanAbs} $=0.1963$ with AUC $=0.8710$, a fairness level comparable to the strongly regularized URW operating points but with somewhat lower predictive utility. This indicates that symmetric degree normalization alone does not fully reproduce the favorable URW operating points observed on Adult.

On \textsc{Employee Salary}, URW similarly reaches a favorable region of the fairness--utility trade-off that is not reproduced by the normalized-Laplacian baseline. For example, URW with $\lambda=0.3$ improves both RMSE and \textsc{EdgeMeanAbs} relative to the unregularized model, whereas the normalized-Laplacian baseline yields a weaker reduction in pairwise discrepancy and degrades utility as regularization is increased.

On \textsc{WHAS}, the normalized-Laplacian baseline is more competitive,
confirming that degree normalization can substantially improve utility preservation relative to the standard Laplacian. However, its operating points remain distinct from those of URW: URW performs better in some lightly regularized regions, whereas the normalized-Laplacian baseline is stronger at some more heavily regularized points.

Overall, these results suggest that degree normalization contributes to the behavior of graph-based regularization but does not by itself reproduce the URW fairness--utility trade-off. The relative advantage remains dataset-dependent, and we therefore do not claim uniform dominance of URW over the normalized-Laplacian baseline.

\section{Comparison with Waller-Style Local Consistency}
\label{app:waller-comparison}

To further compare the proposed local predictive-consistency diagnostic with prior instance-level individual-fairness measures, we evaluate the binary classification experiments using the local consistency formulation of \citet{waller2025beyond}. This comparison is performed on Adult and COMPAS, for which the predictive output is binary.

For each test instance $\mathbf{x}_i$, let $\mathrm{nbr}(\mathbf{x}_i,k)$ denote its $k$ nearest neighbors in the comparison feature space, excluding the instance itself. Following \citet{waller2025beyond}, we compute
\begin{equation}
\label{eq:app-waller-consistency}
c_i^{(k)}
=
1-
\left|
\hat y_i
-
\frac{1}{k}
\sum_{\mathbf{x}_j\in\mathrm{nbr}(\mathbf{x}_i,k)}
\hat y_j
\right|,
\end{equation}
where $\hat y_i\in\{0,1\}$ is the predicted class label. For binary predictions, $c_i^{(k)}$ is equivalently the proportion of the $k$ neighbors receiving the same predicted class as the focal instance.

We report the mean consistency
\[
\overline{C}^{(k)}
=
\frac{1}{n}\sum_{i=1}^{n}c_i^{(k)},
\]
together with the threshold-based summaries proposed by
\citet{waller2025beyond}. For a consistency threshold $\delta$, the Proportional Consistency Score (PCS) is
\begin{equation}
\label{eq:app-waller-pcs}
\mathrm{PCS}_{\delta}^{(k)}
=
\frac{1}{n}
\sum_{i=1}^{n}
\mathbbm{1}\!\left[c_i^{(k)}\geq\delta\right],
\end{equation}
and the Balanced Conditioned Consistency (BCC) is
\begin{equation}
\label{eq:app-waller-bcc}
\mathrm{BCC}_{\delta}^{(k)}
=
\frac{1}{n}
\sum_{i=1}^{n}
c_i^{(k)}
\mathbbm{1}\!\left[c_i^{(k)}\geq\delta\right].
\end{equation}

We use $\delta=0.5$ and evaluate $k\in\{5,10,15\}$. Tables~\ref{tab:waller-adult-full} and \ref{tab:waller-compas-full} report the complete regularization paths for URW-FKL and URW-RKL under all three neighborhood sizes. In addition to mean Waller consistency, PCS, and BCC, we report the Spearman rank correlation $\rho$ between the proposed instance-level URW inconsistency and Waller inconsistency $1-c_i^{(k)}$.

\begin{table*}[t]
\centering
\caption{
Complete Waller-style local-consistency comparison on \textsc{Adult} for $k\in\{5,10,15\}$ and $\delta=0.5$. $\bar C$ denotes mean Waller consistency and $\rho$ is the Spearman correlation between the proposed instance-level URW inconsistency and Waller inconsistency $1-c_i^{(k)}$.
}
\label{tab:waller-adult-full}
\scriptsize
\setlength{\tabcolsep}{2.5pt}
\begin{adjustbox}{max width=\textwidth}
\begin{tabular}{llrr|rrrr|rrrr|rrrr}
\toprule
Method & $\lambda$ & AUC $\uparrow$ & URW IF $\downarrow$
& \multicolumn{4}{c|}{$k=5$}
& \multicolumn{4}{c|}{$k=10$}
& \multicolumn{4}{c}{$k=15$} \\
\cmidrule(lr){5-8}\cmidrule(lr){9-12}\cmidrule(lr){13-16}
& & & &
$\bar C$ $\uparrow$ & PCS $\uparrow$ & BCC $\uparrow$ & $\rho$ &
$\bar C$ $\uparrow$ & PCS $\uparrow$ & BCC $\uparrow$ & $\rho$ &
$\bar C$ $\uparrow$ & PCS $\uparrow$ & BCC $\uparrow$ & $\rho$ \\
\midrule
URW-FKL & $0$   & 0.8772 & 0.07477 & 0.8928 & 0.9139 & 0.8710 & 0.259 & 0.8755 & 0.9205 & 0.8564 & 0.296 & 0.8664 & 0.9025 & 0.8400 & 0.314 \\
URW-FKL & $0.1$ & 0.8787 & 0.07017 & 0.8953 & 0.9155 & 0.8739 & 0.259 & 0.8783 & 0.9185 & 0.8585 & 0.295 & 0.8693 & 0.9029 & 0.8433 & 0.320 \\
URW-FKL & $0.3$ & 0.8790 & 0.06678 & 0.8951 & 0.9123 & 0.8724 & 0.250 & 0.8791 & 0.9201 & 0.8603 & 0.281 & 0.8712 & 0.9026 & 0.8449 & 0.299 \\
URW-FKL & $0.5$ & 0.8796 & 0.06436 & 0.9003 & 0.9215 & 0.8811 & 0.242 & 0.8842 & 0.9242 & 0.8663 & 0.268 & 0.8745 & 0.9103 & 0.8512 & 0.284 \\
URW-FKL & $1$   & 0.8789 & 0.05653 & 0.9047 & 0.9222 & 0.8859 & 0.242 & 0.8911 & 0.9268 & 0.8745 & 0.257 & 0.8830 & 0.9128 & 0.8605 & 0.274 \\
URW-FKL & $3$   & 0.8779 & 0.04839 & 0.9088 & 0.9248 & 0.8912 & 0.218 & 0.8962 & 0.9266 & 0.8798 & 0.236 & 0.8892 & 0.9194 & 0.8702 & 0.254 \\
URW-FKL & $5$   & 0.8771 & 0.04309 & 0.9119 & 0.9282 & 0.8952 & 0.220 & 0.9010 & 0.9318 & 0.8865 & 0.233 & 0.8947 & 0.9181 & 0.8747 & 0.245 \\
URW-FKL & $10$  & 0.8740 & 0.03828 & 0.9106 & 0.9285 & 0.8945 & 0.242 & 0.9007 & 0.9312 & 0.8865 & 0.243 & 0.8950 & 0.9222 & 0.8772 & 0.254 \\
URW-FKL & $30$  & 0.8716 & 0.03161 & 0.9182 & 0.9349 & 0.9040 & 0.219 & 0.9083 & 0.9348 & 0.8946 & 0.240 & 0.9034 & 0.9275 & 0.8868 & 0.244 \\
URW-FKL & $50$  & 0.8715 & 0.03046 & 0.9216 & 0.9368 & 0.9080 & 0.208 & 0.9126 & 0.9385 & 0.8999 & 0.227 & 0.9077 & 0.9282 & 0.8907 & 0.234 \\
URW-FKL & $100$ & 0.8704 & 0.02568 & 0.9249 & 0.9386 & 0.9116 & 0.199 & 0.9162 & 0.9400 & 0.9037 & 0.223 & 0.9122 & 0.9331 & 0.8970 & 0.232 \\
URW-FKL & $300$ & 0.8728 & 0.01765 & 0.9367 & 0.9485 & 0.9264 & 0.191 & 0.9295 & 0.9500 & 0.9201 & 0.207 & 0.9261 & 0.9437 & 0.9141 & 0.216 \\
URW-FKL & $500$ & 0.8737 & 0.01385
& \textbf{0.9473} & \textbf{0.9558} & \textbf{0.9382} & 0.156
& \textbf{0.9405} & \textbf{0.9583} & \textbf{0.9328} & 0.186
& \textbf{0.9366} & \textbf{0.9518} & \textbf{0.9264} & 0.192 \\
\midrule
URW-RKL & $0$   & 0.8772 & 0.10952 & 0.8928 & 0.9139 & 0.8710 & 0.292 & 0.8755 & 0.9205 & 0.8564 & 0.338 & 0.8664 & 0.9025 & 0.8400 & 0.355 \\
URW-RKL & $0.1$ & 0.8772 & 0.09976 & 0.8943 & 0.9131 & 0.8724 & 0.272 & 0.8785 & 0.9185 & 0.8589 & 0.315 & 0.8693 & 0.9014 & 0.8425 & 0.339 \\
URW-RKL & $0.3$ & 0.8780 & 0.09168 & 0.8972 & 0.9165 & 0.8762 & 0.266 & 0.8804 & 0.9215 & 0.8617 & 0.300 & 0.8717 & 0.9023 & 0.8449 & 0.318 \\
URW-RKL & $0.5$ & 0.8785 & 0.08854 & 0.8982 & 0.9174 & 0.8778 & 0.249 & 0.8823 & 0.9207 & 0.8640 & 0.289 & 0.8737 & 0.9060 & 0.8488 & 0.308 \\
URW-RKL & $1$   & 0.8783 & 0.07756 & 0.9027 & 0.9218 & 0.8834 & 0.265 & 0.8883 & 0.9247 & 0.8706 & 0.293 & 0.8793 & 0.9109 & 0.8562 & 0.316 \\
URW-RKL & $3$   & 0.8781 & 0.06558 & 0.9094 & 0.9256 & 0.8919 & 0.233 & 0.8975 & 0.9290 & 0.8819 & 0.250 & 0.8906 & 0.9195 & 0.8713 & 0.270 \\
URW-RKL & $5$   & 0.8784 & 0.05550 & 0.9150 & 0.9301 & 0.8988 & 0.217 & 0.9027 & 0.9318 & 0.8878 & 0.233 & 0.8965 & 0.9217 & 0.8776 & 0.252 \\
URW-RKL & $10$  & 0.8738 & 0.04973 & 0.9150 & 0.9289 & 0.8980 & 0.225 & 0.9039 & 0.9332 & 0.8895 & 0.232 & 0.8985 & 0.9238 & 0.8804 & 0.241 \\
URW-RKL & $30$  & 0.8729 & 0.04087 & 0.9161 & 0.9303 & 0.9001 & 0.230 & 0.9064 & 0.9348 & 0.8931 & 0.259 & 0.9010 & 0.9268 & 0.8843 & 0.268 \\
URW-RKL & $50$  & 0.8703 & 0.03961 & 0.9182 & 0.9355 & 0.9044 & 0.205 & 0.9098 & 0.9373 & 0.8973 & 0.233 & 0.9058 & 0.9318 & 0.8913 & 0.248 \\
URW-RKL & $100$ & 0.8689 & 0.03271 & 0.9263 & 0.9392 & 0.9132 & 0.195 & 0.9179 & 0.9409 & 0.9056 & 0.219 & 0.9130 & 0.9337 & 0.8978 & 0.232 \\
URW-RKL & $300$ & 0.8716 & 0.02156 & 0.9374 & 0.9477 & 0.9266 & 0.188 & 0.9306 & 0.9510 & 0.9216 & 0.203 & 0.9267 & 0.9448 & 0.9149 & 0.204 \\
URW-RKL & $500$ & 0.8739 & 0.01665 & \textbf{0.9484} & \textbf{0.9579} & \textbf{0.9401} & 0.153 & \textbf{0.9420} & \textbf{0.9589} & \textbf{0.9342} & 0.186 & \textbf{0.9381} & \textbf{0.9542} & \textbf{0.9284} & 0.196 \\
\bottomrule
\end{tabular}
\end{adjustbox}
\end{table*}

\begin{table*}[t]
\centering
\caption{
Complete Waller-style local-consistency comparison on \textsc{COMPAS} for $k\in\{5,10,15\}$ and $\delta=0.5$. $\bar C$ denotes mean Waller consistency and $\rho$ is the Spearman correlation between the proposed instance-level URW inconsistency and Waller inconsistency $1-c_i^{(k)}$.
}
\label{tab:waller-compas-full}
\scriptsize
\setlength{\tabcolsep}{2.5pt}
\begin{adjustbox}{max width=\textwidth}
\begin{tabular}{llrr|rrrr|rrrr|rrrr}
\toprule
Method & $\lambda$ & AUC $\uparrow$ & URW IF $\downarrow$
& \multicolumn{4}{c|}{$k=5$}
& \multicolumn{4}{c|}{$k=10$}
& \multicolumn{4}{c}{$k=15$} \\
\cmidrule(lr){5-8}\cmidrule(lr){9-12}\cmidrule(lr){13-16}
& & & &
$\bar C$ $\uparrow$ & PCS $\uparrow$ & BCC $\uparrow$ & $\rho$ &
$\bar C$ $\uparrow$ & PCS $\uparrow$ & BCC $\uparrow$ & $\rho$ &
$\bar C$ $\uparrow$ & PCS $\uparrow$ & BCC $\uparrow$ & $\rho$ \\
\midrule
URW-FKL & $0$   & 0.6742 & 0.10022 & 0.8797 & 0.9037 & 0.8520 & -0.038 & 0.8474 & 0.9044 & 0.8222 & -0.059 & 0.8304 & 0.8870 & 0.7994 & -0.073 \\
URW-FKL & $0.1$ & 0.6740 & 0.09841 & 0.8735 & 0.8981 & 0.8444 & -0.040 & 0.8432 & 0.9009 & 0.8167 & -0.059 & 0.8250 & 0.8857 & 0.7933 & -0.085 \\
URW-FKL & $0.3$ & 0.6736 & 0.09073 & 0.8839 & 0.9127 & 0.8599 & -0.036 & 0.8539 & 0.9085 & 0.8292 & -0.035 & 0.8355 & 0.8891 & 0.8037 & -0.060 \\
URW-FKL & $0.5$ & 0.6742 & 0.08191 & 0.8822 & 0.9037 & 0.8546 & -0.020 & 0.8450 & 0.9016 & 0.8186 & -0.014 & 0.8284 & 0.8801 & 0.7932 & -0.023 \\
URW-FKL & $1$   & 0.6781 & 0.06568 & 0.8840 & 0.9071 & 0.8585 & -0.003 & 0.8558 & 0.9113 & 0.8326 & 0.038 & 0.8393 & 0.8919 & 0.8082 & -0.014 \\
URW-FKL & $3$   & 0.6734 & 0.03195 & 0.8977 & 0.9210 & 0.8758 & 0.097 & 0.8803 & 0.9265 & 0.8610 & 0.109 & 0.8681 & 0.9085 & 0.8404 & 0.070 \\
URW-FKL & $5$   & 0.6734 & 0.01794 & 0.9109 & 0.9335 & 0.8918 & 0.079 & 0.8915 & 0.9328 & 0.8739 & 0.095 & 0.8800 & 0.9175 & 0.8551 & 0.053 \\
URW-FKL & $10$  & 0.6610 & 0.00649 & 0.9390 & 0.9501 & 0.9243 & 0.113 & 0.9220 & 0.9487 & 0.9082 & 0.119 & 0.9053 & 0.9376 & 0.8880 & 0.085 \\
URW-FKL & $30$  & 0.6363 & 0.00162 & 0.9593 & 0.9647 & 0.9494 & 0.146 & 0.9529 & 0.9695 & 0.9464 & 0.172 & 0.9483 & 0.9619 & 0.9384 & 0.189 \\
URW-FKL & $50$  & 0.6273 & 0.00106 & 0.9714 & 0.9723 & 0.9652 & 0.156 & 0.9690 & 0.9751 & 0.9652 & 0.195 & 0.9679 & 0.9737 & 0.9639 & 0.215 \\
URW-FKL & $100$ & 0.6227 & 0.00057 & 0.9867 & 0.9868 & 0.9836 & 0.030 & 0.9855 & 0.9889 & 0.9835 & 0.072 & 0.9839 & 0.9854 & 0.9802 & 0.092 \\
URW-FKL & $300$ & 0.6198 & 0.00032 & 0.9884 & 0.9903 & 0.9866 & 0.064 & 0.9868 & 0.9889 & 0.9843 & 0.097 & 0.9857 & 0.9868 & 0.9821 & 0.117 \\
URW-FKL & $500$ & 0.6178 & 0.00023 & \textbf{0.9904} & \textbf{0.9931} & \textbf{0.9895} & -0.005 & \textbf{0.9888} & \textbf{0.9910} & \textbf{0.9867} & 0.037 & \textbf{0.9875} & \textbf{0.9896} & \textbf{0.9847} & 0.047 \\
\midrule
URW-RKL & $0$   & 0.6742 & 0.13791 & 0.8797 & 0.9037 & 0.8520 & -0.044 & 0.8474 & 0.9044 & 0.8222 & -0.063 & 0.8304 & 0.8870 & 0.7994 & -0.077 \\
URW-RKL & $0.1$ & 0.6775 & 0.13510 & 0.8816 & 0.9092 & 0.8565 & -0.030 & 0.8491 & 0.9085 & 0.8252 & -0.016 & 0.8306 & 0.8891 & 0.7990 & -0.036 \\
URW-RKL & $0.3$ & 0.6745 & 0.13294 & 0.8805 & 0.9037 & 0.8521 & 0.008 & 0.8484 & 0.9009 & 0.8213 & -0.013 & 0.8302 & 0.8760 & 0.7931 & -0.034 \\
URW-RKL & $0.5$ & 0.6782 & 0.11733 & 0.8875 & 0.9148 & 0.8643 & 0.002 & 0.8649 & 0.9182 & 0.8435 & 0.043 & 0.8478 & 0.9023 & 0.8192 & 0.010 \\
URW-RKL & $1$   & 0.6781 & 0.09520 & 0.8855 & 0.9078 & 0.8585 & 0.009 & 0.8616 & 0.9071 & 0.8367 & 0.041 & 0.8472 & 0.8940 & 0.8170 & 0.005 \\
URW-RKL & $3$   & 0.6801 & 0.04167 & 0.8973 & 0.9231 & 0.8760 & 0.055 & 0.8770 & 0.9258 & 0.8575 & 0.085 & 0.8646 & 0.8995 & 0.8335 & 0.037 \\
URW-RKL & $5$   & 0.6749 & 0.02292 & 0.9103 & 0.9342 & 0.8933 & 0.048 & 0.8915 & 0.9349 & 0.8764 & 0.081 & 0.8798 & 0.9203 & 0.8581 & 0.044 \\
URW-RKL & $10$  & 0.6665 & 0.00853 & 0.9278 & 0.9390 & 0.9114 & 0.183 & 0.9105 & 0.9390 & 0.8956 & 0.193 & 0.8965 & 0.9279 & 0.8775 & 0.161 \\
URW-RKL & $30$  & 0.6352 & 0.00176 & 0.9548 & 0.9647 & 0.9458 & 0.159 & 0.9490 & 0.9660 & 0.9423 & 0.163 & 0.9463 & 0.9598 & 0.9370 & 0.165 \\
URW-RKL & $50$  & 0.6280 & 0.00116 & 0.9724 & 0.9751 & 0.9658 & 0.176 & 0.9692 & 0.9778 & 0.9656 & 0.193 & 0.9678 & 0.9744 & 0.9631 & 0.201 \\
URW-RKL & $100$ & 0.6234 & 0.00066 & 0.9845 & 0.9841 & 0.9809 & 0.039 & 0.9832 & 0.9868 & 0.9809 & 0.091 & 0.9814 & 0.9827 & 0.9773 & 0.127 \\
URW-RKL & $300$ & 0.6179 & 0.00036 & 0.9888 & 0.9910 & 0.9872 & 0.041 & 0.9873 & 0.9896 & 0.9850 & 0.080 & 0.9865 & 0.9882 & 0.9834 & 0.102 \\
URW-RKL & $500$ & 0.6143 & 0.00022 & \textbf{0.9904} & \textbf{0.9931} & \textbf{0.9895} & 0.031 & \textbf{0.9888} & \textbf{0.9910} & \textbf{0.9867} & 0.084 & \textbf{0.9875} & \textbf{0.9896} & \textbf{0.9847} & 0.104 \\
\bottomrule
\end{tabular}
\end{adjustbox}
\end{table*}

Across all three neighborhood sizes, the Waller-style diagnostics show the same aggregate trend as the proposed residual-based measure. On Adult, stronger URW regularization generally increases mean Waller consistency, PCS, and BCC while AUC remains relatively stable over a substantial portion of the regularization path. On COMPAS, the Waller-style consistency measures improve more sharply as the URW residual decreases, although stronger regularization also produces a clearer loss in predictive utility.

The complete paths show that this aggregate agreement is not specific to a particular choice of $k$. At the same time, the instance-level rank correlations remain modest on Adult and are close to zero for many COMPAS configurations. Thus, the two diagnostics can agree on the overall direction of model-level improvement while assigning substantially different inconsistency rankings to individual instances.

\subsection{Aggregate Comparison}

Table~\ref{tab:waller-aggregate} reports the unregularized and strongest regularization endpoints for both URW-FKL and URW-RKL at $k=5$. The Waller-style measures improve in the same overall direction as the proposed residual-based diagnostic.

\begin{table}[t]
\centering
\caption{
Comparison with Waller-style local consistency on Adult and COMPAS for $k=5$ and $\delta=0.5$. Higher mean consistency, PCS, and BCC indicate greater hard-label agreement with the local $k$-nearest neighborhood. The table reports the unregularized model and the strongest evaluated regularization level.
}
\label{tab:waller-aggregate}
\small
\begin{tabular}{llrrrrr}
\toprule
Dataset & Method & $\lambda$ &
URW IF $\downarrow$ &
Mean consistency $\uparrow$ &
PCS $\uparrow$ &
BCC $\uparrow$ \\
\midrule
Adult
& URW-FKL & $0$   & 0.07477 & 0.89285 & 0.91388 & 0.87104 \\
& URW-FKL & $500$ & 0.01385 & 0.94731 & 0.95578 & 0.93822 \\
& URW-RKL & $0$   & 0.10952 & 0.89285 & 0.91388 & 0.87104 \\
& URW-RKL & $500$ & 0.01665 & 0.94837 & 0.95788 & 0.94010 \\
\midrule
COMPAS
& URW-FKL & $0$   & 0.10022 & 0.87970 & 0.90367 & 0.85198 \\
& URW-FKL & $500$ & 0.00023 & 0.99044 & 0.99307 & 0.98947 \\
& URW-RKL & $0$   & 0.13791 & 0.87970 & 0.90367 & 0.85198 \\
& URW-RKL & $500$ & 0.00022 & 0.99044 & 0.99307 & 0.98947 \\
\bottomrule
\end{tabular}
\end{table}

On Adult, mean Waller consistency increases from $0.89285$ to $0.94731$ along the URW-FKL path and to $0.94837$ along the URW-RKL path. PCS similarly increases from $0.91388$ to $0.95578$ and $0.95788$, respectively. On COMPAS, the improvement is larger: mean consistency increases from $0.87970$ to $0.99044$, while PCS increases from $0.90367$ to $0.99307$ for both regularization variants at the strongest evaluated operating point.

These results indicate that the reduction in the proposed residual is accompanied by improved local agreement under an independently defined hard-label consistency diagnostic. The same qualitative pattern is observed for $k=10$ and $k=15$, indicating that the comparison is not specific to the choice $k=5$.

\subsection{Instance-Level Disagreement}

Although the two diagnostics show similar aggregate trends, they are not
equivalent at the instance level. We compare the proposed instance-level
residual with Waller inconsistency, $1-c_i^{(k)}$, using Spearman rank
correlation. For the unregularized models at $k=5$, the correlations are
$0.259$ and $0.292$ for URW-FKL and URW-RKL on Adult, respectively. On
COMPAS, the corresponding correlations are close to zero
($-0.038$ and $-0.044$). The correlations also remain modest under stronger
regularization.

Table~\ref{tab:waller-disagreement} illustrates the two directions of
disagreement using representative cases from the unregularized URW-FKL
models. The URW percentile is computed relative to all test instances under
the corresponding model, with larger percentiles indicating larger proposed
local inconsistency.

\begin{table}[t]
\centering
\caption{
Representative instance-level disagreement between Waller-style hard-label
consistency and the proposed probability-level diagnostic for $k=5$.
The cases are selected to illustrate the two possible directions of
disagreement rather than to estimate their prevalence.
}
\label{tab:waller-disagreement}
\small
\begin{tabular}{llrrr}
\toprule
Dataset & Case &
$\hat p_i$ &
Waller consistency &
URW IF percentile (\%) \\
\midrule
Adult
& Waller-consistent / URW-high
& 0.9999 & 1.00 & 99.70 \\

Adult
& Waller-inconsistent / URW-low
& 0.0383 & 0.00 & 0.01 \\

COMPAS
& Waller-consistent / URW-high
& 0.9613 & 1.00 & 99.79 \\

COMPAS
& Waller-inconsistent / URW-low
& 0.2842 & 0.40 & 0.14 \\
\bottomrule
\end{tabular}
\end{table}

The first type of disagreement occurs when the focal instance and all of its
nearest neighbors receive the same hard class label. In this case,
$c_i^{(k)}=1$ by construction, even if their underlying predictive
probabilities differ substantially. The proposed probabilistic diagnostic can
therefore assign a relatively large local discrepancy while the Waller-style
hard-label diagnostic indicates complete consistency.

The reverse disagreement can arise near the classification boundary.
Predictions that lie on opposite sides of the decision threshold receive
different binary labels and therefore reduce Waller consistency, even when the
underlying predictive probabilities are comparatively close. In such cases,
the proposed probability-level residual may remain small.

Taken together, these results show that the two diagnostics provide
complementary information. At the model level, stronger URW regularization
also improves Waller-style local consistency, providing an external check on
the observed fairness trend. At the instance level, however, the diagnostics
need not identify the same individuals as locally inconsistent because they
operate at different levels of the predictive output: the Waller formulation
emphasizes agreement of hard decisions, whereas URW-FKL and URW-RKL retain
information about the predictive probabilities and their local aggregation.

\section{Proof of Proposition~\ref{prop:pairwise-controls-if}}
\label{app:proof-pairwise-controls-if}

\begin{proof}
Fix an instance $i$. Since $P$ is row-stochastic, we have $P_{ij}\ge 0$ and
$\sum_{j=1}^n P_{ij}=1$. Therefore,
\[
(P\mathbf f_\theta)_i=\sum_{j=1}^n P_{ij}f_\theta(\mathbf x_j)
\]
is a convex combination of predictions over the support of the $i$th row of
$P$. Hence,
\begin{align*}
\left|f_\theta(\mathbf{x}_i)-(P\mathbf{f}_\theta)_i\right|
&=
\left|
f_\theta(\mathbf{x}_i)
-
\sum_{j=1}^n P_{ij}f_\theta(\mathbf{x}_j)
\right| \\
&=
\left|
\sum_{j=1}^n P_{ij}
\bigl(f_\theta(\mathbf{x}_i)-f_\theta(\mathbf{x}_j)\bigr)
\right| \\
&\le
\sum_{j=1}^n P_{ij}
\left|
f_\theta(\mathbf{x}_i)-f_\theta(\mathbf{x}_j)
\right|,
\end{align*}
where the last inequality follows from the triangle inequality and the
non-negativity of $P_{ij}$.

By assumption, $P_{ij}=0$ whenever $(i,j)\notin\mathcal E_\tau$. Thus, all
nonzero terms in the above sum correspond to neighbors satisfying
$d_x(\mathbf{x}_i,\mathbf{x}_j)\le \tau$. Applying the local pairwise fairness
condition gives
\[
\left|
f_\theta(\mathbf{x}_i)-f_\theta(\mathbf{x}_j)
\right|
\le
c\,d_x(\mathbf{x}_i,\mathbf{x}_j)
\le
c\,\tau
\]
for every $j$ such that $P_{ij}>0$. Therefore,
\begin{align*}
\left|f_\theta(\mathbf{x}_i)-(P\mathbf{f}_\theta)_i\right|
&\le
\sum_{j=1}^n P_{ij} c\,\tau \\
&=
c\,\tau \sum_{j=1}^n P_{ij} \\
&=
c\,\tau .
\end{align*}
Since this holds for every instance $i$, we obtain
\[
\mathrm{IF}_i(\theta;P)\le c\,\tau,
\]
which completes the proof.
\end{proof}

\section{Derivation of Neighborhood-Averaged Predictions for Binary Classification}
\label{app:kl-aggregation}
This section derives the closed-form expressions for neighborhood-averaged predictions in binary classification, as described in \cref{sec:task-instantiations-binary}. Specifically, we show how Forward-KL averaging leads to a simple weighted mean in probability space, while Reverse-KL averaging corresponds to a weighted average in logit space.

\subsection{Forward-KL averaging}
\label{app:forward-kl}

Let $p_j := f_\theta(\mathbf{x}_j)$ denote the predicted probability for
instance $\mathbf{x}_j$ and consider the objective
\[
\phi(q)
=
\sum_{j=1}^n P_{ij}
\,\mathrm{KL}\!\left(
\mathrm{Bern}(p_j)
\;\|\;
\mathrm{Bern}(q)
\right).
\]

For Bernoulli distributions,
\[
\mathrm{KL}\!\left(
\mathrm{Bern}(p)\,\|\,\mathrm{Bern}(q)
\right)
=
p\log\frac{p}{q}
+
(1-p)\log\frac{1-p}{1-q}.
\]

Discarding terms independent of $q$, we minimize
\[
\tilde\phi(q)
=
-\sum_{j=1}^n P_{ij}
\left[
p_j\log q
+
(1-p_j)\log(1-q)
\right].
\]

Differentiating and setting to zero gives
\[
\frac{d\tilde\phi}{dq}
=
-\sum_{j=1}^n P_{ij}
\left(
\frac{p_j}{q}
-
\frac{1-p_j}{1-q}
\right)=0.
\]

Multiplying by $q(1-q)$ yields
\[
\sum_{j=1}^n P_{ij}
(p_j - q)=0,
\]

which implies
\[
q = \sum_{j=1}^n P_{ij} p_j
=
\sum_{j=1}^n P_{ij} f_\theta(\mathbf{x}_j).
\]
Therefore the forward-KL barycenter equals the weighted average, which coincides with the neighborhood averaging operator.

\subsection{Reverse-KL averaging}
\label{app:reverse-kl}

Let $p_j := f_\theta(\mathbf{x}_j)$ denote the predicted probability for
instance $\mathbf{x}_j$. The reverse-KL neighborhood aggregation solves
\[
\tilde f_i
=
\arg\min_{q\in(0,1)}
\sum_{j=1}^n
P_{ij}\,
\mathrm{KL}\!\left(
\mathrm{Bern}(q)\,\|\,\mathrm{Bern}(p_j)
\right).
\]

Using the Bernoulli KL divergence,
\[
\mathrm{KL}(\mathrm{Bern}(q)\,\|\,\mathrm{Bern}(p_j))
=
q\log\frac{q}{p_j}
+
(1-q)\log\frac{1-q}{1-p_j},
\]
the objective becomes
\[
\mathcal J(q)
=
q\log q + (1-q)\log(1-q)
-
qA_i
-
(1-q)B_i,
\]
where
\[
A_i=\sum_{j=1}^n P_{ij}\log p_j,
\qquad
B_i=\sum_{j=1}^n P_{ij}\log(1-p_j).
\]

Setting $\frac{d\mathcal J}{dq}=0$ yields
\[
\log\frac{q}{1-q}=A_i-B_i.
\]
Substituting the definitions of $A_i$ and $B_i$ gives
\[
\log\frac{q}{1-q}
=
\sum_{j=1}^n P_{ij}\log\frac{p_j}{1-p_j}.
\]

Therefore,
\[
\tilde f_i
=
\sigma\!\left(
\sum_{j=1}^n
P_{ij}\,
\log\frac{f_\theta(\mathbf{x}_j)}{1-f_\theta(\mathbf{x}_j)}
\right),
\]
which corresponds to averaging neighbor predictions in logit space.

\section{Sensitivity to Neural Network Depth}
\label{app:depth}
To assess the robustness of the regularized learning approach, we vary the neural network depth, considering architectures with 1, 2, or 3 hidden layers while keeping all other training settings constant. Figure~\ref{fig:depth_sensitivity} reports predictive performance (RMSE and $R^2$) alongside the mean individual fairness score.
\[
\frac{1}{n}\sum_{i=1}^n \mathrm{IF}_i(\theta;P),
\]
plotted as a function of the regularization strength $\lambda$.

Results show that predictive performance and fairness scores remain largely stable across depths, with only minor variations. This suggests that the local predictive-consistency formulation is not overly sensitive to model depth or expressive capacity. In practice, this stability allows the regularized learning approach to be applied to standard neural-network architectures without careful tuning of depth.

\begin{figure}[H]
\centering
\begin{subfigure}[t]{0.32\linewidth}
    \centering
    \includegraphics[width=\linewidth]{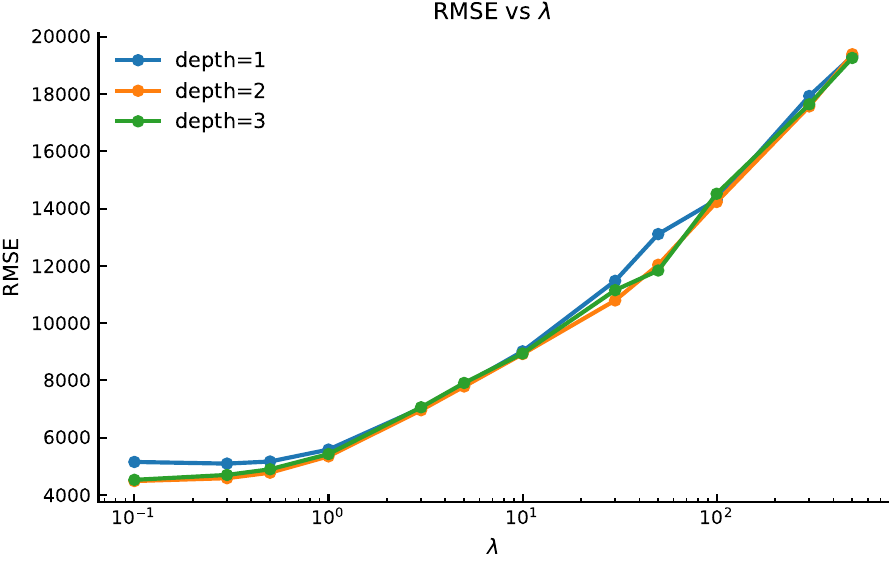}
    \caption{RMSE}
\end{subfigure}\hfill
\begin{subfigure}[t]{0.32\linewidth}
    \centering
    \includegraphics[width=\linewidth]{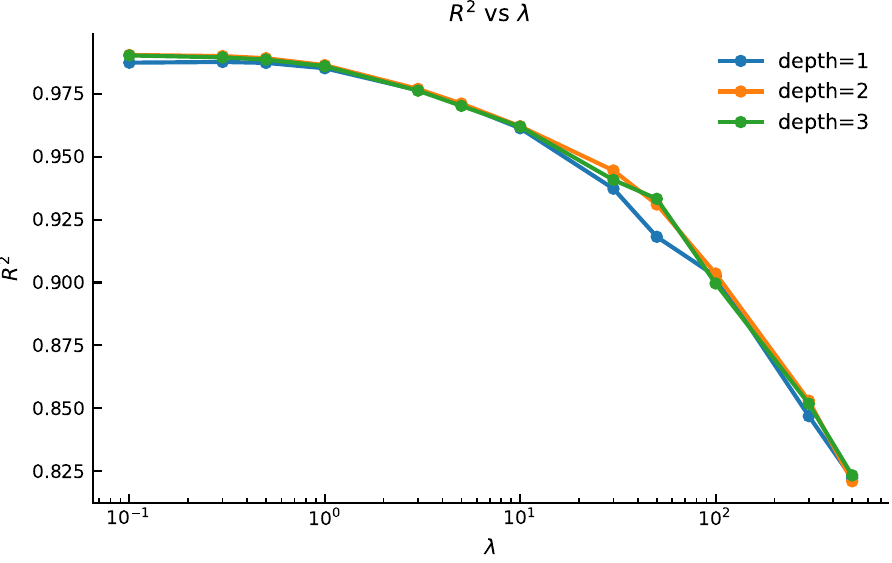}
    \caption{$R^2$}
\end{subfigure}\hfill
\begin{subfigure}[t]{0.32\linewidth}
    \centering
    \includegraphics[width=\linewidth]{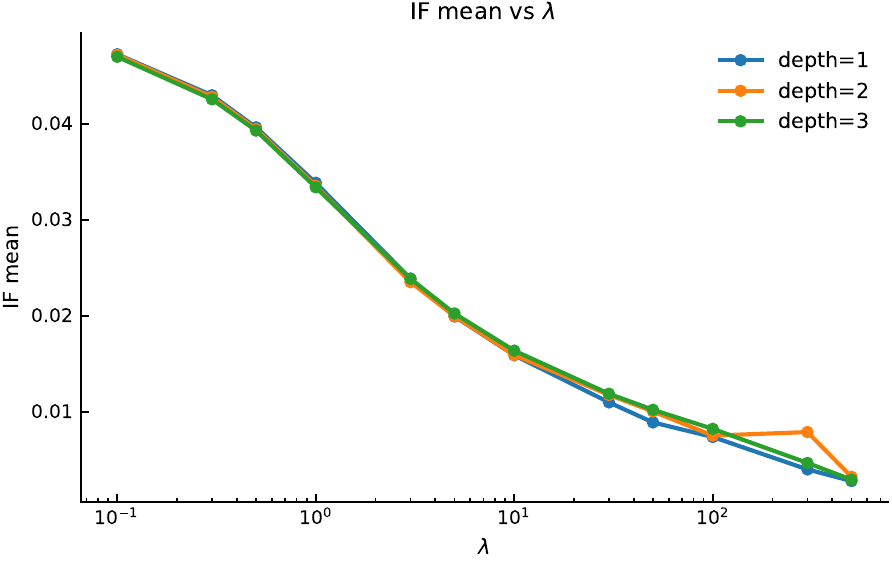}
    \caption{Mean IF score}
\end{subfigure}

\caption{
Sensitivity to neural network depth.
We evaluate model performance and fairness across depths
$1$, $2$, and $3$ on the \textsc{Employee Salary} dataset.
Predictive metrics (RMSE and $R^2$) remain largely stable, and the mean individual fairness score exhibits consistent trends across architectures.
These results indicate that the local predictive-consistency formulation is robust to network depth.
}
\label{fig:depth_sensitivity}
\end{figure}
% \fi
\end{document}